\documentclass{CSML}
\pdfoutput=1

\usepackage{lastpage}
\newcommand{\dOi}{13(4:7)2017}
\lmcsheading{}{1--\pageref{LastPage}}{}{}%
{Nov.~09,~2016}{Nov.~13, 2017}{}

\usepackage{hyperref}
\hypersetup{hidelinks}

\usepackage{amssymb}
\usepackage{bbm}
\usepackage{misc}

\usepackage{boxedminipage}
\usepackage{pifont}
\usepackage{graphicx}
\usepackage{xspace} 
\usepackage{color}

\begin{document}

\title{The Data Complexity of Description
Logic Ontologies}

\author{Carsten Lutz}	
\address{University of Bremen, Germany}	
\email{clu@uni-bremen.de}  

\author{Frank Wolter}	
\address{University of Liverpool}	
\email{wolter@liverpool.ac.uk}  
%
%

\subjclass{Description Logics}


\maketitle

\begin{abstract}
  We analyze the data complexity of ontology-mediated querying where
  the ontologies are formulated in a description logic (DL) of the
  \ALC family and queries are conjunctive queries, positive
  existential queries, or acyclic conjunctive queries. Our approach is
  non-uniform in the sense that we aim to understand the complexity of
  each single ontology instead of for all ontologies formulated in a
  certain language. While doing so, we quantify over the queries and
  are interested, for example, in the question whether all queries can
  be evaluated in polynomial time w.r.t.\ a given ontology. Our
  results include a \PTime/{\sc coNP}-dichotomy for ontologies of
  depth one in the description logic $\mathcal{ALCFI}$, the same
  dichotomy for $\mathcal{ALC}$- and $\mathcal{ALCI}$-ontologies of
  unrestricted depth, and the non-existence of such a dichotomy for
  $\mathcal{ALCF}$-ontologies. For the latter DL, we additionally show
  that it is undecidable whether a given ontology admits {\sc PTime}
  query evaluation. We also consider the connection between {\sc
    PTime} query evaluation and rewritability into (monadic) Datalog.
\end{abstract}

\section{Introduction}

In recent years, the use of ontologies to access instance data has
become increasingly popular
\cite{DBLP:journals/jods/PoggiLCGLR08,DBLP:conf/rweb/KontchakovZ14,DBLP:conf/rweb/BienvenuO15}. The
general idea is that an ontology provides domain knowledge and an
enriched vocabulary for querying, thus serving as an interface between
the query and the data, and enabling the derivation of additional
facts.  In this emerging area, called ontology-mediated querying, it
is a central research goal to identify ontology languages for which
query evaluation scales to large amounts of instance data.  Since the
size of the data typically dominates the size of the ontology and the
size of the query by orders of magnitude, the central measure for such
scalability is \emph{data complexity}---the complexity of query
evaluation where only the data is considered to be an input, but both
the query and the ontology are fixed.

In description logic (DL), ontologies take the form of a TBox, data is
stored in an ABox, and the most important classes of queries are
conjunctive queries~(CQs) and variations thereof, such as positive
existential queries (PEQs). A~fundamental observation regarding this
setup is that, for expressive DLs such as \ALC and \shiq, the
complexity of query evaluation is {\sc coNP}-complete and thus
intractable
\cite{Schaerf-93,journals/jar/HustadtMS07,jair2008g}.\footnote{When
  speaking of complexity, we \emph{always} mean data complexity} The
classical approach to avoiding this problem is to replace \ALC and \shiq
with less expressive DLs that are `Horn' in the sense that they can be
embedded into the Horn fragment of first-order (FO) logic. Horn DLs
typicall admit query evaluation in \PTime, examples include a variety
of logics from the $\mathcal{EL}$ \cite{BaBrLu-IJCAI-05} and DL-Lite
families \cite{CDLLR07} as well as Horn-$\mathcal{SHIQ}$, a large
fragment of \shiq with {\sc PTime} query
evaluation~\cite{journals/jar/HustadtMS07}.

It may thus seem that the data complexity of query evaluation in the
presence of DL ontologies is understood rather well. However, all
results discussed above are at the \emph{level of logics}, i.e.,
traditional results about data complexity concern a class of TBoxes
that is defined in a syntactic way in terms of expressibility in a
certain DL language, but no attempt is made to identify more structure
\emph{inside} these classes. Such a more fine-grained study, however,
seems very natural both from a theoretical and from a practical
perspective; in particular, it is well-known that ontologies which
emerge in practice tend to use `expensive' language constructs that
can result in {\sc coNP}-hardness of data complexity, but they
typically do so in an extremely restricted and intuitively `harmless'
way. This distinction between hard and harmless cases cannot be
analyzed on the level of logics. The aim of this paper is to initiate
a more fine-grained study of data complexity that is
\emph{non-uniform} in the sense that it does not treat all TBoxes
formulated in the same DL in a uniform way.

When taking a non-uniform perspective, there is an important choice
regarding the level of granularity. First, one can analyze the
complexity on the \emph{level of TBoxes}, quantifying over the actual
query. Then, query evaluation for a TBox \Tmc is in \PTime if every
query (from the class under consideration) can be evaluated in \PTime
w.r.t.\ \Tmc and it is {\sc coNP}-hard if there is at least one query
that is {\sc coNP}-hard to evaluate w.r.t.~\Tmc. And second, one might
take an even more fine-grained approach where the query is not
quantified away and the aim is to classify the complexity \emph{on the
  level of ontology-mediated queries (OMQs)}, that is, combinations of
a TBox and an actual query. From a practical perspective, both setups
make sense; when the actual queries are fixed at the design time of
the application, one would probably prefer to work on the level of
OMQs whereas the level of TBoxes seems more appropriate when the
queries can be freely formulated at application running time.  A
non-uniform analysis on the level of OMQs has been carried out in
\cite{DBLP:journals/tods/BienvenuCLW14}. In this paper, we concentrate
on the level of TBoxes. The ultimate goal of our approach is as
follows:

\vspace*{0.1cm}

\noindent {\em For a fixed DL $\Lmc$ and query language \Qmc, classify
  all TBoxes $\Tmc$ in $\Lmc$ according to the complexity of 
  evaluating queries from \Qmc
 w.r.t.~$\Tmc$.
}

\vspace*{0.1cm}

We consider the basic expressive DL \ALC, its extensions \ALCI with
inverse roles and \ALCF with functional roles, and their union
\ALCFI. As query languages, we cover CQs, acyclic CQs, and PEQs (which
have the same expressive power as unions of conjunctive queries, UCQs,
which are thus implicitly also covered). In the current paper, we mainly 
concentrate on understanding the boundary between {\sc PTime} and {\sc coNP}-hardness of query evaluation
w.r.t.\ DL TBoxes, mostly neglecting other relevant classes 
such as AC$^0$, {\sc LogSpace}, and {\sc NLogSpace}.  

Our main results are as
follows (they apply to all query languages mentioned above).

\medskip

\noindent
1. There is a {\sc PTime}/{\sc coNP}-dichotomy for query evaluation
w.r.t.\ $\mathcal{ALCFI}$-TBoxes of depth one, i.e., TBoxes in which
no existential or universal restriction is in the scope of another existential or universal restriction. 

\medskip
\noindent
The proof rests on interesting model-theoretic characterizations of
polynomial time CQ-evaluation which are discussed below.  Note that
this is a relevant case since most TBoxes from practical applications
have depth one. In particular, all TBoxes formulated in DL-Lite and
its extensions proposed
in~\cite{CDLLR07,DBLP:journals/jair/ArtaleCKZ09} have depth one, and
the same is true for more than 80 percent of the 429 TBoxes in the
BioPortal ontology repository.
In connection with Point~1 above, we also show that {\sc PTime} query
evaluation coincides with rewritability into monadic Datalog (with
inequalities, to capture functional roles). As in the case of data
complexity, what we mean here is that \emph{all} queries are
rewritable into monadic Datalog w.r.t.\ the TBox \Tmc under
consideration.

\medskip

\noindent 
2. There is a {\sc PTime}/{\sc coNP}-dichotomy for query evaluation w.r.t.\ $\mathcal{ALCI}$-TBoxes.

\medskip

\noindent
This is proved by showing that there is a {\sc PTime}/{\sc
  coNP}-dichotomy for query evaluation w.r.t.\ $\mathcal{ALCI}$-TBoxes
if and only if there is a {\sc PTime}/{\sc NP}-dichotomy for
non-uniform constraint satisfaction problems with finite templates
(CSPs). The latter is known as the Feder-Vardi conjecture that was
recently proved in \cite{Dicho1,Dicho2}, as the culmination of a major research
programme that combined complexity theory, graph theory, logic, and
algebra
\cite{DBLP:journals/siamcomp/BulatovJK05,DBLP:conf/stoc/KunS09,DBLP:conf/csr/Bulatov11,DBLP:journals/siglog/Barto14}.
Our equivalence proof establishes a close link between query
evaluation in $\mathcal{ALC}$ and $\mathcal{ALCI}$ and CSP that is
relevant for DL research also beyond the stated dichotomy problem.
Note that, in contrast to the proof of the Feder-Vardi conjecture, the
dichotomy proof for TBoxes of depth one (stated as Point~1 above) is
much more elementary. Also, it covers functional roles and establishes
equivalence between \PTime query evaluation and rewritability into
monadic Datalog, which fails for \ALCI-TBoxes of unrestricted depth
even when monadic Datalog is replaced with Datalog; this is a
consequence of the link to CSPs establishes in this paper.

\medskip

\noindent
3. There is no {\sc PTime}/{\sc coNP}-dichotomy for
query evaluation w.r.t.\ $\mathcal{ALCF}$-TBoxes (unless {\sc PTime} $=$
{\sc NP}).

\medskip
\noindent
This is proved by showing that, for every problem in {\sc
  coNP}, there is an $\mathcal{ALCF}$-TBox for which query evaluation has
the same complexity (up to polynomial time reductions); it then remains to
apply Ladner's Theorem, which guarantees the existence of {\sc NP}-intermediate problems.
Consequently, we cannot expect an exhaustive classification 
of the complexity of query evaluation w.r.t.\ $\mathcal{ALCF}$-TBoxes.
Variations of the proof of Point~3 allow us to establish also the
following:

\medskip

\noindent
4. For \ALCF-TBoxes, the following problems are undecidable:
\PTime-hardness of query evaluation, {\sc coNP}-hardness of query
evaluation, and rewritability into monadic Datalog and into Datalog
(with inequalities).

\medskip

To prove the results listed above, we introduce two new notions that
are of independent interest and general utility. The first one is
\emph{materializability} of a TBox \Tmc, which means that evaluating a
query over an ABox \Amc w.r.t.\ \Tmc can be reduced to query
evaluation in a single model of \Amc and \Tmc (a
\emph{materialization}). Note that such models play a crucial role in
the context of Horn DLs, where they are often called canonical models
or universal models. In contrast to the Horn DL case, however, we only
require the \emph{existence} of such a model without making any
assumptions about its form or construction.

\medskip

\noindent
5. If an $\mathcal{ALCFI}$-TBox $\Tmc$ is not materializable, then
CQ-evaluation w.r.t.~$\Tmc$ is {\sc coNP}-hard.

\medskip
\noindent
We also investigate the nature of materializations. It turns out that
if a TBox is materializable for one of the considered query languages,
then it is materializable also for all others. The concrete
materializations, however, need not agree. To obtain these results, we
characterize CQ-materializations in terms of homomorphisms and
ELIQ-materializations in terms of simulations (an ELIQ is an
$\mathcal{ELI}$-instance query, thus the DL version
of an acyclic CQ, with a single answer variable).

\smallskip

Perhaps in contrary to the intuitions that arise from the experience
with Horn DLs, materializability of a TBox \Tmc is \emph{not} a
sufficient condition for query evaluation w.r.t.\ \Tmc to be in \PTime
(unless \PTime = \NP) since the existing materialization might be hard
to compute.  This leads us to study the notion of \emph{unraveling
  tolerance} of a TBox~\Tmc, meaning that answers to acyclic CQs over
an ABox \Amc w.r.t.~$\Tmc$ are preserved under unraveling the ABox
\Amc.  In CSP, unraveling tolerance corresponds to the existence of
tree obstructions, a notion that characterizes the well known arc
consistency condition and rewritability into monadic
Datalog~\cite{DBLP:journals/siamcomp/FederV98,DBLP:conf/csl/Krokhin10}.
It can be shown that every TBox formulated in Horn-$\mathcal{ALCFI}$
(the intersection of $\mathcal{ALCFI}$ and Horn-$\mathcal{SHIQ}$) is
unraveling tolerant and that there are unraveling tolerant TBoxes
which are not equivalent to any Horn-$\mathcal{ALCFI}$-TBox. Thus, the
following result yields a rather general (and uniform!) {\sc PTime}
upper bound for CQ-evaluation.

\medskip

\noindent
6. If an $\mathcal{ALCFI}$-TBox $\Tmc$ is unraveling tolerant, then
query evaluation
w.r.t.~$\Tmc$ is in {\sc PTime}.

\smallskip
\noindent
Although the above result is rather general, unraveling tolerance of a
TBox \Tmc is \emph{not} a necessary condition for CQ-evaluation
w.r.t.\ \Tmc to be in {\sc PTime} (unless \PTime = \NP). However, for
$\mathcal{ALCFI}$-TBoxes \Tmc \emph{of depth one}, being
materializable and being unraveling tolerant turns out to be
equivalent. For such TBoxes, we thus obtain that CQ-evalutation
w.r.t.~$\Tmc$ is in {\sc PTime} iff $\Tmc$ is materializable iff
$\Tmc$ is unraveling tolerant while, otherwise, CQ-evaluation
w.r.t.~$\Tmc$ is {\sc coNP}-hard. This establishes the first main
result above.






\medskip

Our framework also allows one to formally capture some intuitions and
beliefs commonly held in the context of CQ-answering in DLs.  For
example, we show that for every $\mathcal{ALCFI}$-TBox \Tmc,
CQ-evaluation is in {\sc PTime} iff PEQ-evaluation is in {\sc PTime}
iff ELIQ-evaluation is in {\sc PTime}, and the same is true for {\sc
  coNP}-hardness and for rewritability into Datalog and
into monadic Datalog.
In fact, the use of multiple query languages and in particular of
$\mathcal{ELI}$-instance queries does not only yield additional
results, but is at the heart of our proof strategies.  Another
interesting observation in this spirit is that an
$\mathcal{ALCFI}$-TBox is materializable iff it is convex, a condition
that is also called the disjunction property and plays a central role
in attaining \PTime complexity for standard reasoning in Horn DLs such
as \EL, DL-Lite, and Horn-\shiq; see for example
\cite{BaBrLu-IJCAI-05,DBLP:conf/lpar/KrisnadhiL07} for more details.

\medskip
This paper is a significantly extended and revised version of the conference
publication~\cite{DBLP:conf/kr/LutzW12}.

\subsection*{Related Work}

An early reference on data complexity in DLs is \cite{Schaerf-93},
showing {\sc coNP}-hardness of ELQs in the fragment $\mathcal{ALE}$ of
\ALC (an ELQ is an ELIQ in which all edges are directed away from the
answer variable).  A {\sc coNP} upper bound for ELIQs in the much more
expressive DL \shiq was obtained in \cite{journals/jar/HustadtMS07}
and generalized to CQs in \cite{jair2008g}. Horn-\shiq was first
defined in \cite{journals/jar/HustadtMS07}, where also a \PTime upper
bound for ELIQs is established; the generalization to CQs can be found
in \cite{conf/jelia/EiterGOS08}. See also
\cite{DBLP:conf/lpar/KrisnadhiL07,DBLP:conf/icdt/Rosati07,DBLP:journals/jar/OrtizCE08,DBLP:journals/ai/CalvaneseGLLR13}
and references therein for the data complexity in DLs and
\cite{LICSo,IJCO} for related work on the guarded fragment and on
existential rules.

To the best of our knowledge, the conference version of this paper was
first to initiate the study of data complexity in ontology-mediated
querying at the level of individual TBoxes and the first to observe a
link between this area and CSP. There is, however, a certain technical
similarity to the link between view-based query processing for regular
path queries (RPQs) and CSP found in
\cite{LICS00,PODS03,DBLP:journals/sigmod/CalvaneseGLV03}. In this
case, the recognition problem for perfect rewritings for RPQs can be
polynomially reduced to non-uniform CSP and vice versa. On the level
of OMQs, the data complexity of ontology-mediated querying with DLs
has been studied in \cite{DBLP:journals/tods/BienvenuCLW14}, see also
\cite{OurNewPaper}; also here, a connection to CSP plays a central role. In
\cite{DBLP:conf/ijcai/LutzSW13,DBLP:conf/ijcai/LutzSW15}, the
non-uniform data complexity of ontology-mediated query answering is
studied in the case where the TBox is formulated in an inexpressive DL
of the DL-Lite or \EL family and where individual predicates in the
data can be given a closed-world reading, which also gives rise to
{\sc coNP}-hardness of query evaluation; while
\cite{DBLP:conf/ijcai/LutzSW13} is considering the level of TBoxes,
\cite{DBLP:conf/ijcai/LutzSW15} treats the level of OMQs, establishing
a connection to surjective CSPs. Rewritability into Datalog for atomic
queries and at the level of OMQs has also been studied
in~\cite{DBLP:journals/ai/KaminskiNG16}. Finally, we mention
\cite{leif} where a complete classification of the data complexity of
OMQs (also within \PTime) is achieved when the TBox is formulated
in \EL and the actual queries are atomic queries.

Recently, the data complexity at the level of TBoxes has been studied
also in the guarded fragment and in the two-variable guarded fragment
of FO with counting~\cite{PODS17}. This involves a generalization of
the notions of materializability and unraveling tolerance and leads
to a variety of {\sc PTime}/{\sc coNP}-dichotomy results. 
In particular, our dichotomy between Datalog-rewritability and {\sc
  coNP} is extended from $\mathcal{ALCIF}$-TBoxes of depth one to
$\mathcal{ALCHIF}$-TBoxes of depth two.  Using a variant of Ladner's
Theorem, several  
non-dichotomy results for weak fragments of the two-variable guarded
fragment with counting of depth two are established and it is shown
that {\sc PTime} data complexity of query evaluation is
undecidable. For $\mathcal{ALCHIQ}$-TBoxes of depth one, though, {\sc
  PTime} data complexity of query evaluation and, equivalently,
rewritability into Datalog (with inequalities) is proved to be
decidable.  In \cite{DL17a}, the results presented in this paper have
been used to show that whenever an $\mathcal{ALCIF}$-TBox of depth one
enjoys {\sc PTime} query evaluation, then it can be rewritten into a
Horn-$\mathcal{ALCIF}$-TBox that gives the same answers to CQs
(the converse is trivial).
It is also proved that this
result does not hold in other cases such as for
$\mathcal{ALCHIF}$-TBoxes of depth one.
 
%

The work on CSP dichotomies started with Schaefer's {\sc PTime}/{\sc
  NP}-dichotomy theorem, stating that every CSP defined by a two
element template is in {\sc PTime} or {\sc NP}-hard
\cite{DBLP:conf/stoc/Schaefer78}. 
Schaefer's theorem was followed by dichotomy results for CSPs with
(undirected) graph templates \cite{DBLP:journals/jct/HellN90} and
several other special cases, leading to the widely known Feder-Vardi
conjecture which postulates a {\sc PTime}/{\sc NP}-dichotomy for all
CSPs, independently of the size of the template
\cite{DBLP:journals/siamcomp/FederV98}. The conjecture has recently been
confirmed~\cite{Dicho1,Dicho2} using an approach to studying the complexity of
CSPs via universal algebra
\cite{DBLP:journals/siamcomp/BulatovJK05}. Interesting results have
also been obtained for other complexity classes such as AC$^{0}$
\cite{DBLP:conf/mfcs/AllenderBISV05,DBLP:journals/lmcs/LaroseLT07}.

\section{Preliminaries}
\label{sect:prelim}
We introduce the relevant description logics and query languages,
define the fundamental notions studied in this paper, and illustrate
them with suitable examples.

We shall be concerned with the DL $\mathcal{ALC}$ and its extensions
$\mathcal{ALCI}$, $\mathcal{ALCF}$, and $\mathcal{ALCFI}$. Let \NC,
\NR, and \NI denote countably infinite sets of \emph{concept names},
\emph{role names}, and \emph{individual names}, respectively.
\emph{$\mathcal{ALC}$-concepts} are constructed according to the rule
$$
C,D\; := \; \top \mid \bot \mid A \mid  C\sqcap D \mid C \sqcup D \mid \neg C \mid  
\exists r.C \mid \forall r . C
$$
where $A$ ranges over $\NC$ and $r$ ranges over $\NR$. 
\emph{\ALCI-concepts} admit, in addition, \emph{inverse roles} from
the set $\Nsf_\Rsf^{-}=\{r^{-} \mid r\in \NR\}$, which can be used in
place of role names. Thus, $A \sqcap \exists r^- . \forall s . B$ is
an example of an \ALCI-concept. To avoid heavy notation, we set
$r^{-}:=s$ if $r=s^{-}$ for a role name~$s$; in particular, we thus
have $(r^-)^-=r$.  

In DLs, ontologies are formalized as TBoxes. An
\emph{$\mathcal{ALC}$-TBox} is a finite set of \emph{concept
  inclusions (CIs)} $C\sqsubseteq D$, where $C,D$ are $\mathcal{ALC}$
concepts, and $\mathcal{ALCI}$-TBoxes are defined analogously.  An
\emph{$\mathcal{ALCF}$-TBox (resp.\ $\mathcal{ALCFI}$-TBox)} is an
$\mathcal{ALC}$-TBox (resp.\ $\mathcal{ALCI}$-TBox) that additionally
admits functionality assertions ${\sf func}(r)$, where $r \in \NR$
(resp.\ $r \in \NR \cup\Nsf_\Rsf^-$), declaring that $r$ is
interpreted as a partial function. Note that there is no such thing as
an \ALCF-concept or an \ALCFI-concept, as the extension with
functional roles does not change the concept language.

An \emph{ABox $\Amc$} is a non-empty finite set of assertions of the form $A(a)$
and $r(a,b)$ with $A \in \NC$, $r \in \NR$, and $a,b \in \NI$. In some
cases, we drop the finiteness condition on ABoxes and then explicitly
speak about \emph{infinite ABoxes}. We use $\mn{Ind}(\Amc)$ to denote
the set of individual names used in the ABox~\Amc and sometimes write
$r^-(a,b) \in \Amc$ instead of $r(b,a) \in \Amc$. 

%

The semantics of DLs is given by \emph{interpretations} $\Imc
=(\Delta^\Imc,\cdot^\Imc)$, where $\Delta^\Imc$ is a non-empty set and
$\cdot^\Imc$ maps each concept name $A\in\NC$ to a subset $A^\I$ of
$\Delta^\Imc$ and each role name $r\in\NR$ to a binary relation $r^\I$ on
$\Delta^\Imc$. 
%
The extension $(r^{-})^{\Imc}$ of $r^{-}$ under the
interpretation $\Imc$ is defined as the converse relation
$(r^{\Imc})^{-1}$ of $r^{\Imc}$ and the extension $C^{\Imc}\subseteq
\Delta^{\Imc}$ of concepts under the interpretation \Imc is
defined inductively as follows:
\begin{eqnarray*}
\top^{\Imc} & = & \Delta^{\Imc}\\
\bot^{\Imc} & = & \emptyset\\
(\neg C)^{\Imc} & = & \Delta^{\Imc}\setminus C^{\Imc}\\
(C \sqcap D)^{\Imc} & = & C^{\Imc} \cap D^{\Imc}\\
(C \sqcup D)^{\Imc} & = & C^{\Imc} \cup D^{\Imc}\\
(\exists r.C)^{\Imc} & = & \{ d\in \Delta^{\Imc} \mid \exists d'\in \Delta^{\Imc}:\; (d,d')\in r^{\Imc} \mbox{ and } d'\in C^{\Imc}\}\\
(\forall r.C)^{\Imc} & = & \{ d\in \Delta^{\Imc} \mid \forall d'\in \Delta^{\Imc}:\; (d,d')\in r^{\Imc} \mbox{ implies } d'\in C^{\Imc}\}
\end{eqnarray*}
An interpretation $\Imc$ \emph{satisfies} a CI $C \sqsubseteq D$ if
$C^\Imc \subseteq D^\Imc$, an assertion $A(a)$ if $a \in A^\Imc$, an
assertion $r(a,b)$ if $(a,b) \in r^\Imc$, and a functionality
assertion ${\sf func}(r)$ if $r^{\Imc}$ is a partial function.  Note
that we make the \emph{standard name assumption}, that is, individual
names are not interpreted as domain elements (like first-order
constants), but as themselves. This assumption is common both in
DLs and in database theory. The results in this paper do not depend on it.

An interpretation \Imc is a \emph{model} of a TBox \Tmc if it
satisfies all CIs in $\Tmc$ and $\Imc$ is a \emph{model} of an ABox
$\Amc$ if all individual names from \Amc are in in $\Delta^\Imc$ and
\Imc  satisfies all assertions in \Amc.  We call an ABox \Amc
\emph{consistent w.r.t.\ a TBox} \Tmc if \Amc and \Tmc have a joint
model.

\smallskip

We consider several query languages.  A \emph{positive existential
  query (PEQ)} $q(\vec{x})$ is a first-order formula with free
variables $\vec{x}=x_{1},\ldots,x_{n}$ constructed from atoms $A(x)$
and $r(x,y)$ using conjunction, disjunction, and existential
quantification, where $A\in \NC$, $r\in \NR$, and $x,y$ are variables.
The variables in $\vec{x}$ are the \emph{answer variables} of
$q(\vec{x})$.  A PEQ without answer variables is \emph{Boolean}.  An
\emph{assignment $\pi$ for $q(\vec{x})$ in an interpretation $\Imc$}
is a mapping from the variables that occur in $q(\vec{x})$ to
$\Delta^{\Imc}$.  A tuple $\vec{a}=a_{1},\ldots,a_{n}$ in
$\mn{Ind}(\Imc)$ is an \emph{answer to $q(\vec{x})$ in $\Imc$} if
there exists an assigment $\pi$ for $q(\vec{x})$ in $\Imc$ such that
$\Imc\models^{\pi}q(\vec{x})$ (in the standard first-order sense) and
$\pi(x_{i})=a_{i}$ for $1\leq i \leq n$. In this case, we write
$\Imc\models q(\vec{a})$.  A tuple $\vec{a} \in \mn{Ind}(\Amc)$, \Amc
an ABox, is a \emph{certain answer to $q(\vec{x})$ in \Amc w.r.t.}  a
TBox \Tmc, in symbols $\Tmc,\Amc\models q(\vec{a})$, if $\mathcal{I}
\models q(\vec{a})$ for all models \Imc of \Tmc and \Amc. Computing
certain answers to a query in the sense just defined is the main
querying problem we are interested in.  Although this paper
focusses on the theoretical aspects of query answering, we given a
concrete example that illustrates the usefulness of query answering
with DL ontologies.
\begin{example}
  Let
  $$
  \begin{array}{rcl}
    \Tmc &=& \{ \mn{Professer} \sqsubseteq \mn{Academic}, \quad
    \mn{Professor} \sqsubseteq \exists \mn{gives} . \mn{Course} \}
    \\[1mm]
    \Amc &=& \{ \mn{Student}(\mn{john}), \
    \mn{supervisedBy}(\mn{john},\mn{mark}), \
    \mn{Professor}(\mn{mark}) \} \\[1mm]
    q(x,y) &=& \exists z \, \mn{Student}(x) \wedge \mn{supervisedBy}(x,y) \wedge
    \mn{Academic}(y) \wedge \mn{gives}(y,z) \wedge \mn{Course}(z)
  \end{array}
  $$
  Thus the query asks to return all pairs that consist of a student
  $x$ and an academic $y$ such that $x$ is supervised by $y$ and $y$
  gives a course. Although this information is not directly present in
  the ABox, because of the TBox it is easy to see that
  $(\mn{john},\mn{mark})$ is a certain answer.
\end{example}
Apart from PEQs, we also study several fragments thereof. A
\emph{conjunctive query (CQ)} is a PEQ without disjunction.  We
generally assume that a CQ $q(\vec{x})$ takes the form $\exists
\vec{y} \, \varphi(\vec{x},\vec{y})$, where $\varphi(\vec{x},\vec{y})$
is a conjunction of atoms of the form $A(x)$ and $r(x,y)$. It is easy
to see that every PEQ $q(\vec{x})$ is equivalent to a disjunction
$\bigvee_{i\in I}q_{i}(\vec{x})$, where each $q_{i}(\vec{x})$ is a CQ
(such a disjunction is often called a \emph{union of conjunctive
  queries, or UCQ}). 

To introduce simple forms of CQs that play a crucial role in this
paper, we recall two further DLs that we use here for mainly querying
purposes. \emph{$\mathcal{EL}$-concepts} are constructed from $\NC$
and $\NR$ according to the syntax rule
$$
C,D\; := \; \top \mid A \mid  C\sqcap D \mid  
\exists r.C
$$
and \emph{$\mathcal{ELI}$-concepts} additionally admit inverse roles.
An \emph{\EL-TBox} is a finite set of concept inclusions $C
\sqsubseteq D$ with $C$ and $D$ \EL-concepts, and likewise for
\emph{\ELI-TBoxes}. 

We now use \EL and \ELI to define restricted classes of CQs. 
If $C$ is an $\mathcal{ELI}$-concept and $x$ a variable, then $C(x)$
is called an \emph{$\mathcal{ELI}$ query (ELIQ)}; if $C$ is an
$\mathcal{EL}$-concept, then $C(x)$ is called an \emph{\EL query
  (ELQ)}. Note that every ELIQ can be regarded as an acyclic CQ with
one answer variable, and indeed this is an equivalent definition of
ELIQs; in the case of ELQs, it is additionally the case that all edges
are directed away from the answer variable. For example, the ELIQ
$\exists r.(A \sqcap \exists s^{-}.B)(x)$ is equivalent to the acyclic
CQ
$$
\exists y_{1}\exists y_{2}(r(x,y_{1})\wedge A(y_{1})\wedge s(y_{2},y_{1})\wedge B(y_{2})).
$$
In what follows, we will not distinguish between an ELIQ and its
translation into an acyclic CQ with one answer variable and freely apply notions introduced for
PEQs also to ELIQs and ELQs. We also sometimes slightly abuse notation and use PEQ to denote
the set of all positive existential queries, and likewise for CQ, ELIQ, and ELQ.
\begin{example}\label{ex1}
~

\smallskip 
\noindent 
  (1) Let $\Tmc_{\exists,r}= \{A \sqsubseteq \exists r.A\}$ and $q(x)= \exists r.A(x)$.
Then we have for any ABox $\Amc$, $\Tmc_{\exists,r},\Amc\models q(a)$ iff $A(a)\in \Amc$
or there are $r(a,b),A(b)\in \Amc$. 

\smallskip 
\noindent  
 (2) Let $\Tmc_{\exists,l} = \{ \exists r.A\sqsubseteq A\}$
and $q(x) = A(x)$. For any ABox $\Amc$, $\Tmc_{\exists,l},\Amc\models q(a)$ iff
there is an $r$-path in $\Amc$ from $a$ to some $b$ with $A(b)\in \Amc$;
that is, there are $r(a_{0},a_{1}),\dots,r(a_{n-1},a_{n})\in \Amc$, $n \geq 0$, 
with $a_{0}=a$, $a_{n}=b$, and $A(b)\in \Amc$.

\smallskip  
\noindent (3)
Consider an
undirected graph $G$ represented as an ABox $\Amc$ with assertions
$r(a,b),r(b,a) \in \Amc$ iff there is an edge between $a$ and $b$. Let
$A_{1},\ldots,A_{k},M$ be concept names. Then $G$ is
$k$-colorable iff $\Tmc_{k},\Amc\not\models \exists x \, M(x)$, where
$$
\begin{array}{rcl}
\Tmc_{k} & = & \{A_{i} \sqcap A_{j} \sqsubseteq M\mid 1\leq i<j\leq k\} \, \cup\\[0.5mm]
        &    & \{ A_{i} \sqcap \exists r.A_{i} \sqsubseteq M \mid 1\leq i\leq k\}\, \cup\\[0.5mm]
        &     &\{\top \sqsubseteq \midsqcup_{1\leq i \leq k}A_{i}\}. 
\end{array}
$$
\end{example}
Instead of actually computing certain answers to queries, we
concentrate on the query evaluation problem, which is the decision
problem version of query answering. We next introduce this problem
along with associated notions of complexity. An
\emph{ontology-mediated query (OMQ)} is a pair $(\Tmc,q(\vec{x}))$
with \Tmc a TBox $\Tmc$ and $q(\vec{x})$ a query.  The \emph{query
  evaluation problem for $(\Tmc,q(\vec{x}))$} is to decide, given an
ABox $\Amc$ and $\vec{a}$ in $\mn{Ind}(\Amc)$, whether
$\Tmc,\Amc\models q(\vec{a})$.  We shall typically be interested in
joint complexity bounds for evaluating \emph{all} OMQs formulated in a
query language \Qmc of interest w.r.t.\ a given TBox \Tmc. 
\begin{definition}\label{def:main}
Let $\Tmc$ be an $\mathcal{ALCFI}$-TBox and let $\Qmc \in \{ \text{CQ},
\text{PEQ}, \text{ELIQ}, \text{ELQ} \}$. Then
\begin{itemize}
\item \emph{\Qmc-evaluation w.r.t.\ $\Tmc$ is in {\sc PTime}} if for
  every $q(\vec{x}) \in \Qmc$, the query evaluation problem for $(\Tmc,q(\vec{x}))$ is in \PTime.
\item \emph{\Qmc-evaluation w.r.t.\ $\Tmc$ is {\sc coNP}-hard} if there exists
  $q(\vec{x}) \in \Qmc$ such that the query evaluation problem for $(\Tmc,q(\vec{x}))$
  is {\sc coNP}-hard.
\end{itemize}
\end{definition}
Note that one should not think of `\Qmc-evaluation w.r.t.\ \Tmc' as a
decision problem since, informally, this is a collection of infinitely
many decision problems, one for each query in \Qmc. Instead, one 
should think of `\Qmc-evaluation w.r.t.\ \Tmc to be in {\sc PTime}'
(or {\sc coNP}-hard) as a property of \Tmc. 
\begin{example}\label{ex3}
~

\smallskip 
\noindent 
(1) PEQ-evaluation w.r.t.~the TBoxes $\Tmc_{\exists,r}$ and
  $\Tmc_{\exists,l}$ from Example~\ref{ex1} is in {\sc PTime}.  This
  follows from the fact that these TBoxes are $\mathcal{EL}$-TBoxes
  (TBoxes using only $\mathcal{EL}$-concepts) and it is well known
  that PEQ-evaluation w.r.t.~$\mathcal{EL}$-TBoxes is in {\sc PTime}
  \cite{DBLP:conf/lpar/KrisnadhiL07}. 




\smallskip 
\noindent 
(2) Consider the TBoxes $\Tmc_{k}$ from Example~\ref{ex1} that express
$k$-colorability using the query $\exists x\,M(x)$. For $k\geq 3$, CQ-evaluation w.r.t.~$\Tmc_{k}$ is
{\sc coNP}-hard since $k$-colorability is {\sc NP}-hard. However, in
contrast to the tractability of $2$-colorability, CQ-evaluation 
w.r.t.~$\Tmc_{2}$ is still {\sc coNP}-hard. 
This follows from Theorem~\ref{thm:nomatlower} below and, intuitively,
is the case because $\Tmc_{2}$ `entails a disjunction': for $\Amc =
\{B(a)\}$, we have $\Tmc_{2},\Amc\models A_{1}(a)\vee A_{2}(a)$, but
neither $\Tmc_{2},\Amc\models A_{1}(a)$ nor $\Tmc_{2},\Amc\models
A_{2}(a)$.  
\end{example}
%
%
%
In addition to the classification of TBoxes according to whether query
evaluation is in {\sc PTime} or {\sc coNP}-hard, we are also
interested in whether OMQs based on the TBox are rewritable into more
classical database querying languages, in particular into Datalog and
into monadic Datalog.

%
%
%
%
A \emph{Datalog rule} $\rho$ has the form $ S(\vec{x}) \leftarrow
R_{1}(\vec{y}_{1})\wedge \cdots \wedge R_{n}(\vec{y}_{n}) $ where
$n>0$, $S$ is a relation symbol, and $R_{1},\ldots,R_{n}$ are relation
symbols, that is, concept names and role names. We refer to
$S(\vec{x})$ as the \emph{head} of $\rho$ and
$R_{1}(\vec{y}_{1})\wedge \cdots \wedge R_{n}(\vec{y}_{n})$ as its
\emph{body}. Every variable in the head of $\rho$ is required to occur
also in its body.  A \emph{Datalog program} $\Pi$ is a finite set of
Datalog rules with a selected \emph{goal relation} $\mn{goal}$ that
does not occur in rule bodies.  Relation symbols that occur in the
head of at least one rule are called \emph{intensional relation
  symbols (IDBs)}, the remaining symbols are called \emph{extensional
  relation symbols (EDBs)}. Note that, by definition, $\mn{goal}$ is an
IDB. The \emph{arity} of the program is the arity of the \mn{goal}
relation. Programs of arity zero are called \emph{Boolean}. A Datalog
program that uses only IDBs of arity one, with the possible exception
of the \mn{goal} relation, is called \emph{monadic}.

For an ABox $\Amc$, a Datalog program $\Pi$, and $\vec{a}$ from
$\mn{Ind}(\Amc)$ of the same length as the arity of $\mn{goal}$, we
write $\Amc\models \Pi(\vec{a})$ if $\Pi$ returns $\vec{a}$ as an
answer on $\Amc$, defined in the usual way
\cite{DBLP:journals/tkde/CeriGT89}. A (monadic) Datalog program $\Pi$
is a \emph{(monadic) Datalog-rewriting} of an OMQ $(\Tmc,q(\vec{x}))$
if for all ABoxes $\Amc$ and $\vec{a}$ from $\mn{Ind}(\Amc)$,
$\Tmc,\Amc\models q(\vec{a})$ iff $\Amc\models \Pi(\vec{a})$. In this
case the OMQ $(\Tmc,q(\vec{x}))$ is called \emph{(monadic)
  Datalog-rewritable}.  When working with DLs such as
$\mathcal{ALCFI}$ that include functional roles, it is more natural to
admit the use of inequalities in the bodies of Datalog rules instead
of working with `pure' programs. We refer to such extended programs as
(monadic) Datalog$^{\neq}$ programs and accordingly speak of
(monadic) Datalog$^{\neq}$-rewritability.
\begin{example}
\label{exa:newlabel}
~

\smallskip 
\noindent 
(1) The OMQ $(\Tmc_{\exists,l},A(x))$ from Example~\ref{ex1}
    expressing a form of reachability is
    rewritable into the monadic Datalog program
$$
\mn{goal}(x) \leftarrow P(x), \quad P(x) \leftarrow A(x), \quad P(x)\leftarrow r(x,y)\wedge P(y).
$$

\smallskip 
\noindent 
(2) The OMQ $(\Tmc_{k},\exists x \, M(x))$ from Example~\ref{ex1} is
Datalog-rewritable when $k= 2$ since non-2-colorability can be
expressed by a Datalog program (but not as a monadic one). For $k\geq
3$, non-$k$-colorability cannot be expressed by a Datalog program (in
fact, not even by a Datalog$^{\not=}$ program)
\cite{DBLP:conf/pods/AfratiCY91}.

\smallskip 
\noindent 
(3) The OMQ $(\{\mn{func}(r)\},\exists x \,  M(x))$ is rewritable into
the monadic Datalog$^{\neq}$ program 
$$
\mn{goal}() \leftarrow r(x,y_1) \wedge r(x,y_2) \wedge y_1 \neq y_2 , \quad
\mn{goal}() \leftarrow M(x)
$$
but is not rewritable into pure Datalog.
\end{example}
\begin{definition}\label{def:main2}
  Let $\Tmc$ be an $\mathcal{ALCFI}$-TBox and let $\Qmc \in \{
  \text{CQ}, \text{PEQ}, \text{ELIQ}, \text{ELQ} \}$. Then $\Tmc$ is
  \emph{(monadic) Datalog$^{\neq}$-rewritable for $\Qmc$} if
  $(\Tmc,q(\vec{x}))$ is (monadic) Datalog$^{\neq}$-rewritable for
  every $q(\vec{x}) \in \Qmc$.
\end{definition}
We would like to stress that the extension of Datalog to
Datalog$^{\not=}$ makes sense only in the presence of functional
roles. In fact, it follows from the CSP connection established in
Section~\ref{sect:dicho} and the results in
\cite{DBLP:conf/lics/FederV03} that for \ALCI-TBoxes,
Datalog$^{\not=}$-rewritability for \Qmc agrees with
Datalog-rewritability for \Qmc, for all query classes \Qmc
considered in this paper.
\begin{example}
  It is folklore that every \EL-TBox is monadic Datalog-rewritable for
  ELQ, ELIQ, CQs, and PEQs. Thus, this applies in particular to the
  \EL-TBoxes $\Tmc_{\exists,l}$ and $\Tmc_{\exists,r}$ from
  Example~\ref{ex1}.  A concrete construction of Datalog-rewritings
  for ELIQs can be found in the proof of
  Theorem~\ref{thm:unravupperlem:hornalclem:hornalc} below.  In
  contrast, the \ALC-TBox $\Tmc_k$ from Example~\ref{ex1} is not
  Datalog$^{\not=}$-rewritable for ELQ when $k \geq 3$ since the OMQ
  $(\Tmc_{k},\exists x \, M(x))$ is not Datalog-rewritable,
  by Example~\ref{exa:newlabel} (2).
\end{example}

Datalog$^{\neq}$-programs can be evaluated in \PTime
\cite{DBLP:journals/tkde/CeriGT89} in data complexity, and thus
Datalog$^{\neq}$-rewritability for \Qmc of a TBox \Tmc implies that
$\Qmc$-evaluation w.r.t.\ \Tmc is in \PTime in data complexity. We
shall see later that the converse direction does not hold in general.

\smallskip 

We will often be concerned with homomorphisms between ABoxes and
between interpretations, defined next. Let \Amc and \Bmc be ABoxes.  A
function $h:\mn{Ind}(\Amc) \rightarrow \mn{Ind}(\Bmc)$ is a
\emph{homomorphism from \Amc to~\Bmc} if it satisfies the following 
conditions:
\begin{enumerate}

\item $A(a) \in \Amc$ implies $A(h(a)) \in \Bmc$ and

\item $r(a,b) \in \Amc$ implies $r(h(a),h(b)) \in \Bmc$.

\end{enumerate}
We say that \emph{$h$ preserves} $I \subseteq \NI$ if $h(a)=a$ for all
$a \in I$. Homomorphisms from an interpretation \Imc to an
interpretation \Jmc are defined analogously as functions
$h:\Delta^\Imc \rightarrow \Delta^\Jmc$. Note that these two notions
are in fact identical since, up to presentation, ABoxes and finite
interpretations are the same thing. In what follows we will not always
distinguish between the two presentations.

\section{Materializability}
\label{sect:mat}
We introduce materializability as a central notion for analyzing the
complexity and rewritability of TBoxes. A materialization of a TBox
$\Tmc$ and ABox $\Amc$ for a class of queries $\Qmc$ is a (potentially
infinite) model of $\Tmc$ and $\Amc$ that gives the same answers to
queries in $\Qmc$ as $\Tmc$ and $\Amc$ do. It is not difficult to see
that a materialization for ELIQs is not necessarily a materialization
for CQs and that a materialization for ELQs is not necessarily a
materialization for ELIQs. We shall call a TBox $\Tmc$ materializable
for a query language $\Qmc$ if for every ABox $\Amc$ that is
consistent w.r.t.~$\Tmc$, there is a materialization of $\Tmc$ and \Amc
for $\Qmc$. Interestingly, we show that materializability of
\ALCFI-TBoxes does \emph{not} depend on whether one considers ELIQs,
CQs, or PEQs.  This result allows us to simply talk about
materializable TBoxes, independently of the query language
considered. The fundamental result linking materializability of a TBox
to the complexity of query evaluation is that ELIQ-evaluation is {\sc
  coNP}-hard w.r.t.\ non-materializable \ALCFI-TBoxes. As another
application of materializability, we show that for \ALCFI-TBoxes, {\sc
  PTime} query evaluation, {\sc coNP}-hardness of query evaluation,
and Datalog$^{\not=}$-rewritability also do not depend on the query
language.  In the case of \ALCF, materializability for ELIQs
additionally coincides with materializability for ELQs.
\begin{definition}\label{cond0}
  Let $\Tmc$ be an $\mathcal{ALCFI}$-TBox and $\Qmc \in \{ \text{CQ},
  \text{PEQ}, \text{ELIQ}, \text{ELQ} \}$. Then
  \begin{enumerate}
    
  \item a model $\Imc$ of \Tmc and an ABox \Amc is a
    \emph{$\Qmc$-materialization of $\Tmc$ and}~$\Amc$ if for all
    queries $q(\vec{x})\in \Qmc$ and $\vec{a}
    \subseteq {\sf Ind}(\Amc)$, we have $\mathcal{I} \models
    q(\vec{a})$ iff \mbox{$\Tmc,\Amc\models q(\vec{a})$};

  \item $\Tmc$ is \emph{$\Qmc$-materializable} if for every
  ABox $\Amc$ that is consistent w.r.t.~$\Tmc$, there exists a
  $\Qmc$-materialization of $\Tmc$ and~$\Amc$.
  \end{enumerate}
\end{definition}
In Point~(1) of Definition~\ref{cond0}, it is important that the
materialization $\Imc$ of $\Tmc$ and $\Amc$ is a model of $\Tmc$ and
$\Amc$. In fact, for an ABox $\Amc$ that is consistent w.r.t.~$\Tmc$,
we can \emph{always} find an interpretation $\Imc$ such that for every
CQ $q(\vec{x})$ and $\vec{a} \subseteq \mn{Ind}(\Amc)$, $\Imc\models
q(\vec{a})$ iff $\Tmc,\Amc\models q(\vec{a})$. In particular, the
direct product of all (up to isomorphisms) countable models of $\Tmc$
and $\Amc$ can serve as such an \Imc. However, the interpretation is,
in general, not a model of~\Tmc.

\smallskip

Note that a \Qmc-materialization can be viewed as a more abstract
version of the \emph{canonical} or \emph{minimal} or \emph{universal} model
as often used in the context of `Horn DLs' such as \EL and DL-Lite
\cite{DBLP:conf/ijcai/LutzTW09,DBLP:conf/kr/KontchakovLTWZ10,DBLP:conf/rweb/BienvenuO15}
and more expressive ontology languages based on tuple-generating dependencies (tgds) \cite{DBLP:journals/jair/CaliGK13}
as well as in data exchange~\cite{DBLP:journals/tcs/FaginKMP05}.
In fact, the ELQ-materialization in the next example is exactly 
the `compact canonical model' from~\cite{DBLP:conf/ijcai/LutzTW09}.
\begin{example}\label{ex5} 
~

\smallskip 
\noindent 
(1) Let $\Tmc_{\exists,l} = \{ \exists r .A \sqsubseteq A\}$ be as in
Example~\ref{ex1} and let $\Amc$ be an ABox. Let $\Imc$ be the
interpretation obtained from $\Amc$ by adding to $A^\Imc$ all $a\in
\mn{Ind}(\Amc)$ such that there exists an $r$-path from $a$ to some
$b$ with $A(b)\in \Amc$. Then $\Imc$ is a PEQ-materialization of
$\Tmc$ and $\Amc$ and so $\Tmc$ is PEQ-materializable.

\smallskip 
\noindent 
(2) Let $\Tmc_{\exists,r} = \{A \sqsubseteq \exists r.A\}$ be as in
Example~\ref{ex1} and let $\Amc$ be an ABox with at least one
assertion of the form $A(a)$. To obtain an ELQ-materialization $\Imc$
of \Tmc and \Amc, start with \Amc as an interpretation, add a fresh
domain element $d_{r}$ to $\Delta^\Imc$ and to $A^\Imc$, and extend
$r^\Imc$ with $(a,d_r)$ and $(d_r,d_r)$ for all $A(a) \in \Amc$.  Thus
$\Tmc_{\exists,r}$ is ELQ-materializable.

\smallskip
\noindent
(3) The TBox $\Tmc=\{A \sqsubseteq A_{1}\sqcup A_{2}\}$ is not ELQ-materializable.
To see this let $\Amc=\{A(a)\}$. Then no model $\Imc$ of $\Tmc$ and $\Amc$ is an ELQ-materialization
of $\Tmc$ and $\Amc$ as it satisfies $a\in A_{1}^{\Imc}$ or $a\in A_{2}^{\Imc}$ but
neither $\Tmc,\Amc\models A_{1}(a)$ nor $\Tmc,\Amc\models A_{2}(a)$. 
\end{example}
Trivially, every PEQ-materialization is a CQ-materialization, every CQ-materialization is
an ELIQ-materialization and every ELIQ-materialization is an ELQ-materialization. 
Conversely, it follows directly from the fact that each PEQ is equivalent to a disjunction of CQs
that every CQ-materializ\-ation is also a PEQ-materialization.
In contrast, the following example demonstrates that ELQ-materializ\-ations are
different from ELIQ-materializations. A similar argument separates
ELIQ-materializations from CQ-materializations.
\begin{example}\label{ex5b} Let $\Tmc_{\exists,r}$ be as in
    Example~\ref{ex5},
$$
\begin{array}{rcl}
  \Amc&=&\{B_{1}(a),B_{2}(b),A(a),A(b)\} \text{ and}\\[1mm]
  q(x)&=& (B_{1}\sqcap \exists r.\exists r^{-}.B_{2})(x),
\end{array}
$$
Then the ELQ-materialization $\mathcal{I}$ from Example~\ref{ex5} (2) is not 
a \Qmc-materialization
for any $\Qmc$ from the set of query languages 
$\text{ELIQ}, \text{CQ}, \text{PEQ}$. For example, 
we have
$\Imc\models q(a)$, but $\Tmc,\Amc \not\models q(a)$.
An ELIQ/CQ/PEQ-materialization of $\Tmc$ and $\Amc$ is obtained by
unfolding $\Imc$ (see below): instead of using only one additional
domain element $d_{r}$ as a witness for $\exists r.A$, we attach to both
$a$ and $b$ an infinite $r$-path of elements that satisfy $A$.  Note
that every CQ/PEQ-materialization of $\Tmc_{\exists,r}$ and $\Amc$ must be
infinite.  
\end{example}
We will sometimes restrict our attention to materializations $\Imc$ that
are countable and \emph{generated}, i.e, every $d\in \Delta^{\Imc}$ is reachable from some 
$a\in \Delta^\Imc \cap \NI$ in the undirected graph 
$$
G_{\Imc} = (\Delta^{\Imc},\{ \{d,d'\} \mid (d,d') \in \bigcup_{r \in \NR} r^\Imc \}).
$$
The following lemma shows that we can make that assumption without loss of generality.
\begin{lemma}\label{lem:count}
Let $\Tmc$ be an $\mathcal{ALCFI}$-TBox, $\Amc$ an ABox, and $\Qmc \in \{ \text{CQ},
  \text{PEQ}, \text{ELIQ}, \text{ELQ} \}$. If $\Imc$ is a $\Qmc$-materialization
of $\Tmc$ and $\Amc$, then there exists a subinterpretation $\Jmc$ of $\Imc$ that
is a countable and generated $\Qmc$-materialization of $\Tmc$ and $\Amc$.
\end{lemma}
\begin{proof}
  Let $\Imc$ be a $\Qmc$-materialization of $\Tmc$ and
  $\Amc$. To construct $\Jmc$ we apply a
  standard selective filtration procedure to $\Imc$. More precisely,
  we identify a sequence ${\sf Ind}(\Amc) = S_0 \subseteq S_1
  \subseteq \cdots \subseteq \Delta^{\Imc}$ and then define $\Jmc$
  to be the restriction of $\Imc$ to $\bigcup_i S_i$. Let \Cmc be
  the set of all concepts of the form $\exists r . C$ that occur in
  $\Tmc$ and of all concepts $\exists r . \neg C$ such that $\forall r
  . C$ occurs in \Tmc. Assume $S_{i}$ has already been defined. Then
  define $S_{i+1}$ as the union of $S_{i}$ and, for every $d\in S_{i}$
  and concept $\exists r.C \in \Cmc$ with $d\in (\exists
  r.C)^{\Imc}$, an arbitrary $d'\in \Delta^{\Imc}$ with $(d,d')\in
  r^{\Imc}$ and $d'\in C^{\Imc}$ (unless such a $d'$ exists
  already in $S_{i}$). 
  It is easy to see that $\Jmc$ is a countable and generated $\Qmc$-materialization of $\Tmc$ and $\Amc$.
\end{proof}
%
%

\subsection{Model-Theoretic Characterizations of Materializability}

We characterize materializations using simulations and homomorphisms.
This sheds light on the nature of materializations and establishes a
close connection between materializations and initial models as
studied in model theory, algebraic specification, and logic
programming~\cite{Malcev,MesGog,DBLP:journals/jcss/Makowsky87}.
  
A \emph{simulation from an interpretation $\Imc_1$ to an interpretation $\Imc_2$} is a 
relation $S \subseteq \Delta^{\Imc_1} \times \Delta^{\Imc_2}$ such that 
%
%
%
\begin{enumerate}
\item for all $A\in \NC$: if $d_{1}\in A^{\Imc_{1}}$ and $(d_{1},d_{2})\in S$, then $d_{2}\in A^{\Imc_{2}}$; 
\item for all $r\in \NR$: if $(d_{1},d_{2})\in S$ and $(d_{1},d_{1}')\in r^{\Imc_{1}}$,
then there exists $d_{2}' \in \Delta^{\Imc_{2}}$ such that $(d_{1}',d_{2}')\in S$ and 
$(d_{2},d_{2}')\in r^{\Imc_{2}}$;
\item for all $a\in \Delta^{\Imc_1} \cap \NI$: $a\in \Delta^{\Imc_2}$ and 
$(a,a) \in S$.
\end{enumerate}
Note that, by Condition~(3), domain elements that are individual names
need to be respected by simulations while other domain elements need
not. In database parlance, the latter are thus treated as
\emph{labeled nulls}, that is, while their existence is important,
their identity is not.  

We call a simulation $S$ an \emph{i-simulation} if Condition~(2) is satisfied also
for inverse roles. Note that $S$ is a homomorphism preserving $\Delta^{\Imc_1} \cap \NI$ if 
$S$ is a function with domain $\Delta^{\Imc}$. We remind the reader of the following
characterizations of ELQs using simulations, ELIQs using
i-simulations, and CQs using homomorphisms (see e.g.\ \cite{DBLP:journals/jsc/LutzW10}).  An
interpretation $\Imc$ has \emph{finite outdegree} if the undirected
graph $G_{\Imc}$ has finite outdegree. 
\begin{lem}\label{lem:simelq}
Let $\Imc$ and $\Jmc$ be interpretations such that $\Delta^\Imc\cap \NI$ is finite, $\Imc$
is countable and generated, and $\Jmc$ has finite outdegree. Then the following conditions are equivalent
(where none of the assumed conditions on $\Imc$ and $\Jmc$ is required for (2) $\Rightarrow$ (1)).
\begin{enumerate}
\item For all ELIQs $C(x)$ and $a\in \Delta^\Imc\cap \NI$: if $\Imc\models C(a)$, then $\Jmc\models C(a)$;
\item There is an i-simulation from $\Imc$ to $\Jmc$.
\end{enumerate}
The same equivalence holds when ELIQs and i-simulations are replaced
by ELQs and simulations, respectively.  Moreover, the following conditions are
equivalent (where none of the assumed conditions on $\Imc$ and
$\Jmc$ is required for (5) $\Rightarrow$ (3)).
\begin{enumerate}
\item[(3)] For all PEQs $q(\vec{x})$ and $\vec{a} \subseteq \Delta^\Imc\cap \NI$: if $\Imc\models q(\vec{a})$, then $\Jmc\models q(\vec{a})$;
\item[(4)] For all CQs $q(\vec{x})$ and $\vec{a} \subseteq \Delta^\Imc\cap \NI$: if $\Imc\models q(\vec{a})$, then $\Jmc\models q(\vec{a})$;
\item[(5)] There is a homomorphism from $\Imc$ to $\Jmc$ preserving $\Delta^\Imc\cap \NI$.
\end{enumerate}
\end{lem}
\begin{proof}
  We prove the equivalence of (3)-(5). The equivalence of (1) and (2)
  is similar (both for ELIQs and ELQs) but simpler and left to the
  reader. The implication (3) $\Rightarrow$ (4) is trivial. For the
  proof of (4) $\Rightarrow$ (5), assume that $\Imc$ is countable and
  generated and let $\Jmc$ have finite outdegree. We first assume that
  only a finite set $\Sigma$ of concept and role names have a
  non-empty interpretation in $\Imc$ and then generalize the result to
  arbitrary $\Imc$.  Assume that (4) holds. First observe that for
  every finite subset $X$ of $\Delta^{\Imc}$ there is a homomorphism
  $h_{X}$ preserving $X\cap \NI$ from the subinterpretation
  $\Imc_{\restriction X}$ of $\Imc$ induced by $X$ into $\Jmc$:
  associate with every $d\in X$ a variable $x_{d}$ and
  regard $\Imc_{\restriction X}$ as the CQ 
$$
q_{X}(\vec{x}) =
\exists \vec{y} \,
 \bigwedge_{d\in X \cap A^{\Imc}} A(x_{d}) \wedge 
\bigwedge_{(d,d')\in (X\times X) \cap r^{\Imc}}r(x_{d},x_{d'}),
$$
where $\vec{x}$ comprises the variables in $\{x_{a} \mid a \in X\cap \NI\}$ and
$\vec{y}$ comprises the variables $x_{d}$ with $d\in X\setminus \NI$
($q_{X}$ is a CQ by our assumption that only finitely many concept and role names have non-empty interpretation).
For the assignment $\pi(x_{d})=d$, we have $\Imc \models_{\pi} \varphi(\vec{x},\vec{y})$. Thus
$\Imc\models_{\pi}q_{X}(\vec{x})$ and so, by (2), $\Jmc\models_{\pi} q_{X}(\vec{x})$. 
Consequently, there exists an assignment $\pi'$ for $q_{X}(\vec{x})$ in $\Jmc$ which
coincides with $\pi$ on $\{x_{a} \mid a\in X\cap \NI\}$ such that 
$\Jmc\models_{\pi'} \varphi(\vec{x},\vec{y})$. Let 
$h_{X}(d)= \pi'(d)$ for $d\in X$. Then $h_{X}$ is a homomorphism from $\Imc_{\restriction X}$ to $\Jmc$ preserving $X\cap \NI$, 
as required.

We now lift the homomorphisms $h_{X}$ to a homomorphism $h$ from
$\Imc$ to $\Jmc$ preserving $\Delta^\Imc\cap \NI$.  Since $\Imc$ is countable and generated, there
exists a sequence $X_{0}\subseteq X_{1} \subseteq \cdots$ of finite
subsets of $\Delta^{\Imc}$ such that $X_{0}= \Delta^\Imc \cap \NI$,
$\bigcup_{i\geq 0}X_{i}= \Delta^{\Imc}$, and for all $d\in X_{i}$
there exists a path in $X_{i}$ from some $a\in X_0$ to
$d$.

By the observation above, we find homomorphisms $h_{X_{i}}$ from
$\Imc_{\restriction X_{i}}$ to $\Jmc$ preserving $X_{i}\cap \NI$, for $i\geq 0$.  Let
$d_{0},d_{1}\ldots$ be an enumeration of $\Delta^{\Imc}$. We define
the required homomorphism $h$ as the limit of a sequence
$h_{0}\subseteq h_{1}\subseteq \cdots$, where each $h_{n}$ has domain
$\{d_{0},\ldots,d_{n}\}$ and where we ensure for each $h_n$ and all
$d\in \{d_{0},\ldots,d_{n}\}$ that there are infinitely many $j$ with
$h_{n}(d)= h_{\restriction X_{j}}(d)$. Observe that since $\Jmc$ has
finite outdegree and since for all $d\in X_{i}$, there exists a path
in $X_{i}$ from some $a\in X_0$ to $d$, for each $d\in \Delta^{\Imc}$
there exist only finitely many distinct values in $\{h_{\restriction
  X_{i}}(d) \mid i\geq 0\}$.  By the pigeonhole principle, there thus
exist infinitely many $j$ with the same value $h_{X_{j}}(d)$.  For
$h_{0}(d_{0})$ we take such a value for $d_{0}$. Assume $h_{n}$ has
been defined and assume that the set $I=\{ j \mid h_{n}(d)=
h_{\restriction X_{j}}(d) \mbox{ for all }d\in
\{d_{0},\ldots,d_{n}\}\}$ is infinite.  Again by the pigeonhole
principle, we find a value $e\in \Delta^{\Jmc}$ such that
$h_{X_{j}}(d_{n+1})=e$ for infinitely many $j\in I$. We set
$h_{n+1}(d_{n+1})=e$. The function $h=\bigcup_{i\geq 0}h_{0}$ is a
homomorphism from $\Imc$ to $\Jmc$ preserving $\Delta^\Imc\cap \NI$, as required.

To lift this result to arbitrary interpretations $\Imc$, it is
sufficient to prove that the homomorphisms $h_{X}$ still exist. This
can be shown using again the pigeonhole principle. Let
$X\subseteq \Delta^{\Imc}$ be finite.  We may assume that for each
$d\in X$, there exists a path in $X$ from some $a\in X \cap \NI$,
to $d$.  We have shown that for each finite set $\Sigma$ of concept
and role names, there exists a homomorphism $h_{X}^{\Sigma}$ from the
$\Sigma$-reduct $\Imc_{X}^{\Sigma}$ of $\Imc_{X}$ to $\Jmc$
($\Imc_{X}^{\Sigma}$ interprets only the symbols in $\Sigma$ as
non-empty). Since $\Jmc$ has finite outdegree, infinitely many
$h_{X}^{\Sigma}$ coincide.  A straightforward modification of the
pigeonhole argument above can now be used to construct the required
homomorphism $h_{X}$.
 
For the proof of (5) $\Rightarrow$ (3), assume $\Imc\models
q(\vec{a})$ and let $h$ be a homomorphism from $\Imc$ to $\Jmc$ preserving $\Delta^\Imc\cap \NI$. Let
$\pi$ be an assignment for $q(\vec{x})$ in $\Imc$ witnessing
$\Imc\models q(\vec{a})$.  Then the composition $h\circ \pi$ is an
assignment for $q(\vec{x})$ in $\Jmc$ witnessing $\Jmc\models
q(\vec{a})$.
\end{proof}
For the next steps, we need some observations regarding the unfolding
of interpretations into forest-shaped interpretations. Let us first make
precise what we mean by unfolding. The \emph{i-unfolding} of an
interpretation \Imc is an interpretation \Jmc defined as follows.  The
domain $\Delta^{\Jmc}$ of $\Jmc$ consists of all words
$d_{0}r_{1}\ldots r_{n}d_{n}$ with $n\geq 0$, each $d_{i}$ from
$\Delta^{\Imc}$ and each $r_{i}$ a (possibly inverse) role such that
\begin{itemize}
\item[(a)] $d_{i} \in \NI$ iff $i=0$;
\item[(b)] 
$(d_{i},d_{i+1})\in r_{i+1}^{\Imc}$   for $0\leq i <n$; 
\item[(c)] 
if $r_{i}^{-}=r_{i+1}$, then $d_{i-1}\not=d_{i+1}$ for $0<i<n$.
\end{itemize}
For $d_{0}\cdots d_{n}\in \Delta^{\Jmc}$, we set ${\sf tail}(d_{0}\cdots d_{n})=d_{n}$.
Now set
$$
\begin{array}{rcll}
  A^{\Jmc} &=& \{w \in \Delta^{\Jmc} \mid {\sf tail}(w)\in A^{\Imc}\}  
  & \text{ for all } A\in \NC \\[1mm]
  r^{\Jmc} & = & (r^\Imc \cap (\NI \times \NI)) \,\cup\\[1mm]
  && \{ (\sigma,\sigma rd) \mid \sigma,\sigma rd \in \Delta^{\Jmc} \} \,
\cup 
\{(\sigma r^{-} d,\sigma)\mid \sigma,\sigma r^{-} d \in  \Delta^{\Jmc}\}
   & \text{ for all } r\in \NR. 
\end{array}
$$
We say that an interpretation $\Imc$ is \emph{i-unfolded} if it is
isomorphic to its own i-unfolding. Clearly, every i-unfolding
of an interpretation is i-unfolded.

For $\mathcal{ALCF}$-TBoxes, it is not required to unfold along
inverse roles. This is reflected in the \emph{unfolding} of an
interpretation \Imc, where in contrast to the i-unfolding we use as
the domain the set of all words $d_{0}r_{1}\ldots r_{n}d_{n}$ with $n\geq
0$, each $d_{i}$ from $\Delta^{\Imc}$, and each $r_{i}$ a \emph{role
  name} such that Conditions~(a) and (b) above are satisfied. The
interpretation of concept and role names remains the
same. We call an interpretation $\Imc$ \emph{unfolded} if it is
isomorphic to its own unfolding. The following lemma summarizes the
main properties of unfoldings. Its proof is straightforward and left
to the reader.
\begin{lemma}\label{lem:propertiesofunfold}
  Let $\Imc$ be an interpretation, $\Imc^i$ its i-unfolding, and
  $\Imc^u$ its unfolding.  Then for every interpretation \Jmc, the
  following conditions are satisfied:
\begin{enumerate}
\item the function $f(w):={\sf tail}(w)$, $w\in \Delta^{\Imc^i}$, is a homomorphism from $\Imc^i$
to $\Imc$ preserving $\Delta^\Imc\cap \NI$;
\item the function $f(w):={\sf tail}(w)$, $w\in \Delta^{\Imc^u}$, is a homomorphism from $\Imc^u$
to $\Imc$ preserving $\Delta^\Imc\cap \NI$;
\item if there is an i-simulation from
$\Imc$ to $\Jmc$, then there is a homomorphism from $\Imc^i$ to $\Jmc$ preserving $\Delta^\Imc\cap \NI$;
\item if there is a simulation  from
$\Imc$ to $\Jmc$, then there is a homomorphism from $\Imc^u$ to $\Jmc$ preserving $\Delta^\Imc\cap \NI$;
\item if $\Imc$ is a model of $\Tmc$ and $\Amc$ with $\Tmc$ an $\mathcal{ALCFI}$-TBox, 
then $\Imc^i$ is a model of $\Tmc$ and $\Amc$;
\item if $\Imc$ is a model of $\Tmc$ and $\Amc$ with $\Tmc$ an $\mathcal{ALCF}$-TBox, 
then $\Imc^u$ is a model of $\Tmc$ and $\Amc$.
\end{enumerate}
\end{lemma}
An interpretation $\Imc$ is called \emph{hom-initial in a class
  $\mathbb{K}$ of interpretations} if for every $\Jmc\in \mathbb{K}$,
there exists a homomorphism from $\Imc$ to $\Jmc$ preserving $\Delta^\Imc\cap \NI$.
$\Imc$ is called \emph{sim-initial (i-sim-initial) in a class
  $\mathbb{K}$ of interpretations} if for every $\Jmc\in \mathbb{K}$,
there exists a simulation (i-simulation) from $\Imc$ to $\Jmc$.  The
following theorem provides the announced characterization of
materializations in terms of simulations and homomorphisms. In the
following, the class of all models of $\Tmc$ and $\Amc$ is denoted by
${\sf Mod}(\Tmc,\Amc)$.
%
%
\begin{theorem}\label{lem:sem}
Let $\Tmc$ be an $\mathcal{ALCFI}$-TBox, $\Amc$ an ABox, and let
$\Imc\in {\sf Mod}(\Tmc,\Amc)$ be countable and generated. Then \Imc is
\begin{enumerate}
\item an ELIQ-materialization of \Tmc and \Amc iff it is i-sim-initial in 
${\sf Mod}(\Tmc,\Amc)$;
\item a CQ-materialization of \Tmc and \Amc iff it is a
  PEQ-materialization of \Tmc and \Amc iff it is hom-initial in ${\sf
    Mod}(\Tmc,\Amc)$;
\item an ELQ-materialization of \Tmc and \Amc iff it is sim-initial in
  ${\sf Mod}(\Tmc,\Amc)$, provided that $\Tmc$ is an $\mathcal{ALCF}$-TBox.
\end{enumerate}
The `only if' directions of all three points hold without any of the assumed
conditions on $\Imc$.
\end{theorem}
\begin{proof} We show that  (1) follows from Lemma~\ref{lem:simelq} and
  Lemma~\ref{lem:propertiesofunfold}; (2) and (3) can be proved similarly.

  We start with the direction from right to left. Assume that $\Imc$
  is i-sim-initial in ${\sf Mod}(\Tmc,\Amc)$. Since $\Imc$ is a model
  of $\Tmc$ and $\Amc$, we have $\Imc\models C(a)$ whenever
  $\Tmc,\Amc\models C(a)$ for any ELIQ $C(x)$ and $a\in
  \mn{Ind}(\Amc)$. Conversely, if $\Tmc,\Amc\not\models C(a)$ then
  there exists a model $\Jmc$ of $\Tmc$ and $\Amc$ such that
  $\Jmc\not\models C(a)$.  There is an i-simulation from $\Imc$ to
  $\Jmc$. Thus, by the implication (2) $\Rightarrow$ (1) from
  Lemma~\ref{lem:simelq}, we have $\Imc\not\models C(a)$ as required.

  For the direction from left to right, assume that $\Imc$ is a
  materialization of $\Tmc$ and $\Amc$ and take a model $\Jmc$ of
  $\Tmc$ and $\Amc$. We have to construct an i-simulation from $\Imc$
  to $\Jmc$.  It actually suffices to construct an i-simulation from
  \Imc to the i-unfolding $\Jmc^i$ of $\Jmc$: by Point~(3) of
  Lemma~\ref{lem:propertiesofunfold}, there is a homomorphism from
  $\Jmc^i$ to $\Jmc$ and the composition of an i-simulation with a
  homomorphism is again an i-simulation.

  To obtain an i-simulation from \Imc to $\Jmc^i$, we first identify a
  subinterpretation $\Jmc'$ of $\Jmc^i$ that has finite outdegree and
  is still a model of $\Tmc$ and $\Amc$.  By the implication (1)
  $\Rightarrow$ (2) from Lemma~\ref{lem:simelq} and since \Imc is a
  materialization, there must then be an i-simulation from $\Imc$ to
  $\Jmc'$. Clearly this is also an i-simulation from $\Imc$ to
  $\Jmc^i$ and we are done. 

  It thus remains to construct $\Jmc'$, which is done by applying 
	selective filtration to $\Jmc^{i}$ in exactly the same way as in the proof of Lemma~\ref{lem:count}.
  It can be verified that the outdegree of the resulting subinterpretation $\Jmc'$ of $\Jmc^{i}$
	is bounded by $|\Tmc|+|\Amc|$ and, therefore, finite. By construction, $\Jmc'\in {\sf Mod}(\Tmc,\Amc)$. 
\end{proof}
  The following example shows that the generatedness condition in  
  Theorem~\ref{lem:sem} cannot be dropped.  
We leave it open whether the same is true for countability.
\begin{example}
Let
$
\Tmc = \{A\sqsubseteq \exists r.A,B \sqsubseteq A\}$
and $\Amc=\{B(a)\}$ and consider the interpretation \Imc defined by
$$
\begin{array}{rcl}
\Delta^{\Imc} &=& \{a\} \cup \{0,1,2\ldots\} \cup \{
\dots,-2',-1',0',1',2',\dots \} \\[1mm]
A^{\Imc} &=& \Delta^{\Imc} \\[1mm]
B^{\Imc} &=&  \{ a \} \\[1mm]
r^\Imc &=&\{(a,0)\}\cup \{(n,n+1)\mid n \in \mathbbm{N}\} \cup
\{(n',n'+1)\mid n \in \mathbbm{Z}\}.
\end{array}
$$
Then \Imc is a PEQ-materialization of \Tmc and \Amc, but it is not
hom-initial (and in fact not even sim-initial) since the restriction
of \Imc to domain $\{a\} \cup \{0,1,2\ldots\}$ is also a model of
\Tmc and \Amc, but there is no homomorphism (and no simulation) from \Imc to this
restriction preserving $\{a\}$.
\end{example}
As an application of Theorem~\ref{lem:sem},
we now show that materializability coincides for the query languages PEQ, CQ, and ELIQ
(and that for $\mathcal{ALCF}$-TBoxes, all these also coincide with ELQ-materializability). 

\begin{theorem}
\label{thm:disjmat}
Let $\Tmc$ be an $\mathcal{ALCFI}$-TBox. Then the following conditions are equivalent:
\begin{enumerate}
\item $\Tmc$ is PEQ-materializable;
\item $\Tmc$ is CQ-materializable;
\item $\Tmc$ is ELIQ-materializable;
\item ${\sf Mod}(\Tmc,\Amc)$ contains an i-sim-initial $\Imc$, for every ABox $\Amc$ that is consistent w.r.t.~$\Tmc$;
\item ${\sf Mod}(\Tmc,\Amc)$ contains a hom-initial $\Imc$, for every ABox $\Amc$ that is consistent w.r.t.~$\Tmc$.
\end{enumerate}
If $\Tmc$ is an $\mathcal{ALCF}$-TBox, then the above is the case 
iff \Tmc is ELQ-materializable iff ${\sf Mod}(\Tmc,\Amc)$ contains a sim-initial 
$\Imc$, for every ABox $\Amc$ that is consistent w.r.t.~$\Tmc$.
\end{theorem}
\begin{proof}
  The implications (1) $\Rightarrow$ (2) and (2) $\Rightarrow$ (3) are
  trivial.  For (3) $\Rightarrow$ (4), let $\Imc$ be an
  ELIQ-materialization of $\Tmc$ and an ABox $\Amc$. By
  Lemma~\ref{lem:count}, we may assume that $\Imc$ is countable and
  generated. By Lemma~\ref{lem:sem}, ${\sf Mod}(\Tmc,\Amc)$ contains
  an i-sim-initial interpretation.  For (4) $\Rightarrow$ (5), assume
  that $\Imc \in {\sf Mod}(\Tmc,\Amc)$ is i-sim-initial.  By
  Points~(3) and (5) of Lemma~\ref{lem:propertiesofunfold}, the
  i-unfolding of \Imc is hom-initial in ${\sf Mod}(\Tmc,\Amc)$ and (5)
  follows.  (5) $\Rightarrow$ (1) follows from
  Theorem~\ref{lem:sem}.  The implications for $\mathcal{ALCF}$-TBoxes
  are proved similarly.
\end{proof}
Because of Theorem~\ref{thm:disjmat}, we sometimes speak of
\emph{materializability} without reference to a query language and of
\emph{materializations} instead of PEQ-materializations. 
\subsection{Materializability and {\sc coNP}-hardness}

We show that non-materializability of a TBox \Tmc implies {\sc
  coNP}-hardness of ELIQ-evaluation w.r.t.\ \Tmc.  To this end, we
first establish that materializability is equivalent to the
disjunction property, which is sometimes also called convexity and
plays a central role in attaining {\sc PTime} complexity for
subsumption in DLs \cite{BaBrLu-IJCAI-05}, and for attaining {\sc
  PTime} data complexity for query answering with DL TBoxes
\cite{DBLP:conf/lpar/KrisnadhiL07}.

Let \Tmc be a TBox. For an ABox \Amc, individual names $a_0,\dots,a_k
\in \mn{Ind}(\Amc)$, and ELIQs $C_{0}(x),\ldots,C_{k}(x)$, we write
$\Tmc,\Amc\models C_{0}(a_{0}) \vee \dots \vee C_{k}(a_{k})$ if for
every model \Imc of \Tmc and \Amc, $\Imc \models C_i(a_i)$ holds for
some $i \leq k$. We say that $\Tmc$ has the \emph{ABox disjunction
  property for ELIQ (resp.\ ELQ)} if for all ABoxes~$\Amc$,
individual names $a_0,\dots,a_k \in \mn{Ind}(\Amc)$, and ELIQs (resp.\
ELQs) $C_{0}(x),\ldots,C_{k}(x)$, $\Tmc,\Amc\models C_{0}(a_{0}) \vee
\dots \vee C_{k}(a_{k})$ implies $\Tmc,\Amc\models C_{i}(a_{i})$ for
some $i\leq k$.
\begin{theorem}
\label{thm:disjproperty}
An $\mathcal{ALCFI}$- ($\mathcal{ALCF}$-)TBox \Tmc is materializable
iff it has the ABox disjunction property for ELIQs (ELQs).
\end{theorem}
\begin{proof}
  For the nontrivial ``if'' direction, let $\Amc$ be an ABox
  that is consistent w.r.t.~$\Tmc$ and such that there is no
  ELIQ-materialization of \Tmc and \Amc.  Then $\Tmc\cup \Amc \cup \Gamma$
  is not satisfiable, where
$$
\Gamma = \{ \neg C(a) \mid \Tmc,\Amc\not\models C(a), a\in {\sf Ind}(\Amc), 
C(x) \mbox{ ELIQ}\}.
$$
In fact, any satisfying interpretation would be an
ELIQ-materialization.  By compactness, there is a finite subset
$\Gamma'$ of $\Gamma$ such that $\Tmc \cup \Amc \cup \Gamma'$ is not
satisfiable, i.e.  $ \Tmc,\Amc \models \bigvee_{\neg C(a) \in \Gamma'}
C(a).  $ Since $\Gamma' \subseteq \Gamma$, we have
$\Tmc,\Amc\not\models C(a)$ for all $\neg C(a)\in \Gamma'$.  Thus,
$\Tmc$ lacks the ABox disjunction property.
\end{proof}
Based on Theorems~\ref{thm:disjmat} and~\ref{thm:disjproperty}, we now
establish that materializability is a necessary condition for
query evaluation to be it {\sc PTime}. 
\begin{theorem}
\label{thm:nomatlower}
If an $\mathcal{ALCFI}$-TBox \Tmc (\ALCF-TBox \Tmc) is not
materializable, then ELIQ-evaluation (ELQ-evaluation) w.r.t.~$\Tmc$ is
{\sc coNP}-hard.
\end{theorem}
\begin{proof}
The proof is by reduction of 2+2-SAT, a variant of propositional
satisfiability that was first introduced by Schaerf as a tool for
establishing lower bounds for the data complexity of query answering
in a DL context~\cite{Schaerf-93}.  A \emph{2+2 clause} is of the form
$(p_1 \vee p_2 \vee \neg n_1 \vee \neg n_2)$, where each of
$p_1,p_2,n_1,n_2$ is a propositional letter or a truth constant $0$,
$1$. A \emph{2+2 formula} is a finite conjunction of 2+2 clauses. Now,
2+2-SAT is the problem of deciding whether a given 2+2 formula is
satisfiable. It is shown in \cite{Schaerf-93} that 2+2-SAT is
\NP-complete.

We first show that if an $\mathcal{ALCFI}$-TBox \Tmc is not
materializable, then UELIQ-evaluation w.r.t.\ \Tmc is {\sc coNP}-hard,
where a UELIQ is a disjunction $C_0(x) \vee \cdots \vee C_k(x)$,
with each $C_i(x)$ an ELIQ. We then sketch the modifications
necessary to lift the result to ELIQ-evaluation w.r.t.\ \Tmc.

  \smallskip

  Since \Tmc is not materializable, by Theorem~\ref{thm:disjproperty}
  it does not have the ABox disjunction property. Thus, there is an
  ABox $\Amc_\vee$, individual names $a_0,\dots,a_k \in
  \mn{Ind}(\Amc)$, and ELIQs $C_0(x),\dots,C_k(x)$, $k \geq 1$, such
  that $\Tmc,\Amc_\vee \models C_0(a_0) \vee \cdots \vee C_k(a_k)$,
  but $\Tmc,\Amc_\vee \not\models C_i(a_i)$ for all $i \leq k$. Assume
  w.l.o.g.\ that this sequence is minimal, i.e., $\Tmc,\Amc_\vee
  \not\models C_0(a_0) \vee \cdots \vee C_{i-1}(a_{i-1}) \vee
  C_{i+1}(a_{i+1}) \vee \cdots \vee C_k(a_k)$ for all $i \leq k$. This
  clearly implies that for all $i \leq k$,
  \begin{itemize}

  \item[($*$)] there is a model $\Imc_i$ of \Tmc
    and $\Amc_\vee$ with $\Imc \models C_i(a_i)$ and $\Imc \not\models
    C_j(a_j)$ for all $j \neq i$.

  \end{itemize}
  We will use $\Amc_\vee$, the individual names $a_1,\dots,a_k$, and the ELIQs
  $C_0(x),\dots,C_k(x)$ to generate truth values for variables in
  the input 2+2 formula.

  Let $\vp=c_0 \wedge \cdots \wedge c_{n}$ be a 2+2 formula in
  propositional letters $z_0,\dots,z_{m}$, and let $c_i=p_{i,1} \vee
  p_{i,2} \vee \neg n_{i,1} \vee \neg n_{i,2}$ for all $i \leq n$. Our
  aim is to define an ABox $\Amc_\vp$ with a distinguished individual
  name $f$ and a UELIQ $q(x)$ such that $\vp$ is unsatisfiable iff
  $\Tmc,\Amc_\vp \models q(f)$. To start, we represent the formula
  $\vp$ in the ABox $\Amc_\vp$ as follows:
  \begin{itemize}

  \item the individual name $f$ represents the formula $\vp$;

  \item the individual names $c_0,\dots,c_n$ represent the clauses of
    $\vp$;

  \item the assertions $c(f,c_0), \dots, c(f,c_{n})$, associate $f$
    with its clauses, where $c$ is a role name that does not occur in
    \Tmc;

  \item the individual names $z_0,\dots,z_m$ represent variables,
    and the individual names $0,1$ represent truth constants;

  \item the assertions
    $$
    \bigcup_{i \leq n} \{ p_1(c_i,p_{i,1}),
    p_2(c_i,p_{i,2}), n_1(c_i,n_{i,1}), n_2(c_i,n_{i,2}) \}
    $$
    associate each clause with the variables/truth constants that occur in it,
    where $p_1,p_2,n_1,n_2$ are role names that do not occur in \Tmc.

  \end{itemize}
  We further extend $\Amc_\vp$ to enforce a truth value for each of
  the variables $z_i$. To this end, add to $\Amc_\vp$ copies
  $\Amc_0,\dots,\Amc_{m}$ of $\Amc_\vee$ obtained by renaming
  individual names such that $\mn{Ind}(\Amc_i) \cap \mn{Ind}(\Amc_j) =
  \emptyset$ whenever $i \neq j$. As a notational convention, let
  $a^i_j$ be the name used for the individual name $a_j \in
  \mn{Ind}(\Amc_\vee)$ in $\Amc_i$ for all $i \leq m$ and $j \leq k$.
  Intuitively, the copy $\Amc_i$ of \Amc is used to generate a truth
  value for the variable $z_i$. 
  To actually connect each individual name $z_i$ to the associated
  ABox $\Amc_i$, we use role names $r_0,\dots,r_k$ that do not occur
  in \Tmc. More specifically, we extend $\Amc_\vp$ as follows:
 \begin{itemize}

 \item link variables $z_i$ to the ABoxes $\Amc_i$ by adding
   assertions $r_j(z_i,a^i_j)$ for all $i \leq m$ and $j \leq k$;
   thus, truth of $z_i$ means that the concept $\exists r_0 . C_0$
   is true at  $z_i$ and falsity means that one of the concepts
   $\exists r_j . C_j, j \leq k$, is true at $z_i$;

 \item to ensure that 0 and 1 have the expected truth values, add a
   copy of $C_0(x)$ viewed as an ABox with root $0'$ and a copy of $C_1(x)$
   viewed as an ABox with root $1'$; add $r_0(0,0')$ and $r_1(1,1')$.

 \end{itemize}
Consider the query $q_0(x)=C(x)$ with
   $$C=\exists c .  (\exists p_1 .  \mn{ff} \sqcap \exists p_2 . \mn{ff} \sqcap
   \exists n_1 .  \mn{tt} \sqcap \exists n_2 . \mn{tt})$$ 
   which describes the existence of a clause with only false literals
   and thus captures falsity of $\vp$, where $\mn{tt}$ is an
   abbreviation for $\exists r_0. C_0$ and $\mn{ff}$ an abbreviation
   for the $\mathcal{ELU}$-concept $\exists r_1. C_1 \sqcup \cdots
   \sqcup \exists r_k . C_k$. It is straightforward to show that $\vp$
   is unsatisfiable iff $\Tmc,\Amc \models q_0(f)$. To obtain the
   desired UELIQ $q(x)$, it remains to take $q_0(x)$ and distribute disjunction
   to the outside.

   \medskip

   We now show how to improve the result from UELIQ-evaluation to
   ELIQ-evaluation. Our aim is to change the encoding of falsity of a
   variable $z_i$ from satisfaction of
    $
    \exists
    r_1 . C_1 \sqcup \cdots \sqcup \exists r_k . C_k
    $ at $z_i$
    to satisfaction of
    $
    \exists h . (\exists r_1 . C_1 \sqcap \cdots \sqcap \exists r_k . C_k),
    $
    at $z_i$,
    where $h$ is an additional role that does not occur in~\Tmc. We
    can then replace the concept $\mn{ff}$ in the query $q_0$ with $
    \exists h . (\exists r_1 . C_1 \sqcap \cdots \sqcap \exists r_k
    . C_k)$, which gives the desired ELIQ $q(x)$. 

    It remains to modify $\Amc_\vp$ to support the new encoding of
    falsity. The basic idea is that each $z_i$ has $k$ successors
    $b_{1}^i,\dots,b_{k}^i$ reachable via $h$ such that for $1 \leq
    j \leq k$,
    \begin{itemize}

    \item $\exists r_\ell. C_\ell$ is satisfied at $b_j^i$ for all $\ell=1,\dots,j-1,j+1,\dots,k$ and

    \item the assertion $r_j(b_j^i,a^i_j)$ is in $\Amc_\vp$.

    \end{itemize}
    Thus, $\exists r_1 . C_1 \sqcap \cdots \sqcap \exists r_k
    . C_k$ is satisfied at $b_j^i$ iff $C_j$ is satisfied at $a^i_j$, for
    all $j$ with $1 \leq j \leq k$. In detail, the modification of
    $\Amc_\vp$ is as follows:
\begin{itemize}

  \item for $1 \leq j \leq k$, add to $\Amc_\vp$ a copy of $C_j$ viewed
  as an ABox, where the root individual name is $d_j$; 

\item for all $i \leq m$, replace the assertions $r_j(z_i,a^i_j)$, $1
  \leq j \leq k$, with the following:
    \begin{itemize}

    \item $h(z_i,b^i_1), \dots, h(z_i,b^i_k)$ for all $i \leq m$;

    \item $r_j(b^i_j,a^i_j), r_1(b^i_j,d_1),\dots,r_{j-1}(b^i_j,d_{j-1})$, \newline
       $r_{j+1}(b^i_j,d_{j+1}),\dots,r_{k}(b^i_j,d_{k})$ for all $i \leq m$
      and $1 \leq j \leq k$.

    \end{itemize}

  \end{itemize}
  This finishes the modified construction. Again, it is not hard to 
  prove correctness.

  \medskip

  It remains to note that, when $\Tmc$ is an \ALCF-TBox, then the above
  construction of $q$ yields an ELQ instead of an ELIQ.
\end{proof}
The converse of Theorem~\ref{thm:nomatlower} fails, i.e., there are
TBoxes that are materializable, but for which ELIQ-evaluation is {\sc
  coNP}-hard. In fact, materializations of such a TBox \Tmc and ABox
\Amc are guaranteed to exist, but cannot always be computed in \PTime
(unless \PTime = {\sc coNP}). Technically, this follows from
Theorem~\ref{thm:red0} later on which states that for every
non-uniform CSP, there is a materializable $\mathcal{ALC}$-TBox for
which Boolean CQ-answering has the same complexity, up to
complementation of the complexity class.

\subsection{Complexity of TBoxes for Different Query Languages}
We make use of our results on materializability to show that {\sc
  PTime} query evaluation w.r.t.~$\mathcal{ALCFI}$-TBoxes does not
depend on whether we consider PEQs, CQs, or ELIQs, and the same is
true for {\sc coNP}-hardness, for Datalog$^{\not=}$-rewritability, and
for monadic Datalog$^{\not=}$-rewritability. For
$\mathcal{ALCF}$-TBoxes, we can add ELQs to the
list. Theorem~\ref{equivalence} below gives a uniform explanation for
the fact that, in the traditional approach to data complexity in OBDA,
the complexity of evaluating PEQs, CQs, and ELIQs has turned out to be
identical for almost all DLs. It allows us to (sometimes) speak of the
`complexity of query evaluation' without reference to a concrete query
language.
\begin{theorem}\label{equivalence}
For all \ALCFI-TBoxes \Tmc, 
\begin{enumerate}
\item 
PEQ-evaluation w.r.t.\ \Tmc is in {\sc PTime} iff CQ-evaluation
  w.r.t.\ \Tmc is in {\sc PTime} iff ELIQ-evaluation w.r.t.\ \Tmc is in {\sc PTime};
\item $\Tmc$ is (monadic) Datalog$^{\not=}$-rewritable for PEQ iff it is
  Datalog$^{\not=}$-rewritable for CQ iff it is
  (monadic) Datalog$^{\not=}$-rewritable for ELIQ (unless {\sc PTime} $= $ {\sc
    coNP});
\item PEQ-evaluation w.r.t.\ \Tmc is {\sc coNP}-hard iff CQ-evaluation
  w.r.t.\ \Tmc is  {\sc coNP}-hard iff ELIQ-evaluation w.r.t.\ \Tmc is {\sc coNP}-hard.

\end{enumerate}
If $\Tmc$ is an $\ALCF$-TBox, then ELIQ can be replaced by ELQ in (1), (2), and (3).
If $\Tmc$ is an $\mathcal{ALCI}$-TBox, then Datalog$^{\not=}$-rewritability can be replaced by
Datalog-rewritability in (2).
\end{theorem} 
\begin{proof}
  We start with Points~(1) and~(2), for which the ``only if''
  directions are trivial. For the converse directions, we may assume
  by Theorem~\ref{thm:nomatlower} that the TBox $\Tmc$ is
  materializable.  The implications from CQ to PEQ in Points~(1)
  and~(2) follow immediately from this assumption: one can first
  transform a given PEQ $q(\vec{x})$ into an equivalent disjunction of
  CQs $\bigvee_{i\in I}q_{i}(\vec{x})$.  CQ-materializability of
  $\Tmc$ implies that, for any ABox $\Amc$ and $\vec{a}$ in
  $\mn{Ind}(\Amc)$, $\Tmc,\Amc\models q(\vec{a})$ iff there exists
  $i\in I$ such that $\Tmc,\Amc\models q_{i}(\vec{a})$.  Thus if
  CQ-evaluation w.r.t.~$\Tmc$ is in {\sc PTime}, 
  evaluation of $(\Tmc,q(\vec{x}))$ is in {\sc PTime}. The same holds
  for (monadic) Datalog$^{\not=}$-rewritability because the class of
  Datalog$^{\not=}$-queries is closed under finite union.

  We now consider the implications from ELIQ to CQ (and from ELQ to CQ
  if $\Tmc$ is a $\mathcal{ALCF}$-TBox) in Points~(1) and~(2).  The
  following claim is the main step of the proof. It states that for
  any CQ $q(\vec{x})$, we can reduce the evaluation of $q(\vec{x})$
  w.r.t.\ \Tmc on an ABox \Amc to evaluating quantifier-free CQs and
  ELIQs $C(x)$ w.r.t.\ \Tmc (ELQs if $\Tmc$ is an
  $\mathcal{ALCF}$-TBox), both on \Amc.

\medskip
\noindent
{\bf Claim 1}. For any materializable TBox $\Tmc$ and CQ $q(\vec{x})$
with $\vec{x}=x_1\cdots x_{n}$,
one can construct a finite set $\mathcal{Q}$ of pairs $(\varphi(\vec{x},\vec{y}),\mathcal{C})$, where
\begin{itemize}
\item $\varphi(\vec{x},\vec{y})$ is a (possibly empty) conjunction of atoms of the form $x=y$ or $r(x,y)$, 
where $r$ is a role name in $q(\vec{x})$ and
\item $\mathcal{C}$ is a finite set of ELIQs
\end{itemize}
such that the following conditions are equivalent for any ABox $\Amc$ and 
$\vec{a}=a_{1}\cdots a_{n}$ from $\mn{Ind}(\Amc)$:
\begin{enumerate}
\item[(i)] $\Tmc,\Amc\models q(\vec{a})$;
\item[(ii)] there exists $(\varphi(\vec{x},\vec{y}),\mathcal{C})\in \mathcal{Q}$ and an assignment $\pi$ in $\mn{Ind}(\Amc)$
with $\pi(x_{i})=a_{i}$ for $1\leq i \leq n$, $\Amc\models_{\pi} \varphi(\vec{x},\vec{y})$,
and $\Tmc,\Amc\models C(\pi(x))$ for all $C(x)\in \mathcal{C}$.
\end{enumerate}
If $\Tmc$ is an $\mathcal{ALCF}$-TBox, then one can choose ELQs instead of ELIQs in each $\mathcal{C}$ 
in $\mathcal{Q}$.

\medskip
\noindent
Before we prove Claim~1, we show how the desired results follow from
it.  Let a CQ $q(\vec{x})$ be given and let $\mathcal{Q}$ be the set
of pairs from Claim~1.
\begin{itemize}
\item Assume that ELIQ-evaluation w.r.t.\ \Tmc is in {\sc PTime}. Then $\Tmc,\Amc\models q(\vec{a})$ can be decided in
  polynomial time since there are only polynomially many assignments
  $\pi$ and for any such $\pi$, $\Amc\models_{\pi}
  \varphi(\vec{x},\vec{y})$ can be checked in polynomial time (using a
  naive algorithm) and $\Tmc,\Amc\models C(\pi(x))$ can be checked in
  polynomial time for each ELIQ $C(x)\in \mathcal{C}$.
\item Assume that $\Tmc$ is (monadic) Datalog$^{\not=}$-rewritable for
  ELIQ. Let $p=(\varphi(\vec{x},\vec{y}),\mathcal{C})\in
  \mathcal{Q}$. For each $C(x)\in \mathcal{C}$, choose a (monadic)
  Datalog$^{\not=}$-rewriting $\Pi_{C}(x)$ of $(\Tmc,C(x))$, assume
  w.l.o.g\ that none of the chosen programs share any IDB relations,
  and that the goal relation of $\Pi_{C}(x)$ is $\mn{goal}_C$. Let
  $\Pi_p$ be the (monadic) Datalog$^{\not=}$ program that consists of
  the rules of all the chosen programs, plus the following rule:
  $$
  \begin{array}{rcl}
    \mn{goal}(\vec{x}) &\leftarrow& \displaystyle \varphi(\vec{x},\vec{y})
    \wedge \bigwedge_{C(x)\in \mathcal{C}}\mn{goal}_{C}(x).
  \end{array}
  $$
  The desired (monadic) Datalog$^{\not=}$-rewriting of
  $(\Tmc,q(\vec{x}))$ is obtained by taking the union of all the
  constructed (monadic) Datalog$^{\not=}$ queries.
\end{itemize}
The implications from ELQs to CQs for $\mathcal{ALCF}$-TBoxes in
Points~(1) and~(2) follow in the same way since, then, each
$\mathcal{C}$ in $\mathcal{Q}$ consists of ELQs only.

For the proof of Claim~1, we first require a technical
observation that allows us to deal with subqueries that are not
connected to an answer variable in the CQ $q(\vec{x})$.  To
illustrate, consider the query $q_{0}=\exists x \, B(x)$. To prove
Claim~1 for $q_{0}$, we have to find a set $\mathcal{Q}$ of pairs
$(\varphi(\vec{y}),\mathcal{C})$ satisfying Conditions (i) and
(ii). Clearly, in this case the components $\varphi(\vec{y})$ will be
empty and so we have to construct a finite set $\mathcal{C}$ of ELIQs
such that for any ABox $\Amc$, $\Tmc,\Amc\models \exists x \, B(x)$ iff
there exists an ELIQ $C(x)\in \mathcal{C}$ and an assignment $\pi$ in
$\mn{Ind}(\Amc)$ such that $\Tmc,\Amc\models C(\pi(x))$. An
\emph{infinite} set $\mathcal{C}$ with this property is given by the
set of all ELIQs $\exists \vec{r}.B(x)$, where $\vec{r}$ is a sequence
$r_{1}\cdots r_{n}$ of roles $r_{i}$ in $\Tmc$ and
$\exists \vec{r}.B$ stands for $\exists r_{1}\cdots \exists
r_{n}.B$---this follows immediately from the assumption that $\Tmc$ is
materializable and that, by Lemma~\ref{lem:count}, for any ABox $\Amc$
that is consistent w.r.t.~$\Tmc$, there exists a generated CQ-initial
model of $\Tmc$ and ABox $\Amc$.  The following result states that it
is sufficient to include in $\mathcal{C}$ the set of all $\exists
\vec{r}.B(x)$ with $\vec{r}$ of length bounded by
$n_0 := 2^{(2(|\Tmc|+|C|)}\cdot 2|\Tmc|+1$.

\medskip
\noindent
{\bf Claim 2.}
Let $C$ be an $\mathcal{ELI}$-concept and assume that
$\Tmc,\Amc\models \exists x \, C(x)$.
If $\Tmc$ is materializable, then there exists a sequence of roles $\vec{r}= r_{1}\cdots r_{n}$ 
with $r_{i}$ in $\Tmc$ and of length $n \leq n_0$ 
and an $a\in {\sf Ind}(\Amc)$ such that $\Tmc,\Amc \models \exists \vec{r}.C(a)$.
If $C$ is an $\mathcal{EL}$-concept and $\Tmc$ an $\mathcal{ALCF}$-TBox, 
then the sequence $\vec{r}$ consists of role names in $\mathcal{T}$.

\medskip
\noindent
{\bf Proof of Claim 2.} 
Let $\Imc$ be a CQ-materialization of $\Tmc$ and $\Amc$. By Points~(3)
and (5) of Lemma~\ref{lem:propertiesofunfold}, we may assume that
$\Imc$ is i-unfolded. From $\Tmc,\Amc\models \exists x \, C(x)$, we
obtain $C^{\Imc}\not=\emptyset$.  Let $n$ be minimal such that there
are $a\in \mn{Ind}(\Amc)$ and $d\in C^{\Imc}$ with $n={\sf
  dist}_{\Imc}(d,a)$ where ${\sf dist}_{\Imc}(d,a)$ denotes the length
of
the shortest path from $d$ to $a$ in the undirected graph $G_\Imc$.
If $n\leq n_0$, we are done. Otherwise fix an $a\in \mn{Ind}(\Amc)$ and denote
by $M$ the set all $e\in C^{\Imc}$ with $n={\sf dist}_{\Imc}(e,a)$. 
Let $d_{0},\ldots,d_{n}$ with $a=d_{0}$, $d_{n}=d$, and $(d_{i},d_{i+1})\in r_{i+1}^{\Imc}$ for $i<n$ be
the shortest path from $a$ to $d$. Observe that $\Tmc,\Amc\models \exists \vec{r}.C(a)$ for 
$\vec{r}=r_{0}\cdots r_{n-1}$ since $\Imc$ is a materialization of $\Tmc$ and $\Amc$. We now
employ a pumping argument to show that this leads to a contradiction.
Let ${\sf cl}(\Tmc,C)$ denote the closure under single negation of the 
set of subconcepts of concepts in $\Tmc$ and $C$. Set 
$$
t^{\Imc}_{\Tmc,C}(e)= \{ D \in {\sf cl}(\Tmc,C) \mid e\in D^{\Imc}\}.
$$
As $n> n_0$, then there exist $d_{i}$ and $d_{i+j}$ with $j>1$ and $i+j<n$ 
such that
$$
t^{\Imc}_{\Tmc,C}(d_{i})=t^{\Imc}_{\Tmc,C}(d_{i+j}), \quad t^{\Imc}_{\Tmc,C}(d_{i+1})=t^{\Imc}_{\Tmc,C}(d_{i+j+1}), \quad
r_{i+1}= r_{i+j+1}
$$
Now replace in $\Imc$ the interpretation induced by the subtree rooted at $d_{i+j+1}$ by the interpretation
induced by the subtree rooted at $d_{i+1}$. We do the same construction for all elements of $M$
and denote the resulting interpretation by $\Jmc$. 
One can easily show by induction that $\Jmc$ is still a model of $\Tmc$ and $\Amc$,
but now $\Jmc\not\models \exists \vec{r}.C(a)$ and so $\Tmc,\Amc\not\models \exists \vec{r}.C(a)$.
This contradiction finishes the proof of Claim~2.
For $\mathcal{EL}$-concepts and $\mathcal{ALCF}$-TBoxes, Claim~2 can
be proved similarly by using an unfolded (rather than i-unfolded)
materialization, which exists by Points~(4) and (6) of
Lemma~\ref{lem:propertiesofunfold}.

\medskip
\noindent
{\bf Proof of Claim~1.}
Assume that $\Tmc$ and $q(\vec{x})=\exists \vec{x} \,
\psi(\vec{x},\vec{y})$ are given.  We have to construct a set \Qmc
such that Conditions~(i) and~(ii) are satisfied. Let \Amc be an ABox
with $\Tmc,\Amc\models q(\vec{a})$, \Imc a materialization of $\Tmc$
and $\Amc$ that is i-unfolded, and $\pi$ an assignment in $\Imc$ such
that $\Imc\models_{\pi} \psi(\vec{x},\vec{y})$. We define a
corresponding pair
$p=(\varphi(\vec{x},\vec{z}),\mathcal{C})$ to be included in \Qmc
(and these are the only pairs in~\Qmc).

For identifying $\varphi(\vec{x},\vec{z})$, set $x \sim y$ if
$\pi(x)=\pi(y)$ and denote by $[x]$ the equivalence class of $x$
w.r.t.\ ``$\sim$''.  Let $\varphi_{0}$ be the set of all atoms
$r([x],[y])$ such that $\pi(x),\pi(y)\in \mn{Ind}(\Amc)$ and there are
$x'\in [x]$ and $y'\in [y]$ with $r(x',y')$ in $\psi$. We obtain
$\varphi(\vec{x},\vec{y})$ by selecting an answer variable $y \in [x]$
for every $[x]$ that contains such a variable and then replacing $[x]$
by $y$ in $\varphi_{0}$, adding $x_{i}=x_{j}$ to
$\varphi(\vec{x},\vec{z})$ for any two (selected) answer variables $x_{i},x_{j}$
with $x_i \sim x_j$, and by regarding the remaining equivalences
classes $[y]$ that do not contain answer variables as variables in
$\vec{z}$. 

We now identify $\mathcal{C}$. Assume w.l.o.g.\ that \Imc uses the
naming scheme of i-unravelings. Let $a\in \mn{Ind}(\Amc)$. By
$\Imc_{a}$, we denote the subinterpretation of $\Imc$ induced by the
set of all elements $a r_{1}d_{1}\cdots r_{n}d_{n}\in \Delta^{\Imc}$.
Let $M$ be a maximal connected component of $\Delta^{\Imc_{a}}\cap
\{\pi(y)\mid y\in \mn{var}(\psi)\}$.  We associate with $M$ an ELIQ to
be included in \Cmc (and these are the only ELIQs in~\Cmc).

The conjunctive query $\varphi_{M}$ consists of all atoms $r([x],[y])$
such that $\pi(x),\pi(y)\in M$ and there are $x'\in [x],y'\in [y]$
with $r(x',y')$ in $\psi$ and all atoms $A([x])$ such that $\pi(x)\in
M$ and there is $x'\in [x]$ with $A(x')$ in $\psi$. We again assume
that equivalence classes $[x]$ that contain an answer variable (there
is at most one such class in $\varphi_{M}$) are replaced with an
answer variable from $[x]$ and regard the remaining equivalences
classes as variables. 
Note that $\varphi_{M}$ is tree-shaped since $\Imc_a$ is. We can thus
pick a variable $x_0$ with $\pi(x_0)\in M$ such that
$\mn{dist}_{\Imc}(a,\pi(x_0))$ is minimal.  Let $x$ be $[x_0]$ if
$[x_0]$ contains no answer variable and, otherise, let $x$ be the
answer variable that $[x_0]$ has been replaced with.  Let
$[y_{1}],\ldots,[y_{m}]$ be the variables in $\varphi_M$ that are
distinct from $x$ and consider the ELIQ $\exists
[y_{1}]\cdots \exists [y_{m}]\varphi_{M}(x,[y_{1}],\ldots,[y_{m}])$,
which we write as $C_{M}(x)$ where $C_M$ is an appropriate
$\mathcal{ELI}$-concept. We now distinguish the following cases:
\begin{itemize}
\item $\pi(x) = a$. In this case, we include $C_{M}(x)$ in $\mathcal{C}$;
\item otherwise, we still know that 
$\Tmc,\Amc\models \exists x \, C(x)$. Thus, by
  Claim~2 there is a sequence of roles $\vec{r}= r_{1}\cdots r_{n}$
  with $r_{i}$ in $\Tmc$ and $n \leq n_0$ such that $\Tmc,\Amc \models
  \exists \vec{r}.C_{M}(a)$ for some $a\in {\sf Ind}(\Amc)$.  In this
  case, we include $\exists \vec{r}.C_{M}(y)$ in $\mathcal{C}$ for some
  fresh variable~$y$.
\end{itemize}
This finishes the construction of \Cmc and thus of \Qmc. 
Clearly, $\mathcal{Q}$ is finite. The stated properties of
$\mathcal{Q}$ follow directly from its construction.
For $\mathcal{ALCF}$-TBoxes and ELQs, Claim~1 can be proved
similarly using an unfolded materialization (instead of an i-unfolded
one) and the observation that in this case all $C_{M}([x])$ and
$\exists \vec{r}.C_{M}(y)$ are ELQs ($\vec{r}$ uses role names only by
Claim~2).

\medskip

We now turn our attention to Point~(3). Here, the ``if'' directions
are trivial and we prove the ``only if'' part. It suffices to show
that  if PEQ-evaluation w.r.t.\ a TBox \Tmc is {\sc coNP}-hard, then
so is ELIQ-evaluation. We start with showing the slightly simpler
result that {\sc coNP}-hardness of CQ-evaluation w.r.t.\ \Tmc 
implies {\sc coNP}-hardness of UELIQ-evaluation, and then sketch
the modifications required to strengthen the proof to attain the
original statement. 

Thus assume that evaluating the CQ $q(\vec{x})$ w.r.t.\ \Tmc is {\sc
  coNP}-hard. We shall exhibit an UELIQ $q'(x)$ such that for every
ABox \Amc and all $\vec{a} \in \mn{Ind}(\Amc)$, one can produce in
polynomial time an ABox $\Amc'$ with a distinguished individual name
$a_0$ such that $\Tmc,\Amc \models q(\vec{a})$ iff $\Tmc,\Amc' \models
q'(a_0)$. Instead of constructing $q'(\vec{x})$ right away, we will
start with describing the translation of \Amc to $\Amc'$. Afterwards,
it will be clear how to construct $q'(\vec{x})$.

Thus, let \Amc be an ABox and $\vec{a}$ from $\mn{Ind}(\Amc)$. The
construction of $\Amc'$ builds on Claim~1 above. Let $\Qmc$ be the set
of pairs from that claim and reserve a fresh individual name $a_0$. To
obtain the desired ABox $\Amc'$, we extend \Amc for every pair
$p=(\varphi(\vec{x},\vec{y}),\Cmc)$ in \Qmc. Let
$\Cmc = C_{p,1}(x_1),\dots,C_{p,k_p}(x_{k_p})$. Then
\begin{itemize}

\item introduce a fresh individual name $a_p$ and fresh role names
  $r_p,r,r_{p,1},\dots,r_{p,k_p}$;

\item add the assertion $r_p(a_0,a_p)$;

\item for every assignment $\pi$ in
  $\mn{Ind}(\Amc)$ with
  $\Amc \models_\pi \varphi(\vec{x},\vec{y})$ and
  $\pi(\vec{x})=\vec{a}$, introduce

\begin{itemize}

\item a fresh individual name $a_{p,\pi}$ and the assertion
  $r(a_p,a_{p,\pi})$;

\item the assertion $r_{p,i}(a_p,\pi(x_i))$ for $1 \leq i \leq k_p$.

\end{itemize}
\end{itemize}
From Claim~1, it is immediate that $\Tmc,\Amc \models q(\vec{a})$ iff
$\Tmc,\Amc' \models q'(x)$ where $q'(x)$ is the UELIQ $\midsqcup_{p
  \in Q} q'(x)$ 
with
$
C_p = \exists r_p . \exists r . \midsqcap_{1 \leq i \leq k_p} \exists
r_{p,i} . C_{p,i}.
$
Thus, evaluating $q'(x)$ w.r.t.\ \Tmc is {\sc coNP}-hard, as required.
It remains to modify the reduction by replacing CQs with PEQs and
UELIQs with ELIQs. The former is straightforward: every PEQ is
equivalent to a finite disjunction of CQs, and thus we can construct
$\Amc'$ and $q'(x)$ in essentially the same way as before; instead of
considering all pairs from \Qmc for a single CQ, we now use the union
of all sets \Qmc for the finitely many CQs in question. Finally, we
can replace UELIQs with ELIQs by using the same construction as in the
proof of Theorem~\ref{thm:nomatlower}: after adding some
straightforward auxiliary structure to $\Amc'$, one can replace a
disjunction of ELIQs by (essentially) their conjunction, which is
again an ELIQ. Details are left to the reader.
\end{proof}
We remark that Theorem~\ref{equivalence} can be extended to also cover
rewritability into first-order (FO) queries, and that the proof is
almost identical to that for Datalog$^{\not=}$-rewritability.

\section{Unraveling Tolerance}
\label{sect:unrav}
We develop a condition on TBoxes, called unraveling tolerance, that is
sufficient for the TBox to be monadic Datalog$^{\neq}$-rewritable for
PEQ, and thus also sufficient for PEQ-evaluation w.r.t.\ the TBox
being in {\sc PTime}. Unraveling tolerance strictly generalizes
syntactic `Horn conditions' such as the ones used to define the DL
Horn-$\shiq$, which was designed as a (syntactically) maximal DL with
{\sc PTime} query evaluation
\cite{journals/jar/HustadtMS07,conf/jelia/EiterGOS08}.

Unraveling tolerance is based on an unraveling operation on ABoxes, in
the same spirit as the unfolding of an interpretation into a tree
interpretation we discussed above.
Formally, the \emph{unraveling} $\Amc^u$ of an ABox \Amc is the
following (possibly infinite) ABox:
\begin{itemize}

\item $\mn{Ind}(\Amc^u)$ is the set of
  sequences \mbox{$b_0 r_0 b_1 \cdots r_{n-1} b_n$, $n \geq 0$}, with $b_0,\dots,b_n \in
  \mn{Ind}(\Amc)$ and $r_0,\dots,r_{n-1} \in \NR \cup \Nsf_\Rsf^-$
  such that for all $i <n$, we have $r_i(b_i,b_{i+1}) \in \Amc$ and
  $(b_{i-1},r_{i-1}^-) \neq (b_{i+1},r_i)$ when $i>0$;


\item for each $C(b) \in \Amc$ and $\alpha=b_0 r_0 \cdots 
  r_{n-1}b_n \in \mn{Ind}(\Amc^u)$ with $b_n=b$, $C(\alpha) \in \Amc^u$;


\item for each $\alpha=b_0 r_{0} \cdots r_{n-1} b_n \in \mn{Ind}(\Amc^u)$ with $n >
  0$, $r_{n-1}(b_0r_0 \cdots r_{n-2}b_{n-1},\alpha) \in
  \Amc^u$.

\end{itemize}
For all $\alpha = b_0 \cdots b_n \in \mn{Ind}(\Amc^u)$, we write
$\mn{tail}(\alpha)$ to denote~$b_n$. Note that the condition
$(b_{i-1},r_{i-1}^-) \neq (b_{i+1},r_i)$ is needed to ensure that
functional roles can still be interpreted in a
functional way after unraveling.
%
\begin{definition}[Unraveling Tolerance]
  A TBox \Tmc is \emph{unraveling tolerant} if for all ABoxes \Amc and
  ELIQs $q$, we have that $\Tmc,\Amc \models q$ implies $(\Tmc,\Amc^u)
  \models q$.
\end{definition}
It is not hard to prove that the converse direction `$\Tmc,\Amc^u
\models q$ implies $\Tmc,\Amc \models q$' is true for \emph{all}
$\mathcal{ALCFI}$-TBoxes. Note that it is pointless to define
unraveling tolerance for queries that are not necessarily tree shaped,
such as CQs.
%
%
\begin{example}
\label{ex:unrav}
~

\smallskip 
\noindent 
(1)~The \ALC-TBox $\Tmc_1 = \{ A \sqsubseteq \forall r . B \}$ is
unraveling tolerant. This can be proved by showing that (i)~for any
(finite or infinite) ABox \Amc, the interpretation $\Imc^+_\Amc$ that
is obtained from $\Amc$ by extending $B^{\Imc^+_\Amc}$ with all $a \in
\mn{Ind}(\Amc)$ that satisfy $\exists r^- . A$ in \Amc (when viewed as an
interpretation) is an ELIQ-materialization of $\Tmc_1$ and \Amc; and
(ii)~$\Imc^+_\Amc \models C(a)$ iff $\Imc^+_{\Amc^u} \models C(a)$ for
all ELIQs $C(x)$ and $a \in \mn{Ind}(\Amc)$.  The proof of (ii) is
based on a straightforward induction on the structure of the
$\mathcal{ELI}$-concept~$C$. As illustrated by the ABox $\Amc =
\{r(a,b), A(a) \}$ and the fact that $\Amc^u,\Tmc \models B(b)$, the
use of inverse roles in the definition of $\Amc^u$ is crucial here
despite the fact that $\Tmc_1$ does not use inverse roles.

\smallskip
\noindent
(2)~A simple example for an \ALC-TBox that is not unraveling tolerant
is 
$$\Tmc_2= \{ A \sqcap \exists r . A \sqsubseteq B, \neg A \sqcap
\exists r . \neg A \sqsubseteq B \}. 
$$
For $\Amc = \{ r(a,a) \}$, it is easy to see that we have
$\Tmc_2,\Amc \models B(a)$ (use a case distinction on the truth
value of $A$ at $a$), but $\Tmc_2,\Amc^u \not\models B(a)$.  
%
\end{example}
Before we show that unraveling tolerance indeed implies \PTime query
evaluation, we first demonstrate the generality of this property by
relating it to Horn-\ALCFI, the \ALCFI-fragment of Horn-\shiq.
Different versions of Horn-$\mathcal{SHIQ}$ have been proposed in the
literature, giving rise to different versions of
Horn-$\mathcal{ALCFI}$
\cite{journals/jar/HustadtMS07,2007cbfhdl,conf/jelia/EiterGOS08,conf/ijcai/Kazakov09}.
As the original definition from \cite{journals/jar/HustadtMS07} based
on polarity is rather technical, we prefer to work with the following
equivalent and less cumbersome definition.  A
\emph{Horn-$\mathcal{ALCFI}$-TBox} \Tmc is a finite set of concept
inclusions $L \sqsubseteq R$ and functionality assertions where $L$
and $R$ are built according to the following syntax rules:
$$
\begin{array}{r@{\,}c@{\,}l}
    R,R' &::=& \top \mid \bot \mid A \mid \neg A \mid R \sqcap R' \mid 
L \rightarrow R \mid \exists r . R \mid
   \forall r . R \\[1mm]
    L,L' &::=& \top \mid \bot \mid A \mid L \sqcap L' \mid L \sqcup L' \mid \exists r . L 
\end{array}
$$
with $r$ ranging over $\NR \cup \Nsf^-_\Rsf$ and $L\rightarrow R$
abbreviating $\neg L \sqcup R$. Whenever convenient, we may assume
w.l.o.g.\ that \Tmc contains only a single concept inclusion $\top
\sqsubseteq C_\Tmc$ where $C_\Tmc$ is built according to the topmost
rule above.

By applying some simple transformations, it is not hard to show that
every Horn-$\mathcal{ALCFI}$-TBox according to the original
polarity-based definition is equivalent to a
Horn-$\mathcal{ALCFI}$-TBox of the form introduced here. Although not
important in our context, we note that even a polynomial time
transformation is possible.
\begin{theorem}
\label{lem:hornalc}
Every Horn-$\mathcal{ALCFI}$-TBox is unraveling tolerant.
\end{theorem}
\begin{proof} 
  As a preliminary, we give a characterization of the entailment of
  ELIQs in the presence of Horn-\ALCFI-TBoxes which is in the spirit
  of the chase procedure as used in database theory \cite{DBLP:journals/tcs/FaginKMP05,DBLP:journals/jair/CaliGK13}
  and of consequence-driven algorithms as used for reasoning in Horn
  description logics such as $\EL^{{+}{+}}$ and Horn-\shiq
  \cite{BaBrLu-IJCAI-05,conf/ijcai/Kazakov09,conf/jelia/Krotzsch10}.

  We use \emph{extended ABoxes}, i.e., finite sets of assertions
  $C(a)$ and $r(a,b)$ with $C$ a \emph{potentially compound} concept.
  An \emph{$\mathcal{ELIU}_\bot$-concept} is a concept that is formed according
  to the second syntax rule in the definition of Horn-\ALCFI.
  For an extended ABox $\Amc'$ and
  an assertion $C(a)$, $C$ an $\mathcal{ELIU}_\bot$-concept, we write
  $\Amc' \vdash C(a)$ if $C(a)$ \emph{has a syntactic match} in $\Amc'$,
  formally:
  \begin{itemize}

  \item $\Amc' \vdash \top(a)$ is unconditionally true;

  \item $\Amc' \vdash \bot(a)$ if $\bot(b) \in \Amc'$ for some $b \in \mn{Ind}(\Amc)$;

  \item $\Amc' \vdash A(a)$ if $A(a) \in \Amc'$;

  \item $\Amc' \vdash C\sqcap D(a)$ if $\Amc' \vdash C(a)$ and $\Amc' \vdash D(a)$;

  \item $\Amc' \vdash C\sqcup D(a)$ if $\Amc' \vdash C(a)$ or $\Amc' \vdash D(a)$;

  \item $\Amc' \vdash \exists r . C(a)$ if there is an $r(a,b) \in
    \Amc'$ such that $\Amc' \vdash C(b)$.

  \end{itemize}
  Now for the announced characterization.  Let $\Tmc=\{\top
  \sqsubseteq C_\Tmc \}$ be a Horn-$\mathcal{ALCFI}$-TBox and $\Amc$ a
  potentially infinite ABox (so that the characterization also applies
  to unravelings of ABoxes).  We produce a sequence of extended ABoxes
  $\Amc_0,\Amc_1,\dots$, starting with $\Amc_0 = \Amc$.
  In what follows, we use additional individual
  names of the form $a r_1 C_1 \cdots r_k C_k$ with $a \in
  \mn{Ind}(\Amc_0)$, $r_1,\dots,r_k$ roles that occur in \Tmc, and
  $C_1,\dots,C_k$ subconcepts of concepts in $\Tmc$.  We assume that \NI contains
  such names as needed and use the symbols $a,b,\dots$ also to refer to
  individual names of this compound form. Each extended ABox
  $\Amc_{i+1}$ is obtained from $\Amc_i$ by applying the following
  rules in a fair way:
  \begin{itemize}[leftmargin=11mm]

 \item[{\sf R1}] if $a \in \mn{Ind}(\Amc_i)$, then add $C_\Tmc(a)$.
 
 \item[{\sf R2}] if $C \sqcap D(a) \in \Amc_i$, then add 
   $C(a)$ and $D(a)$;

 \item[{\sf R3}] if $C \rightarrow D(a) \in \Amc_i$ and $\Amc_i \vdash C(a)$, 
  then add $D(a)$;

\item[{\sf R4}] if $\exists r . C(a) \in \Amc_i$ and $\mn{func}(r)
  \notin \Tmc$, then add $r(a,a r C)$ and $C(a r C)$;

\item[{\sf R5}] if $\exists r . C(a) \in \Amc_i$, $\mn{func}(r) \in
  \Tmc$, and $r(a,b) \in \Amc_i$, then add $C(b)$;

\item[{\sf R6}] if $\exists r . C(a) \in \Amc_i$, $\mn{func}(r) \in
  \Tmc$, and there is no $r(a,b) \in \Amc_i$, then add $r(a,a r C)$ and $C(a r C)$;

 \item[{\sf R7}] if $\forall r . C(a) \in \Amc_i$ and $r(a,b) \in \Amc_i$,
   then add $C(b)$.

\end{itemize}
Let $\Amc_c = \bigcup_{i \geq 0} \Amc_i$ be the \emph{completion} of
the original ABox \Amc.\footnote{Order of rule application has an
  impact on the shape of $\Amc_c$, but is irrelevant for the remainder
  of the proof.}
Note that $\Amc_c$ may be infinite even if $\Amc$
is finite, and that none of the above rules is applicable in $\Amc_c$.
%
We write `$\Amc_c \vdash \bot$' instead of `$\Amc_c \vdash \bot(a)$ for some $a \in \NC$'.
If $\Amc\not\vdash\bot$, then $\Amc_{c}$ corresponds to an interpretation $\Imc_{c}$ in the standard way, i.e.,
$$
  \begin{array}{rcll}
  \Delta^{\Imc_c}&=&\mn{Ind}(\Amc_c) \\[1mm]
    A^{\Imc_c} &=& \{ a \mid A(a) \in \Amc_c \} & \text{for all } A \in \NC\\[1mm]
   r^{\Imc_c} &=& \{ r(a,b) \mid r(a,b) \in  \Amc_c \} & \text{for all } r
   \in \NR 
  \end{array}
  $$
where in $\Imc_{c}$ we assume that only the individual names in $\mn{Ind}(\Amc)$ are elements of $\NI$.
\\[2mm]
{\bf Claim~1}. If $\Amc_{c}\not\vdash\bot$, then $\Imc_{c}$ is a PEQ-materialization of $\Tmc$ and $\Amc$.
\\[2mm]
To prove Claim~1 it suffices to show that there is a homomorphism $h$
preserving $\mn{Ind}(\Amc)$ from $\Imc_c$
into every model $\Jmc$ of $\Tmc$ and $\Amc$ and that $\Imc_{c}$ is a
model of $\Tmc$ and $\Amc$. The latter is immediate by construction of 
$\Amc_{c}$. Regarding the former, the desired homomorphism $h$ can can
be constructed inductively starting with the ABox $\Amc_0$ and then
extending to $\Amc_1,\Amc_2,\dots$.
Using Claim~1 and the easily proved fact that $\Amc_{c}\not\vdash\bot$ iff $\Amc$ is consistent w.r.t.~$\Tmc$ one 
can now show the following.
  %
 \\[2mm]
 {\bf Claim~2}. $\Tmc,\Amc \models C(a)$ iff $\Amc_c \vdash C(a)$ or
 $\Amc_c \vdash \bot$, for all ELIQs $C(x)$ and $a \in
 \mn{Ind}(\Amc)$. 
 %
%
%
%
\\[2mm]
 %
%
%
%
We now turn to the actual proof of Theorem~\ref{lem:hornalc}.
  Consider the application of the above completion construction to
  both the original ABox \Amc and its unraveling $\Amc^u$. Recall that
  individual names in $\Amc^u$ are of the form $a_0 r_0 a_1 \cdots r_{n-1}
  a_n$. Consequently, individual names in $\Amc_c^u$ take the form $a_0 r_0 a_1
  \cdots r_{n-1} a_n s_1 C_1 \cdots s_k C_k$. For $a \in
  \mn{Ind}(\Amc_c)$ and $\alpha \in \mn{Ind}(\Amc_c^u)$, we write
  $a \sim \alpha$ if $a$ and $\alpha$ are of the form
  %
$a_ns_1 C_1 \cdots s_k C_k$
and
$a_0 r_0 a_1 \cdots r_{n-1} a_n s_1 C_1 \cdots s_k C_k$,
  %
respectively, with $k \geq 0$. Note that, in particular, $a \sim a$
for all $a \in \mn{Ind}(\Amc)$. 
  The following
  claim can be shown by induction on $i$.
  \\[2mm]
  {\bf Claim~3}. For all $a \in \mn{Ind}(\Amc_i)$ and $\alpha \in
  \mn{Ind}(\Amc^u_i)$ with $a \sim \alpha$, we have
 \begin{enumerate}

 \item $\Amc_i \vdash C(a)$ iff $\Amc^u_i \vdash C(\alpha)$ for 
   all $\mathcal{ELI}$-concepts $C$;

 \item $C(a) \in \Amc_i$ iff $C(\alpha) \in \Amc^u_i$ for all 
   subconcepts $C$ of concepts in $\Tmc$.

 \end{enumerate}
 Now, unraveling tolerance of \Tmc follows from Claims~2 and~3.
\end{proof}
%
%
%
Theorem~\ref{lem:hornalc} shows that unraveling tolerance and Horn
logic are closely related. Yet, the next example demonstrates that
there are unraveling tolerant \ALCFI-TBoxes that are not equivalent to
any Horn sentence of FO. Since any Horn-\ALCFI-TBox is equivalent to
such a sentence, it follows that unraveling tolerant
$\mathcal{ALCFI}$-TBoxes strictly generalize
Horn-$\mathcal{ALCFI}$-TBoxes. This increased generality will pay off
in Section~\ref{sect:depth1} when we establish a dichotomy result for
TBoxes of depth one.
\begin{example}
\label{ref:matexa} 
Take the \ALC-TBox
$$
  \Tmc = \{ \exists r.(A \sqcap \neg B_1 \sqcap \neg B_2) \sqsubseteq
\exists r.(\neg A \sqcap \neg B_1 \sqcap \neg B_2)\}.
$$
One can show as in Example~\ref{ex:unrav} (1) that \Tmc is unraveling
tolerant; here, the materialization is actually $\Amc$ itself rather
than some extension thereof, i.e., as far as ELIQ (and even PEQ)
evaluation is concerned, \Tmc cannot be distinguished from the empty
TBox.

It is well-known that FO Horn sentences are preserved under direct
products of interpretations \cite{ChangKeisler}.  To show that \Tmc is
not equivalent to any such sentence, it thus suffices to show that
$\Tmc$ is not preserved under direct products. This is simple: let
$\Imc_1$ and $\Imc_2$ consist of a single $r$-edge between elements
$d$ and $e$, and let $e \in (A \sqcap B_1 \sqcap \neg B_2)^{\Imc_1}$
and $e \in (A \sqcap \neg B_1 \sqcap B_2)^{\Imc_2}$; then the direct
product \Imc of $\Imc_1$ and $\Imc_2$ still has the $r$-edge between
$(d,d)$ and $(e,e)$ and satisfies $(e,e) \in (A \sqcap \neg B_1 \sqcap
\neg B_2)^{\Imc}$, thus is not a model of \Tmc. 
\end{example}
We next show that unraveling tolerance is indeed a sufficient
condition for monadic Datalog$^{\neq}$-rewritability (and thus for
\PTime query evaluation). In Section~\ref{sect:dicho}, we will
establish a connection between query evaluation under DL TBoxes and
constraint satisfaction problems (CSPs). The monadic Datalog$^{\not=}$
programs that we construct resemble the canonical monadic Datalog
programs for CSPs~\cite{DBLP:journals/siamcomp/FederV98}. 

%
%

Let $\Tmc$ be an unraveling tolerant \ALCFI-TBox and $q=C_0(x)$ an
ELIQ. We show how to construct a Datalog$^{\not=}$-rewriting of the
OMQ $(\Tmc,q(x))$. Using the construction from the proof of
Theorem~\ref{equivalence}, one can extend this construction from ELIQs
to PEQs. Recall from the proof of Theorem~\ref{equivalence} that
$\mn{cl}(\Tmc,C_{0})$ denotes the closure under single negation of the
set of subconcepts of \Tmc and $C_{0}$.
For an interpretation \Imc and $d \in \Delta^\Imc$, we use
$t^\Imc_{\Tmc,q}(d)$ to denote the set of concepts $C \in
\mn{cl}(\Tmc,C_{0})$ such that $d\in C^\Imc$. A \emph{$\Tmc,q$-type}
is a subset $t \subseteq \mn{cl}(\Tmc,C_{0})$ such that for some model
\Imc of \Tmc, we have $t=t^\Imc_{\Tmc,q}(d)$. We use $\mn{tp}(\Tmc,q)$
to denote the set of all \Tmc,$q$-types. Observe that one can
construct the set $\mn{tp}(\Tmc,q)$ in exponential time as the set of
all $t\subseteq \mn{cl}(\Tmc,C_{0})$ such that for any concept $\neg
C\in\mn{cl}(\Tmc,C_{0})$ either $C\in t$ or $\neg C \in t$ and the
concept $\bigsqcap_{C\in t}C$ is satisfiable in a model of~$\Tmc$.

For $t,t' \in \mn{tp}(\Tmc,q)$ and $r$ a role, we write $t
\rightsquigarrow_r t'$ if there are a model \Imc of \Tmc and $d,d' \in
\Delta^\Imc$ such that $t^\Imc_{\Tmc,q}(d)=t$,
$t^\Imc_{\Tmc,q}(d')=t'$, and $(d,d') \in r^\Imc$. One can construct
the set of all $(t,t',r)$ such that $t \rightsquigarrow_r t'$ in
exponential time by checking for each candidate tuple $(t,t',r)$
whether the concept
$$
\big(\bigsqcap_{C\in t}C\big)\sqcap \exists r.\big(\bigsqcap_{C\in t'}C\big)
$$
is satisfiable in a model of $\Tmc$. 

%
%
%
%
%
%
  Introduce, for every set $T \subseteq \mn{tp}(\Tmc,C_{0})$ a unary IDB
  relation $P_{T}$. Let $\Pi$ be the monadic Datalog$^{\neq}$ program that contains the following rules:
\begin{enumerate}

\item $P_{T}(x) \leftarrow A(x)$ for all concept names
  $A \in \mn{cl}(\Tmc,C_{0})$ and $T=\{ t\in \mn{tp}(\Tmc,q)\mid A\in t\}$;

\item $P_{T}(x) \leftarrow P_{T_{0}}(x) \wedge r(x,y) \wedge
  P_{T_{1}}(y)$ for all $T_0,T_1 \subseteq \mn{tp}(\Tmc,q)$ and all
  role names $r$ that occur in $\mn{cl}(\Tmc,C_{0})$ and their
  inverses, where $T = \{ t\in T_{0} \mid \exists t' \in T_1: t
  \rightsquigarrow_r t' \}$;

\item $P_{T_0 \cap T_1}(x) \leftarrow P_{T_{0}}(x) \wedge
  P_{T_{1}}(x)$ for all $T_0,T_1 \subseteq \mn{tp}(\Tmc,q)$;

\item $\mn{goal}(x) \leftarrow P_T(x)$ for all $T \subseteq
  \mn{tp}(\Tmc,q)$ such that $t \in T$ implies $C_0 \in T$;

\item $\mn{goal}(x) \leftarrow P_\emptyset(y)$;



\item $\mn{goal}(x) \leftarrow r(y,z_1) \wedge r(y,z_2) \wedge z_1\not= z_2$
  for all ${\sf func}(r)\in \Tmc$.

\end{enumerate}
To show that $\Pi$ is a rewriting of the OMQ $(\Tmc,C_0(x))$, it suffices
to establish the following lemma.
\begin{lemma}\label{datalogrewr}
$\Amc\models \Pi(a_0)$ iff $\Tmc,\Amc\models C_{0}(a_0)$, for
all ABoxes $\Amc$ and $a_0\in \mn{Ind}(\Amc)$.
\end{lemma}
\begin{proof}
The ``if'' direction is straightforward: by induction on the number of
rule applications, one can show that whenever $\Pi$ derives $P_T(a)$,
then every model of $\Tmc$ and $\Amc$ satisfies
$t^\Imc_{\Tmc,q}(a) \in T$.  By definition of the goal rules of $\Pi$,
$\Amc \models \Pi(a_0)$ thus implies that every model of \Tmc and \Amc
makes $C_0(a_0)$ true or that \Amc is inconsistent
w.r.t. $\Tmc$. Consequently, $\Tmc,\Amc\models C_{0}(a_0)$.

For the ``only if'' direction, it suffices to show that
$\Amc \not\models \Pi(a_0)$ implies $\Tmc,\Amc^u \not\models C_0(a_0)$
since \Tmc is unraveling tolerant. Because of the rules in $\Pi$ of
the form (3), for every $a \in \mn{Ind}(\Amc)$ we can find a unique
minimal $T_a$ such that $P_{T_a}(a)$ is derived by $\Pi$.  Observe
that, $A(\alpha) \in \Amc^u$, $\mn{tail}(\alpha)=a$, and $t \in T_a$
implies $A \in t$ because of the rules of the form (1) in $\Pi$ and by
construction of $\Amc^u$.

We first associate with every $\alpha \in \mn{Ind}(\Amc^u)$ a concrete
$\Tmc,q$-type $t_\alpha \in T_{\mn{tail}(\alpha)}$.  To start, we
choose $t_a \in T_a$ arbitrarily for all $a \in \mn{Ind}(\Amc)$. Now
assume that $t_\alpha$ has already been chosen and that
$\beta=\alpha r b \in \mn{Ind}(\Amc^u)$. Then
$r(\mn{tail}(\alpha),b) \in \Amc$. Because of the rules in $\Pi$ of the
form (2) and (5), we can thus choose $t_\beta \in T_b$ such that
$t_\alpha \rightsquigarrow_{r} t_\beta$. In this way, all types
$t_\alpha$ will eventually be chosen.
We now construct an interpretation \Imc, starting with
$$
\begin{array}{rcl}
  \Delta^\Imc &=& \mn{Ind}(\Amc^u) \\[1mm]
  A^\Imc &=& \{  \alpha \mid A \in t_\alpha \} \text{ for all concept names } A\\[1mm]
  r^\Imc &=& \{ (\alpha,\beta) \mid r(\alpha,\beta) \in \Amc^u \} \text{ for all role names } r.
\end{array}
$$
Next, consider every $\alpha \in \mn{Ind}(\Amc^u)$ and every
$\exists r. C \in t_\alpha$ such that $\Amc^u$ does not contain an
assertion $r(\alpha,\beta)$ with $C\in t_{\beta}$. 
First assume that $\mn{func}(r)\not\in \Tmc$.
There must be a $\Tmc,q$-type $t$ such
that $t_\alpha \rightsquigarrow_r t$ and $C \in t$. Choose a model
$\Jmc_{\alpha,\exists r .C}$ of \Tmc and
$D=\midsqcap t_a \sqcap \exists r . \midsqcap t$, a
$d \in D^{\Jmc_{\alpha,\exists r .C}}$, and an
$e \in (\midsqcap t)^{\Jmc_{\alpha,\exists r .C}}$ with
$(d,e) \in r^{\Jmc_{\alpha,\exists r .C}}$. W.l.o.g., we can assume
that $\Jmc_{\alpha,\exists r .C}$ is tree-shaped with root $d$. Let
$\Jmc^-_{\alpha,\exists r .C}$ be obtained from
$\Jmc_{\alpha,\exists r .C}$ by dropping the subtree rooted at $e$.
Now disjointly add $\Jmc^-_{\alpha,\exists r .C}$ to \Imc,
additionally including $(a,d)$ in~$r^\Imc$. 
Now assume that $\mn{func}(r)\in \Tmc$.
Then, if there exists $r(\alpha,\beta)\in \Amc^{u}$, then $C\in t_{\beta}$
as otherwise we do not have $t_{\alpha}\rightsquigarrow_r t_{\beta}$. Thus, assume
there is no $r(\alpha,\beta)\in \Amc^{u}$. There must be a $\Tmc,q$-type $t$ such
that $t_\alpha \rightsquigarrow_r t$ and $C \in t$. We then have $D\in t$ for
all $\exists r.D \in t_{\alpha}$ and so construct only a single $\Jmc^-_{\alpha,\exists r .C}$ for
the role $r$ and disjointly add $\Jmc^-_{\alpha,\exists r .C}$ to \Imc,
additionally including $(a,d)$ in~$r^\Imc$. 
This finishes the
construction of \Imc. The following claim can be proved by induction
on $C$, details are omitted.
\\[2mm]
{\bf Claim.}  For all $C \in \mn{cl}(\Tmc,C_0)$ :
\begin{enumerate}[label=(\alph*)]

\item $\alpha \in C^\Imc$ iff $C \in t_\alpha$ for all $\alpha \in \mn{Ind}(\Amc^u)$ and

\item $d \in C^{\Jmc_{\alpha,\exists r .D}}$ iff $d \in C^{\Imc}$ for
  all $\Jmc_{\alpha,\exists r .D}$ and all $d \in \Delta^{\Jmc^-_{\alpha,\exists r .D}}$.

\end{enumerate}
By construction of \Imc and since $A(\alpha) \in \Amc^u$ implies
$A \in t_\alpha$, \Imc is a model of \Amc. Due to the rules in $\Pi$
that are of the form (4), Point~(a) of the claim yields
$\Imc \not\models C_0(a_0)$. Finally, we observe that \Imc is a model
of \Tmc. The concept inclusions in \Tmc are satisfied by the above
claim, since $C \sqsubseteq D \in \Tmc$ means that $C \in t$ implies
$D \in t$ for every $\Tmc,q$-type $t$, and since each
$\Jmc_{\alpha,\exists r .C}$ is a model of \Tmc. Due to the rules in
$\Pi$ that are of the form (6) and since each
$\Jmc_{\alpha,\exists r .C}$ is a model of \Tmc, all functionality
assertions in \Tmc are satisfied as well.  Summing up, we have shown
that $\Tmc,\Amc^u \not\models C_0(a_0)$, as required.
\end{proof}
Together with Theorem~\ref{equivalence}, we have established the following result.
\begin{theorem}
\label{thm:unravupperlem:hornalclem:hornalc}
Every unraveling tolerant $\mathcal{ALCFI}$-TBox is
monadic Datalog$^{\neq}$-rewritable for PEQ.
\end{theorem}
Together with Theorems~\ref{equivalence} and~\ref{lem:hornalc},
Theorem~\ref{thm:unravupperlem:hornalclem:hornalc} also reproves the
known \PTime upper bound for the data complexity of CQ-evaluation in
Horn-$\mathcal{ALCFI}$ \cite{conf/jelia/EiterGOS08}. Note that it is
not clear how to attain a proof of
Theorem~\ref{thm:unravupperlem:hornalclem:hornalc} via the CSP
connection established in Section~\ref{sect:dicho} since functional
roles break this connection.

\smallskip 

By Theorems~\ref{thm:nomatlower} and~\ref{thm:unravupperlem:hornalclem:hornalc}, unraveling
tolerance implies materializability unless $\PTime = \text{\sc{NP}}$.
Based on the disjunction property, this implication can also be proved
without the side condition.
\begin{theorem}
  \label{lem:utimpliesmat}
  Every unraveling tolerant $\mathcal{ALCFI}$-TBox is materializable.
\end{theorem}
\begin{proof}
  We show the contrapositive using a proof strategy that is very
  similar to the second step in the proof of
  Theorem~\ref{thm:nomatlower}. Thus, take an $\mathcal{ALCFI}$-TBox
  \Tmc that is not materializable. By Theorem~\ref{thm:disjmat}, \Tmc
  does not have the disjunction property. Thus, there are an ABox
  $\Amc_\vee$, ELIQs $C_0(x_0),\dots,C_k(x_k)$, and $a_1,\dots,a_k \in
  \mn{Ind}(\Amc_\vee)$ such that $\Tmc,\Amc_\vee \models C_0(a_0) \vee
  \cdots \vee C_k(a_k)$, but $\Tmc,\Amc_\vee \not\models C_i(a_i)$ for
  all $i \leq k$. Let $\Amc_i$ be $C_i$ viewed as a tree-shaped ABox
  with root $b_i$, for all $i \leq k$. Assume w.l.o.g.\ that none of
  the ABoxes $\Amc_\vee, \Amc_0,\dots,\Amc_k$ share any individual
  names and reserve fresh individual names $c_0,\dots,c_k$ and fresh
  role names $r,r_0,\dots,r_k$. Let the ABox \Amc be the union of 
  $$
  \Amc_\vee \cup \Amc_0 \cup \cdots \cup \Amc_k \cup \{
  r(c,c_0),\dots,r(c,c_k) \} 
  $$
  and 
  $$
    \{r_0(c_j,b_0),\dots,r_{j-1}(c_j,b_{j-1}),r_j(c_j,a_j),
    r_{j+1}(c_j,b_{j+1}),\dots,r_k(c_j,b_k) \}
  $$
  for $1 \leq j \leq k$. Consider the ELIQ
  $$
  q=\exists r . (\exists r_0 . C_0 \sqcap \cdots \sqcap \exists r_k . C_k)(x).
  $$
  By the following claim, $\Amc$ and $q$ witness that \Tmc is not
  unraveling tolerant.
  \\[2mm]
  {\bf Claim.}  $\Tmc,\Amc \models q(c)$, but $\Tmc,\Amc^u \not\models
  q(c)$.
  \\[2mm]
  \emph{Proof.}  ``$\Tmc,\Amc \models q(c)$''. Take a model \Imc of
  \Tmc and \Amc. By construction of \Amc, we have
  $a_i^\Imc \in (\exists r_j . C_j)^\Imc$ whenever $i \neq j$. Due to
  the edges $r_0(c_0,a_0),\dots,r_k(c_k,a_k)$ and since
  $\Tmc,\Amc_\vee \models C_0(a_0) \vee \cdots \vee C_k(a_k)$, we thus
  find
  at least one $c_i$ such that
  $c_i^\Imc \in (\exists r_i . C_i)^\Imc$. Consequently,
  $\Imc \models q(c)$.

  \smallskip

  ``$\Tmc,\Amc^u \not\models q(c)$'' (sketch). Consider the elements
  $crc_ir_ia_i$ in $\Amc^u$. Each such element is the root of a copy
  of the unraveling $\Amc_\vee^u$ of $\Amc_\vee$, restricted to those
  individual names in $\Amc_\vee$ that are reachable from~$a_i$.  Since
  $\Tmc,\Amc_\vee \not\models C_i(a_i)$, we find a model $\Imc_i$ of
  \Tmc and $\Amc_\vee$ with $a_i \notin C_i^{\Imc_i}$. By unraveling
  $\Imc_i$, we obtain a model $\Imc^u_i$ of \Tmc and $\Amc^u_\vee$
  with $a_i \notin C_i^{\Imc^u_i}$. Combining the models
  $\Imc^u_0,\dots,\Imc^u_k$ in a suitable way, one can craft a model
  \Imc of \Tmc and $\Amc_\vee^u$ such that $crc_ir_ia_i \notin
  C_i^\Imc$ for all $i \leq k$ and the `role edges of \Imc' that
  concern the roles $r,r_0,\dots,r_k$ are exactly those in \Amc. This
  implies $\Imc \not\models q(c)$ as desired. 

  In some more detail, \Imc is obtained as follows. We can assume
  w.l.o.g.\ that the domains of $\Imc^u_0,\dots,\Imc^u_k$ are
  disjoint.  Take the disjoint union of $\Imc^u_0,\dots,\Imc^u_k$,
  renaming $a_i$ in $\Imc^u_i$ to $crc_ir_ia_i$ for all~$i$. Now take
  copies $\Jmc,\Jmc_0,\dots,\Jmc_k$ of any model of \Tmc, make sure
  that their domains are disjoint and that they are also disjoint from
  the domain of the model constructed so far. Additionally make sure
  that $c \in \Delta^\Jmc$ and $c_i \in \Delta^{\Jmc_i}$ for all
  $i$. Disjointly add these models to the model constructed so far.
  It can be verified that the model constructed up to this point is a
  model of \Tmc.  Add all role edges from \Amc that concern the roles
  $r,r_0,\dots,r_k$ to the resulting model, which has no impact on the
  satisfaction of \Tmc since $r,r_0,\dots,r_k$ do not occur in \Tmc.
  It can be verified that \Imc is as required.
\end{proof}
%


\section{Dichotomy for $\mathcal{ALCFI}$-TBoxes of Depth One}
\label{sect:depth1}

In practical applications, the concepts used in TBoxes are often of
very limited quantifier depth. Motivated by this observation, we
consider TBoxes of \emph{depth one} which are sets of CIs
$C \sqsubseteq D$ such that no restriction $\exists r . E$ or $\forall r . E$ in $C$ and $D$ is
in the scope of another restriction of the form $\exists r.E$ or $\forall r.E$.
To confirm that this is indeed a practically relevant case, we
have analyzed the 429 ontologies in the BioPortal repository,\footnote{The ontologies are available at 
https://bioportal.bioontology.org/ontologies.} finding that after removing all constructors that
are not part of $\ALCFI$, more than 80\% of them are of depth one.
The main result of this section is a dichotomy between \PTime and {\sc
  coNP} for TBoxes of depth one which is established by proving a
converse of Theorem~\ref{lem:utimpliesmat}, that is, showing that
materializability implies unraveling tolerance (and thus \PTime query
evaluation and even monadic Datalog$^{\neq}$-rewritability by
Theorem~\ref{thm:unravupperlem:hornalclem:hornalc}) for TBoxes of depth one. Together with
Theorem~\ref{thm:nomatlower}, which says that non-materializability
implies {\sc coNP}-hardness, this yields the dichotomy.

We remark that the same strategy cannot be used to obtain a dichotomy
in the case of unrestricted depth. In particular, the well-known
technique of rewriting a TBox into depth one by introducing fresh
concept names can change its complexity because it enables querying
for concepts such as $\neg A$ or $\forall r . A$ which are otherwise
`invisible' to (positive existential) queries. 
For TBoxes of unrestricted depth (and even in \ALC) it is in fact
neither the case that \PTime query evaluation implies unraveling
tolerance (or even Datalog$^{\not=}$-rewritability) nor that
materializability implies \PTime query
evaluation.  
This is formally established in Section~\ref{sect:dicho}.
\begin{theorem}
\label{thm:depthoneimpliesut}
  Every materializable $\mathcal{ALCFI}$-TBox of depth one is unraveling tolerant. 
\end{theorem}
\begin{proof} Let \Tmc be a materializable TBox of depth~one, \Amc an
  ABox, $C_0(x)$ an ELIQ, and $a_0 \in \mn{Ind}(\Amc)$ such that
  $\Tmc,\Amc^u \not \models C_0(a_0)$. We have to show that
  $\Tmc,\Amc \not\models C_0(a_0)$. It follows from
  $\Tmc,\Amc^u \not \models C_0(a_0)$ that $\Amc^u$ is consistent
  w.r.t.~\Tmc. There must thus be a materialization $\Imc^u$ for \Tmc
  and $\Amc^u$, despite the fact that $\Amc^u$ is infinite: by
  Theorem~\ref{thm:depthoneimpliesut}, \Tmc has the disjunction
  property and the argument from the proof of
  Theorem~\ref{thm:depthoneimpliesut} that the disjunction property
  implies materializability goes through without modification also for
  infinite ABoxes.  Our aim is to turn $\Imc^u$ into a model \Imc of
  \Amc and \Tmc such that $\Imc \not \models C_0(a_0)$. To achieve
  this, we first uniformize $\Imc^u$ in a suitable way. 

  We assume w.l.o.g.\ that $\Imc^u$ has forest-shape, i.e., that
  $\Imc^u$ can be constructed by selecting a tree-shaped
  interpretation $\Imc_\alpha$ with root $\alpha$ for each $\alpha \in
  \mn{Ind}(\Amc^u)$, then taking the disjoint union of all these
  interpretations, and finally adding role edges $(\alpha,\beta)$ to
  $r^{\Imc^u}$ whenever $r(\alpha,\beta) \in \Amc^u$. In fact, to
  achieve the desired shape we can take the i-unfolding of $\Imc^u$ defined
  and analysed in Lemmas~\ref{lem:propertiesofunfold} and~\ref{lem:sem}, where we start
  the i-unfolding from the elements of $\mn{Ind}(\Amc^u) \subseteq \Delta^{\Imc^u}$.

  We start with exhibiting a self-similarity inside the unraveled ABox
  $\Amc^u$ and inside $\Imc^u$.
  \\[2mm]
  {\bf Claim~1}.  For all $\alpha, \beta \in \mn{Ind}(\Amc^u)$ with
    $\mn{tail}(\alpha)=\mn{tail}(\beta)$,
  \begin{enumerate}

  \item 
    $\Amc^u \models C(\alpha)$ iff $\Amc^u \models C(\beta)$
    for all ELIQs $C(x)$;

  \item $\alpha \in C^{\Imc^u}$ iff $\beta \in C^{\Imc^u}$ for all
    concepts $C$ built only from concept names, $\neg$, and $\sqcap$.

  \end{enumerate}
  To establish Point~(1), take 
  $\alpha, \beta \in \mn{Ind}(\Amc^u)$ such that
  $\mn{tail}(\alpha)=\mn{tail}(\beta)$ and
  $\Amc^u \not\models C(\alpha)$. Then there is a model \Imc of
  $\Amc^u$ and \Tmc such that $\Imc \not\models C(\alpha)$. One can find
  a model \Jmc of $\Amc^u$ and \Tmc such that
  $\Jmc \not\models C(\beta)$, as follows. By construction of
  $\Amc^u$, there is an isomorphism $\iota: \mn{Ind}(\Amc^u) \rightarrow
  \mn{Ind}(\Amc^u)$ with $\iota(\alpha)=\beta$ such that 
$A(\gamma) \in \Amc^u$ iff $A(\iota(\gamma)) \in
    \Amc^u$ and
%
$r(\gamma,\gamma') \in \Amc^u$ iff
    $r(\iota(\gamma),\iota(\gamma')) \in \Amc^u$
%
  %
    for all $\gamma \in \mn{Ind}(\Amc^u)$, all concept names $A$, and
    all role names $r$. We obtain $\Jmc$ from \Imc by
    renaming each $\gamma \in \mn{Ind}(\Amc^u)$ with $\iota(\gamma)$.
%
    Point~(2) can be proved by a straightforward induction on $C$. The
    base case uses Point~(1) and the fact that $\Imc^u$ is a
    materialization of \Tmc and \Amc.  This finishes the proof of
    Claim~1.

%
%
%

  \smallskip

  Now for the announced uniformization of $\Imc^u$. What we want to
  achieve is that for all $\alpha,\beta \in \mn{Ind}(\Amc^u)$,
  $\mn{tail}(\alpha)=\mn{tail}(\beta)$ implies $\Imc_\alpha =
  \Imc_\beta$ (recall that $\Imc_\alpha$ is the tree component of
  $\Imc^u$ rooted at $\alpha$, and likewise for
  $\Imc_\beta$). Construct the interpretation $\Jmc^u$ as follows:
  \begin{itemize}

  \item for each $\alpha \in \mn{Ind}(\Amc^u)$ with
    $\mn{tail}(\alpha)=a$, take a copy $\Jmc_\alpha$ of $\Imc_a$
    with the root $a$ renamed to~$\alpha$;

  \item then $\Jmc^u$ is the disjoint union of all interpretations
    $\Jmc_\alpha$, $\alpha \in \mn{Ind}(\Amc^u)$,  extended with a role
    edge $(\alpha,\beta) \in r^{\Jmc^u}$ whenever $r(\alpha,\beta)
    \in \Amc^u$.

  \end{itemize}
  Our next aim is to show that $\Jmc^u$ is as required, that is, it is
  a model of $\Tmc$ and $\Amc^u$ and satisfies
  $\Jmc^u \not\models C_0(a_0)$.

  It is indeed straightforward to verify that $\Jmc^u$ is a model of
  $\Amc^u$: all role assertions are satisfied by construction; moreover, $A(\alpha) \in \Amc^u$ implies $A(a) \in \Amc^u$
  where $a=\mn{tail}(\alpha)$ , thus $a \in A^{\Imc_u}$ and
  $\alpha \in A^{\Jmc_u}$.

  Next, we show that $\Jmc^u$ is a model of \Tmc. Let
  $f:\Delta^{\Jmc^u} \rightarrow \Delta^{\Imc^u}$ be a mapping that
  assigns to each domain element of $\Jmc^u$ the original element in
  $\Imc^u$ of which it is a copy.
  \\[2mm]
  {\bf Claim 2}. 
$d \in C^{\Jmc^u}$ iff $f(d) \in C^{\Imc^u}$ for all $d \in \Delta^{ \Jmc^u}$ and \ALCI-concepts
  $C$ of depth one.
  \\[2mm]
    The proof of claim~2 is by induction on the structure of $C$. We
  assume w.l.o.g.\ that $C$ is built only from the constructors
  $\neg$, $\sqcap$, and $\exists r . C$. The base case, where $C$ is a
  concept name, is an immediate consequence of the definition of
  $\Jmc^u$. The case where $C= \neg D$ and $C=D_1 \sqcap D_2$ is
  routine. It thus remains to consider the case $C=\exists r . D$,
  where $D$ is built from $\neg$ and $\sqcap$ only.



  First let $d \in C^{\Jmc^u}$.  Then there is a
  $(d,e) \in r^{\Jmc^u}$ with $e \in D^{\Jmc^u}$. First assume that
  the edge $(d,e)$ was added to $r^{\Jmc^u}$ because $d=\alpha$ and
  $e=\beta$ for some $\alpha,\beta \in \mn{Ind}(\Amc^u)$ with
  $r(\alpha,\beta) \in \Amc^u$. Let $\mn{tail}(\alpha)=a$ and
  $\mn{tail}(\beta)=b$.  Then we have $f(\alpha)=a$ and
  $f(\beta)=b$. By construction of $\Amc^u$,
  $r(\alpha,\beta) \in \Amc^u$ implies that $\beta = \alpha r b$ or
  $\alpha = \beta r^- a$.  In both cases we have $r(a,b) \in \Amc$,
  thus $r(a,arb) \in \Amc^u$, thus $(a,arb) \in r^{\Imc^u}$. Since
  $\beta=e \in D^{\Jmc^u}$, induction hypothesis yields that $b \in D^{\Imc^u}$. From
  Point~(2) of Claim~1, we obtain $arb \in D^{\Imc^u}$ and we are
  done. Now assume that there is an $\alpha \in \mn{Ind}(\Amc^u)$ such
  that $(d,e) \in \Jmc_\alpha$. By construction of $\Jmc^u$, we then
  have $(f(d),f(e)) \in r^{\Imc^u}$ and induction hypothesis yields
  $f(e) \in D^{\Imc^u}$.
  
  Now let $f(d) \in C^{\Imc^u}$. Then there is an
  $(f(d),e) \in r^{\Imc^u}$ with $e \in D^{\Imc^u}$. First assume that
  $f(d)=\alpha$ and $e=\beta$ for some
  $\alpha,\beta \in \mn{Ind}(\Amc^u)$ with
  $r(\alpha,\beta) \in \Amc^u$. Since $f(d) \in \mn{Ind}(\Amc^u)$, we
  must have $d = \gamma \in \mn{Ind}(\Amc^u)$ and
  $f(d)=a \in \mn{Ind}(\Amc)$ with $\mn{tail}(\gamma)=a$. By
  construction of $\Amc^u$, $r(a,\beta) \in \Amc^u$ implies that
  $\beta = a r b$, thus $r(a,b) \in \Amc$, thus
  $r(\gamma,\delta) \in \Amc^u$ with (i)~$\delta=\gamma r b$ or
  (ii)~$\gamma=\delta r^- a$ and $\mn{tail}(\delta)=b$. Since
  $arb=e \in D^{\Imc^u}$, Point~(2) of Claim~1 yields
  $b \in D^{\Imc_u}$. Since $\mn{tail}(\delta)=b$ implies
  $f(\delta)=b$, induction hypothesis yields $\delta \in D^{\Jmc^u}$
  and we are done. Now assume that there is an
  $\alpha \in \mn{Ind}(\Amc^u)$ such that $(f(d),e) \in
  \Imc_\alpha$.
  By construction of $\Jmc^u$, $f(d)$ being in $\Imc_\alpha$ implies
  that $\alpha = a$ for some $a \in \mn{Ind}(\Amc)$ and that there is
  an $\alpha' \in \mn{Ind}(\Amc^u)$ such that $d$ is in
  $\Jmc_{\alpha'}$ and $\mn{tail}(\alpha')=a$. Again by construction
  of $\Jmc^u$, we thus find an $e'$ in $\Jmc_{\alpha'}$ with $f(e')=e$
  and $(d,e') \in r^{\Jmc_{\alpha'}} \subseteq r^{\Jmc^u}$. Induction
  hypothesis yields $e' \in D^{\Jmc^u}$. This finishes the proof of
  Claim~2.

  \smallskip It follows from Claim~2 that $\Jmc^u$ satisfies all CIs
  in \Tmc. To show that $\Jmc^u$ is a model of \Tmc, it remains to
  show that $\Jmc^u$ satisfies all functionality assertions. Thus, let
  $\mn{func}(r) \in \Tmc$ and $d \in \Delta^{\Jmc^u}$. If
  $d \notin \mn{Ind}(\Amc^u)$, then it is straightforward to verify
  that, by construction of $\Jmc^u$, $d$ has at most one $r$-successor
  in $\Jmc^u$. Now assume $d=\alpha \in \mn{Ind}(\Amc^u)$ and let
  $\mn{tail}(\alpha)=a$. By construction of $\Jmc^u$ and $\Amc^u$,
  $\alpha$ has the same number of $r$-successors in $\Jmc^u$ as $a$ in
  $\Imc^u$. Since $\Imc^u$ satisfies $\mn{func}(r)$, $\alpha$ can have
  at most one $r$-successor in~$\Jmc^u$.

\smallskip

The final condition that $\Jmc^u$ should satisfy is
$\Jmc^u \not\models C_0(a_0)$. Assume to the contrary. We view
$C_0(x_0)$ as a (tree-shaped) CQ. Take a homomorphism $h$ from
$C_0(x_0)$ to $\Jmc^u$ with $h(x_0)=a_0$. (In this proof we consider homomorphisms that do not have to preserve
any individual names.) Let the CQ $q(x)$ be obtained from $C_0(x_0)$ by identifying variables
$y_1,y_2$ whenever $h(y_1)=h(y_2)$. To achieve a contradiction, it
suffices to exhibit a homomorphism $h'$ from $q(x_0)$ to $\Imc^u$ with
$h'(x_0)=a_0$.  We start with setting $h'(x)=h(x)$ whenever
$h(x) \in \mn{Ind}(\Amc^u)$. Let $q'$ be obtained from $q(x_0)$ by
dropping all role atoms $r(x,y)$ with $h'(x)$ and $h'(y)$ already
defined (which are satisfied under $h'$ by construction of $\Jmc^u$
and since $\Imc^u$ is a model of \Amc). Because of the forest shape of
$\Jmc^u$ and by construction, $q'$ is a disjoint union of ELIQs such
that, in each ELIQ $C(x)$ contained in $q'$, $h'$ is defined for the
root $x$ of $C(x)$, but not for any other variable in
it. Consequently, it suffices to show that whenever
$\Jmc_\alpha \models C(\alpha)$ for some ELIQ $C(x)$ and
$\alpha \in \mn{Ind}(\Amc^u)$, then $\Imc^u \models C(\alpha)$; the
remaining part of $h'$ can then be constructed in a straightforward
way. Now $\Jmc_\alpha \models C(\alpha)$ implies $\Imc_a \models C(a)$
where $a=\mn{tail}(\alpha)$ by choice of $\Jmc_\alpha$, which yields
$\Imc^u \models C(a)$ and thus $\Imc^u \models C(\alpha)$ by Point~(1)
of Claim~1.

  \medskip This finishes the construction and analysis of the uniform
  model $\Jmc^u$. It remains to convert $\Jmc^u$ into a model \Imc of
  \Tmc and the original ABox \Amc such that $\Imc \not\models C_{0}(a_{0})$:
  \begin{itemize}

  \item take the disjoint union of the components $\Jmc_{a}$ of
    $\Jmc^u$, for each $a \in \mn{Ind}(\Amc)$;
    

  \item add the edge $(a,b)$ to $r^{\Imc}$ whenever $r(a,b)
    \in \Amc$.

  \end{itemize}
  It is straightforward to verify that \Imc is a model of \Amc: all
  role assertions are satisfied by construction of \Imc; moreover,
  $A(a) \in \Amc$ implies $A(a) \in \Amc^u$ implies $a \in A^{\Jmc^u}$
  implies $a \in A^\Imc$.  To show that \Imc is a model of \Tmc and
  that $\Imc \not\models C_0(a_0)$, we first observe the following.  A
  \emph{bisimulation between interpretations $\Imc_1$ and
$\Imc_2$} is a relation
  $S \subseteq \Delta^{\Imc_1} \times \Delta^{\Imc_2}$ such that
%
%
%
\begin{enumerate}
\item for all $A\in \NC$ and $(d_{1},d_{2})\in S$: $d_{1}\in
  A^{\Imc_{1}}$ iff $d_{2}\in A^{\Imc_{2}}$; 
\item for all $r\in \NR \cup \{s^- \mid s \in \NR\}$: if $(d_{1},d_{2})\in S$ and $(d_{1},d_{1}')\in r^{\Imc_{1}}$,
then there exists $d_{2}' \in \Delta^{\Imc_{2}}$ such that $(d_{1}',d_{2}')\in S$ and 
$(d_{2},d_{2}')\in r^{\Imc_{2}}$;
\item for all $r\in \NR \cup \{s^- \mid s \in \NR\}$: if $(d_{1},d_{2})\in S$ and $(d_{2},d_{2}')\in r^{\Imc_{2}}$,
then there exists $d_{1}' \in \Delta^{\Imc_{1}}$ such that $(d_{1}',d_{2}')\in S$ and 
$(d_{1},d_{1}')\in r^{\Imc_{2}}$.
\end{enumerate}
Recall that, whenever there is a bisimulation $S$ between $\Imc_1$ and
$\Imc_2$ with $(d,e) \in S$, then $d \in C^{\Imc_1}$ iff
$e \in C^{\Imc_2}$ for all \ALCI-concepts $C$ \cite{goranko20075,TBoxpaper}.
\\[2mm]
{\bf Claim 3}. There is a bisimulation $S$ between $\Jmc^u$ and
$\Imc$ such that $(a,a) \in S$ for all $a \in \mn{Ind}(\Amc)$.
\\[2mm]
Since $\Jmc^u$ is uniform in the sense that $\Jmc_\alpha$ is
isomorphic to $\Jmc_\beta$ whenever
$\mn{tail}(\alpha)=\mn{tail}(\beta)$, we find a bisimulation between
$\Jmc_\alpha$ and $\Jmc_a$ whenever $\mn{tail}(\alpha)=a$. It can be
verified that the union of all these bisimulations qualifies as the
desired bisimulation $S$ for Claim~3.  Thus, Claim~3 is proved.

\smallskip

It follows from Claim~3 that $\Imc$ satisfies all concept inclusions
in \Tmc, and that $\Imc \not\models C_0(a_0)$. It thus remains to
verify that \Imc also satisfies all functionality assertions in
\Tmc. This can be done in the same way in which we have verified
that $\Jmc^u$ satisfies all those assertions.
\end{proof}

The desired dichotomy follows: If an $\mathcal{ALCFI}$-TBox \Tmc of
depth one is materializable, then PEQ-evaluation w.r.t.\ \Tmc is in
{\sc PTime} and monadic Datalog$^{\neq}$-rewritable by
Theorems~\ref{thm:depthoneimpliesut}
and~\ref{thm:unravupperlem:hornalclem:hornalc}. Otherwise, ELIQ-evaluation w.r.t.\ \Tmc is
{\sc coNP}-complete by  Theorem~\ref{thm:nomatlower}.
\begin{theorem}[Dichotomy]
\label{thm:depthonedicho}
  For every $\mathcal{\ALCFI}$-TBox \Tmc of depth one, one of the
  following is true:
  \begin{itemize}

  \item \Qmc-evaluation w.r.t.\ \Tmc is in \PTime for any $\Qmc \in \{
    $PEQ$, $CQ$, $ELIQ$ \}$ (and monadic Datalog$^{\neq}$-rewritable);

  \item \Qmc-evaluation w.r.t.\ \Tmc is {\sc coNP}-complete for any
    $\Qmc \in \{ $PEQ$, $CQ$, $ELIQ$ \}$.

  \end{itemize}
\end{theorem}
For example of depth one TBoxes for which query evaluation is in
\PTime and for which it is {\sc coNP}-hard, please see
Example~\ref{ex5}; there, cases~(1) and~(2) are materializable and
thus in \PTime while case~(3) is not materializable and thus {\sc
  coNP}-hard.

%
\section{Query Evaluation in $\mathcal{ALC}/\mathcal{ALCI}$ = CSP}
\label{sect:dicho}

We drop functional roles and consider TBoxes formulated in \ALC and in
\ALCI showing that query evaluation w.r.t.\ TBoxes from these classes
has the same computational power as non-uniform CSPs, in the following
sense:
\begin{enumerate}

\item for every OMQ $(\Tmc,q)$ with $\Tmc$ an \ALCI-TBox and $q$ an 
  ELIQ, there is a CSP such that the complement of the CSP and the 
  query evaluation problem for the $(\Tmc,q)$ are 
  reducible to each other in polynomial time;

\item for every CSP, there is an \ALC-TBox $\Tmc$ such that the CSP is
  equivalent to the complement of evaluating an OMQ
  $(\Tmc,\exists x \, M(x))$ and, conversely, for \emph{every} OMQ
  $(\Tmc,q)$ query evaluation can be reduced in polynomial time to the
  CSP's complement.

\end{enumerate}
This result has many interesting consequences. In particular, the
\PTime/{\sc NP}-dichotomy for non-uniform CSPs~\cite{Dicho1,Dicho2},
formerly also known as the Feder-Vardi conjecture, yields a {\sc
  PTime}/{\sc coNP}-dichotomy for query evaluation w.r.t.\
$\mathcal{ALC}$-TBoxes (equivalently: w.r.t.\
$\mathcal{ALCI}$-TBoxes).
%
%
%
Remarkably, all this is true already for \emph{materializable} TBoxes. By Theorem~\ref{equivalence}
and since we carefully choose the appropriate query language in each
technical result below, it furthermore holds for any of the query languages
ELIQ, CQ, and PEQ (and ELQ for \ALC-TBoxes).

\smallskip

We begin by introducing CSPs. Since every CSP is equivalent to a CSP
with a single predicate that is binary, up to polynomial time
reductions \cite{DBLP:journals/siamcomp/FederV98}, we consider CSPs
over unary and binary predicates (concept names and role names) only.
A \emph{signature} $\Sigma$ is a finite set of concept and role
names. We use $\mn{sig}(\Tmc)$ to denote the set of all concept and
role names that occur in the TBox \Tmc.  An ABox $\Amc$ is a
\emph{$\Sigma$-ABox} if all concept and role names in $\Amc$ are in
$\Sigma$. Moreover, we write $\Amc|_{\Sigma}$ to denote the
restriction of an ABox \Amc to the assertions that use a symbol from
$\Sigma$.
For two finite $\Sigma$-ABoxes $\Amc$ and $\Bmc$, we write $\Amc
\rightarrow \Bmc$ if there is a
homomorphism from $\mathcal{A}$ to $\mathcal{B}$ 
that does not have to preserve any individual names.
A $\Sigma$-ABox $\Bmc$ gives rise to the following
(non-uniform) constraint satisfaction problem ${\sf CSP}(\mathcal{B})$: given a finite $\Sigma$-ABox
  $\mathcal{A}$, decide whether $\mathcal{A} \rightarrow \mathcal{B}$.
$\Bmc$ is called the \emph{template} of ${\sf CSP}(\Bmc)$.
Many problems in {\sc NP} can be given in the form ${\sf
  CSP}(\mathcal{B})$. For example,
$k$-colorability 
is ${\sf CSP}(\Cmc_{k})$, where $\Cmc_{k}$ is the $\{r\}$-ABox
that contains $r(i,j)$ for all $1\leq i\not= j \leq k$.

We now formulate and prove Points (1) and (2) from above, starting
with (1).  The following is proved in
\cite{DBLP:journals/tods/BienvenuCLW14}.
\begin{theorem}\label{thm:tods}
For every $\mathcal{ALCI}$-TBox $\Tmc$ and ELIQ $C(x)$, one can compute a template $\Bmc$ in exponential time such that
the query evaluation problem for the OMQ $(\Tmc,C(x))$ and the complement of
${\sf CSP}(\Bmc)$ are polynomial time
reducible to each other.
\end{theorem}
The proof of Theorem~\ref{thm:tods} given in
\cite{DBLP:journals/tods/BienvenuCLW14} proceeds in two steps. To deal
with answer variables, it uses \emph{generalized CSPs with constants},
defined by a finite set of templates (instead of a single one) and
admitting the inclusion of constant symbols in the signature of the
CSP (instead of only relation symbols). One shows that (i) for every
OMQ $(\Tmc,C(x))$, one can construct a generalized CSP with a single
constant whose complement is mutually reducible in polynomial time with the
query evaluation problem for $(\Tmc,C(x))$ and (ii) every generalized
CSP with constants is mutually reducible in polynomial time with some
standard CSP.  For the reader's convenience, we illustrate the
construction of the template from a given OMQ, concentrating on
\emph{Boolean ELIQs} which are of the form $\exists x \, C(x)$ with
$C(x)$ an ELIQ. In this special case, one can avoid the use of generalized
CSPs with constants.

\begin{theorem}\label{thm:red2}
  Let $\Tmc$ be an $\mathcal{ALCI}$-TBox, $q=\exists x \, C(x)$ with $C(x)$ an ELIQ,
  and $\Sigma$ the signature of \Tmc and $q$. 
  Then one can construct (in time single exponential in $|\Tmc|+|C|$)
  a $\Sigma$-template $\Bmc_{\Tmc,q}$ such that for all ABoxes~$\Amc$:
\begin{equation}
 \Tmc,\Amc \models q \ \text{ iff } \
\Amc|_{\Sigma} \not\rightarrow \Bmc_{\Tmc,q}
\tag{\textsf{HomDual}}
\end{equation}
\end{theorem}
\begin{proof} Assume $\Tmc$ and $q=\exists x \, C(x)$ are given. We use the notation from the proof of Theorem~\ref{thm:unravupperlem:hornalclem:hornalc}.
  Thus, $\mn{cl}(\Tmc,C)$ denotes the closure under single negation of the set of subconcepts of $\Tmc$ and $C$, $\mn{tp}(\Tmc,q)$
  denotes the set of $\Tmc,q$-types and for $t,t'\in \mn{tp}(\Tmc,q)$ we write $t\rightsquigarrow_r t'$ if $t$ and $t'$ can be satisfied in
  domain elements of a model of $\Tmc$ that are related by $r$. Now, a \emph{$\Tmc,q$-type $t$ omits~$q$} 
  if it is satisfiable in a model $\Imc$ of $\Tmc$ with $C^{\Imc}=\emptyset$. Let $\Bmc_{\Tmc,q}$ be the set of
  assertions $A(t)$ such $t$ omits $q$ and $A\in t$ and $r(t,t')$ such that $t$ and $t'$ omit $q$ and 
  $t \rightsquigarrow_r t'$.
It is not difficult to show that condition $({\sf HomDual})$ holds for all ABoxes $\Amc$.
Observe that $\Bmc_{\Tmc,q}$ can be constructed in exponential time since the set of $\Tmc,q$-types omitting $q$
can be constructed in exponential time.
\end{proof}
\begin{example}
Let $\Tmc=\{A \sqsubseteq \forall r.B\}$ and define $q=\exists x
  \, B(x)$.
Then up to isomorphism, $\Bmc_{\Tmc,q}$ is 
$\{r(a,a), r(a,b), A(b), r(a,c)\}$.
\end{example}
As a consequence of Theorem~\ref{thm:tods}, we obtain the following
dichotomy result.
\begin{theorem}[Dichotomy]
    For every $\mathcal{\ALCI}$-TBox \Tmc, one of the
  following is true:
  \begin{itemize}

  \item \Qmc-evaluation w.r.t.\ \Tmc is in \PTime for any $\Qmc \in \{
    $PEQ$, $CQ$, $ELIQ$ \}$;

  \item \Qmc-evaluation w.r.t.\ \Tmc is {\sc coNP}-complete for any
    $\Qmc \in \{ $PEQ$, $CQ$, $ELIQ$ \}$.

  \end{itemize}
  %
%
For \ALC-TBoxes, this dichotomy additionally holds for ELQs.
\end{theorem} 
\begin{proof}
%
  Assume to the contrary of what is to be shown that there is an
  $\mathcal{ALCI}$-TBox \Tmc such that \Qmc-evaluation w.r.t.\ \Tmc is neither in
  {\sc PTime} nor {\sc coNP}-hard, for some $\Qmc \in \{ $PEQ$, $CQ$,
  $ELIQ$ \}$.  Then by Theorem~\ref{equivalence}, the same holds for
  ELIQ-evaluation w.r.t.~\Tmc. It follows that there is a concrete
  ELIQ~$q$ such that query evaluation for $(\Tmc,q)$ is {\sc
    coNP}-intermediate. By Theorem~\ref{thm:tods}, there is a template
  $\Bmc$ such that evaluating $(\Tmc,q)$ is mutually reducible in
  polynomial time with the complement of ${\sf CSP}(\Bmc)$.  Thus
  ${\sf CSP}(\Bmc)$ is \NP-intermediate, a contradiction to the fact
  that there are no such CSPs~\cite{Dicho1,Dicho2}.
\end{proof}
%
%
We now establish Point~(2) from the beginning of the section.
In a sense, the following provides a converse to Theorem~\ref{thm:tods}.
%
%
%
%
\begin{theorem}\label{thm:red0}
For every template \Bmc over signature $\Sigma$, 
one can construct in polynomial time a materializable
$\mathcal{ALC}$-TBox $\Tmc_{\mathcal{B}}$ such that, for a
distinguished
concept name $M$, the following hold:
\begin{enumerate}

\item ${\sf CSP}(\Bmc)$ is equivalent to the complement of the OMQ
  $(\Tmc_{\Bmc},\exists x \, M(x))$ in the sense that for
  every $\Sigma$-ABox \Amc, $\Amc \rightarrow \Bmc$ iff
  $\Tmc_\Bmc,\Amc \not\models \exists x \, M(x)$;

\item the query evaluation problem for $(\Tmc_{\Bmc},q)$ is
  polynomial time reducible to the complement of 
${\sf CSP}(\mathcal{B})$, for all PEQs $q$.

\end{enumerate}
\end{theorem}
Note that the equivalence formulated in Point~(1) implies polynomial
time reducibility of ${\sf CSP}(\Bmc)$ to the complement of
$(\Tmc_{\Bmc},\exists x \, M(x))$ and vice versa, but is much stronger
than that.

%
Our approach to proving Theorem~\ref{thm:red0} is to generalize the
reduction of $k$-colorability to query evaluation w.r.t.\ \ALC-TBoxes
discussed in Examples~\ref{ex1} and~\ref{ex3}, where the main
challenge is to overcome the observation from Example~\ref{ex3} that
\PTime CSPs such as 2-colorability may be translated into {\sc
  coNP}-hard TBoxes. Note that this is due to the disjunction in the
TBox $\Tmc_k$ of Example~\ref{ex1}, which causes
non-materializability. Our solution is to replace the concept names
$A_{1},\ldots,A_{k}$ in $\Tmc_k$ with suitable compound concepts that
are `invisible' to the (positive existential) query. Unlike the
original depth one TBox $\Tmc_k$, the resulting TBox is of depth
three. This `hiding' of concept names also plays a crucial role in the
proofs of non-dichotomy and undecidability presented in
Section~\ref{sec:undec}.\footnote{The `hiding technique' introduced
  here 
has been adopted in~\cite{DBLP:conf/ijcai/BotoevaLRWZ16} 
in the context of query inseparability.} We now formally develop this idea and
establish some crucial properties of the TBoxes that are obtained by
hiding concept names (which are called enriched abstractions below).
We return to the proof of Theorem~\ref{thm:red0} afterwards.

Let \Tmc be an \ALCI-TBox and $\Sigma \subseteq \mn{sig}(\Tmc)$ a
signature that contains all role names in \Tmc. Our aim is to hide all
concept names that are not in $\Sigma$.  For $B\in \NC \setminus
\Sigma$, let $Z_{B}$ be a fresh concept name and let $r_{B}$ and
$s_{B}$ be fresh role names. The \emph{abstraction of $B$} is the
$\mathcal{ALC}$-concept $H_{B}:= \forall r_{B}.\exists s_{B}.\neg
Z_{B}$.  The \emph{$\Sigma$-abstraction $C'$ of a (potentially
  compound) concept $C$} is obtained from $C$ by replacing every
$B\in \NC \setminus \Sigma$ with~$H_{B}$. The \emph{$\Sigma$-abstraction of a
  TBox $\Tmc$} is obtained from $\Tmc$ by replacing all concepts in
$\Tmc$ with their $\Sigma$-abstractions.  We associate with \Tmc and
$\Sigma$ an auxiliary TBox
$$
\Tmc^{\exists} = \{ \top \sqsubseteq \exists r_{B}.\top, \top \sqsubseteq \exists s_{B}.Z_{B}\mid B \in \Sigma\}
$$
Finally, $\Tmc'\cup \Tmc^{\exists}$ is called the \emph{enriched
  $\Sigma$-abstraction} of $\Tmc$ and $\Sigma$. To hide the concept
names that are not in $\Sigma$, we can replace a TBox \Tmc with its
enriched abstraction. The following example shows that the TBox
$\Tmc^{\exists}$ is crucial for this: without $\Tmc^\exists$,
disjunctions in \Tmc over concept names from $\Sigma$ can still induce
disjunctions in the $\Sigma$-abstraction.
\begin{example}\label{ex:abstraction} 
  Let $\Tmc= \{A \sqsubseteq \neg B_{1} \sqcup \neg B_{2}\}$ and
  $\Sigma=\{ A \}$.  Then
  $\Tmc'= \{A \sqsubseteq \neg H_{B_{1}} \sqcup \neg H_{B_{2}}\}$ is
  the $\Sigma$-abstraction of $\Tmc$.  For $\Amc = \{A(a)\}$, we
  derive
  $\Tmc',\Amc\models \exists r_{B_{1}}.\top(a) \vee \exists r_{B_{2}}.\top(a)$
  but $\Tmc',\Amc\not\models \exists r_{B_{1}}.\top(a)$ and
  $\Tmc',\Amc\not\models \exists r_{B_{2}}.\top(a)$. Thus $\Tmc'$ does not
  have the ABox disjunction property and is not materializable. In
  contrast, it follows from Lemma~\ref{lem:abstraction} below that
  $\Tmc'\cup \Tmc^{\exists}$ is materializable and, in fact,
  $\Tmc'\cup \Tmc^{\exists},\Amc\models q(a)$ iff
  $\Tmc^{\exists},\Amc\models q(a)$ holds for all PEQs $q$.
\end{example}
In the proof of Theorem~\ref{thm:red0} and in Section~\ref{sec:undec},
we work with TBoxes that enjoy two crucial properties which ensure a
good behaviour of enriched $\Sigma$-abstractions. We introduce these
properties next.

A TBox $\Tmc$ \emph{admits trivial models} if the singleton
interpretation $\Imc$ with $X^{\Imc}=\emptyset$ for all $X\in \NC \cup
\NR$ is a model of $\Tmc$.
It is \emph{$\Sigma$-extensional} if for every $\Sigma$-ABox $\Amc$
consistent w.r.t.~$\Tmc$, there exists a model $\Imc$ of $\Tmc$ and
$\Amc$ such that $\Delta^{\Imc}= {\sf Ind}(\Amc)$, $A^{\Imc}= \{ a
\mid A(a)\in \Amc\}$ for all concept names $A\in \Sigma$, and
$r^{\Imc}= \{ (a,b)\mid r(a,b)\in \Amc\}$ for all role names
$r\in\Sigma$.

%
%

The following lemma summarizes the main properties of abstractions.
Point~(1) relates consistency of ABoxes w.r.t.\ a TBox \Tmc to
consistency w.r.t.\ their enriched $\Sigma$-abstraction $\Tmc' \cup
\Tmc^\exists$. Note that the ABox \Amc might contain the fresh symbols
from $\Tmc'$ but these have no impact on consistency (as witnessed by
the use of $\Amc|_\Sigma$ rather than $\Amc$ on the left-hand side of
the equivalence). Point~(2) is similar to Point~(1) but concerns the
evaluation of OMQs based on \Tmc and based on $\Tmc' \cup
\Tmc^\exists$; we only consider a restricted form of actual queries
that are sufficient for the proofs in Section~\ref{sec:undec}.
Points~(3) and (4) together state that evaluating OMQs $(\Tmc'\cup
\Tmc^{\exists},q)$ with $q$ a PEQ is tractable on ABoxes whose
$\Sigma$-part is consistent w.r.t.~$\Tmc$.
\begin{lemma}\label{lem:abstraction}
  Let $\Tmc$ be an $\mathcal{ALCI}$-TBox, $\Sigma\subseteq {\sf sig}(\Tmc)$ contain all role names in $\Tmc$, 
  and assume that $\Tmc$ is $\Sigma$-entensional and admits trivial models.
  Let $\Tmc' \cup \Tmc^\exists$ be the enriched $\Sigma$-abstraction of \Tmc. Then for
  every ABox $\Amc$ and all concept names $A$ (that are not among
  the fresh symbols in $\Tmc'$):
\begin{enumerate}
\item $\Amc|_{\Sigma}$ is consistent w.r.t.~$\Tmc$ iff $\Amc$ is consistent w.r.t.~$\Tmc'\cup \Tmc^{\exists}$;

\item for all $a\in {\sf Ind}(\Amc)$ and the $\Sigma$-abstraction $A'$ of $A$:
$$
\Tmc,\Amc|_{\Sigma}\models A(a) \quad \text{ iff } \quad  \Tmc' \cup \Tmc^{\exists},\Amc \models A'(a)
$$
and 
$$
 \Tmc,\Amc|_{\Sigma}\models \exists x \, A(x) \quad \text{ iff
 } \quad \Tmc' \cup \Tmc^{\exists},\Amc \models \exists x \, A'(x);
$$
\item $\Tmc^{\exists}$ is monadic Datalog$^{\neq}$-rewritable for PEQs;
\item if $\Amc|_{\Sigma}$ is consistent w.r.t.~$\Tmc$, then
$$
\Tmc'\cup \Tmc^{\exists},\Amc \models q(\vec{a}) \quad \text{ iff } \quad \Tmc^{\exists},\Amc \models q(\vec{a})
$$
for all PEQs $q$ and all $\vec{a}$.
\end{enumerate}
\end{lemma}
\begin{proof}
(1) Assume first that $\Amc$ is consistent w.r.t.~$\Tmc'\cup \Tmc^{\exists}$.
We show that $\Amc|_{\Sigma}$ is consistent w.r.t.~$\Tmc$.
Take a model $\Imc$ of $\Tmc'\cup \Tmc^{\exists}$ and $\Amc$.
Define an interpretation $\Jmc$ in the same way as $\Imc$ except that
$B^{\Jmc}:=H_{B}^{\Imc}$ for all $B\in\NC\setminus \Sigma$. It is straightforward
to show by induction for all $\mathcal{ALCI}$-concepts $D$ not using the
fresh symbols from $\Sigma$-abstractions and their
$\Sigma$-abstractions $D'$: $d\in D^{\Jmc}$ iff $d\in
D'^{\Imc}$, for all $d\in \Delta^{\Imc}$.
Thus $\Jmc$ is a model of $\Tmc$ and $\Amc|_{\Sigma}$ and it follows that $\Amc|_{\Sigma}$ is
consistent w.r.t.~$\Tmc$.

Now assume that $\Amc|_{\Sigma}$ is consistent w.r.t.~$\Tmc$. We show that
$\Amc$ is consistent w.r.t.~$\Tmc'\cup \Tmc^{\exists}$.
Take a model $\Imc$ of $\Tmc$ and $\Amc|_{\Sigma}$.  
Construct a model $\Jmc$ of $\Tmc'\cup \Tmc^{\exists}$ and $\Amc$ as follows:
$\Delta^{\Jmc}$ is the set of words $w= d v_{1}\cdots v_{n}$
such that $d\in \Delta^{\Imc}$ and $v_{i}\in
\{r_{B},s_{B},\bar{s}_{B}\mid B \in \NC \setminus \Sigma\}$ where
$v_{i}\not=\bar{s}_{B}$ if (i) $i>2$ or (ii) $i=2$ and ($d\not\in
H_{B}^{\Imc}$ or $v_{1}\not=r_{B}$).  Now let
\begin{eqnarray*}
A^{\Jmc}  &  = & A^{\Imc} \text{ for all $A\in\NC \cap \Sigma$ }\\
B^{\Jmc} &  = &\{ d\in {\sf Ind}(\Amc) \mid B(d)\in \Amc\}
\text{ for all $B\in \NC \setminus \Sigma$}\\
Z_{B}^{\Jmc} & = & Z_{B}^{\Imc} \cup \{w \mid {\sf tail}(w) =
s_{B}\} \text{ for all $B\in \NC \setminus \Sigma$}\\ 
r^{\Jmc} & =  & r^{\Imc} \text{ for all $r \in \NC \cap \Sigma$}\\
r_{B}^{\Jmc} &= & r_{B}^{\Imc} \cup \{(w,wr_{B}) \mid wr_{B}\in
\Delta^{\Jmc}\} \text{ for all $B\in \NC \setminus \Sigma$}\\				
s_{B}^{\Jmc} &= & s_{B}^{\Imc} \cup \{(w,ws_{B}) \mid wr_{B}\in \Delta^{\Jmc}\} \cup 
\{(w,w\bar{s}_{B}) \mid w\bar{s}_{B}\in \Delta^{\Jmc}\} \text{
  for all $B\in \NC \setminus \Sigma$ }
\end{eqnarray*}					
It follows directly from the construction of $\Jmc$ that $H_{B}^{\Jmc}=B^{\Imc}$,
for all $B\in \NC \setminus \Sigma$. Thus, for all concepts $D$ (not
using fresh symbols from $\Sigma$-abstractions) and their
$\Sigma$-abstractions $D'$ and all $d\in \Delta^{\Imc}$: $d\in
D'^{\Jmc}$ iff $d\in D^{\Imc}$. Thus, the CIs of $\Tmc'$ hold
in all $d\in \Delta^{\Imc}$ since the CIs of $\Tmc$ hold in all $d\in
\Delta^{\Imc}$. The CIs of $\Tmc'$ also hold in all $d\in
\Delta^{\Jmc}\setminus\Delta^{\Imc}$ since $\Tmc$ admits trivial
models. Thus, $\Jmc$ is a model of $\Tmc'$. Since $\Jmc$
is a model of $\Tmc^{\exists}$ by construction, it follows that
$\Jmc$ is a model of $\Tmc'\cup \Tmc^{\exists}$. 

(2) can be proved using the models constructed in the proof of (1).

(3) is a consequence of the fact that $\Tmc^\exists$ can be viewed as
a TBox formulated in the description logic DL-Lite$_\Rmc$ and that any
OMQ $(\Tmc,q)$ with \Tmc a DL-Lite$_\Rmc$-TBox and $q$ a PEQ is
known to be rewritable into a union of CQs~\cite{CDLLR07,DBLP:journals/jair/ArtaleCKZ09}.


(4) Assume that $\Amc|_{\Sigma}$ is consistent
w.r.t.~$\Tmc$ and that $\Tmc^{\exists},\Amc\not\models q(\vec{a})$. We
show $\Tmc'\cup \Tmc^{\exists},\Amc\not\models q(\vec{a})$.  Note
first that one can construct a hom-initial model
$\Imc_{\Amc}^{\exists}$ of $\Tmc^{\exists}$ and $\Amc$ in the same way
as $\Jmc$ was constructed from $\Imc$ in the proof of Point~(2) (by replacing $\Imc$
with the interpretation $\Imc_{\Amc}$ corresponding to $\Amc$ and not
using the symbols $\bar{s}_{B}$ in the
construction). Thus, $\Delta^{\Imc_{\Amc}^{\exists}}$ is the set of
words $w= a v_{1}\cdots v_{n}$ such that $a\in {\sf Ind}(\Amc)$ and
$v_{i}\in \{r_{B},s_{B},\mid B \in \NC \setminus \Sigma\}$.  We have
$\Imc^{\exists}_{\Amc}\not\models q(\vec{a})$.  Now, as $\Tmc$ is $\Sigma$-extensional, there is a
model $\Imc$ of $\Tmc$ and $\Amc|_{\Sigma}$ with
$\Delta^{\Imc}= {\sf Ind}(\Amc)$ and $A^{\Imc}= \{ a \mid A(a)\in
\Amc\}$ for all $A\in \Sigma$, and $r^{\Imc}= \{ (a,b)\mid
r(a,b)\in \Amc\}$ for all role names~$r\in \Sigma$. Construct the model
$\Jmc$ of $\Tmc'\cup \Tmc^{\exists}$ and $\Amc$ in the same way
as in the proof of Point~(2). Define a mapping $h:\Jmc\rightarrow
\Imc_{\Amc}^{\exists}$ by setting $h(w)=w'$, where $w'$ is obtained
from $w$ by replacing every $\bar{s}_{B}$ by $s_{B}$. One can 
show that $h$ is a homomorphism. Thus $\Jmc\not\models
q(\vec{a})$ and so $\Tmc'\cup \Tmc^{\exists},\Amc \not\models
q(\vec{a})$, as required. The converse direction is trivial.
\end{proof}

We are now ready to prove Theorem~\ref{thm:red0}. 

\begin{proof}[Proof of Theorem~\ref{thm:red0}] Assume a
$\Sigma$-template $\Bmc$ is given.  We construct the TBox
$\Tmc_{\Bmc}$ in two steps. First take for any $d\in {\sf Ind}(\Bmc)$
a concept name $A_{d}$ and define a TBox $\Hmc_{\Bmc}$ with the
following CIs:
$$
\begin{array}{r@{\;}c@{\;}l@{\;}l@{}l}
   {\sf dom} & \sqsubseteq  & \midsqcup_{d\in {\sf Ind}(\Bmc)} A_{d} \\[1mm]
   A_{d} \sqcap A_{e} &\sqsubseteq& \bot & \hspace*{-6mm} \text{for all } d,e\in {\sf Ind}(\Bmc), d\not=e \\[1mm]
  A_{d} \sqcap \exists r. A_{e} &\sqsubseteq& \bot & \hspace*{-6mm} \text{for all } d,e\in {\sf Ind}(\Bmc), r \in \Sigma, r(d,e)\not\in \Bmc\\[1mm]
  A_d \sqcap A &\sqsubseteq& \bot & \hspace*{-6mm} \text{for all } d\in {\sf Ind}(\Bmc), A \in \Sigma, A(d)\not\in \Bmc.
\end{array}
$$
Here ${\sf dom} \sqsubseteq \midsqcup_{d\in {\sf Ind}(\Bmc)} A_{d}$ stands for the set of CIs 
$$
 \exists r .\top \sqsubseteq \midsqcup_{d\in {\sf Ind}(\Bmc)} A_{d}, \quad
 A \sqsubseteq \midsqcup_{d\in {\sf Ind}(\Bmc)} A_{d}, \quad
 \top \sqsubseteq \forall r. (\midsqcup_{d\in {\sf Ind}(\Bmc)} A_{d})
$$
where $r$ and $A$ range over all role and concept names in $\Sigma$,
respectively.  We use a CI of the form ${\sf dom} \sqsubseteq C$
rather than $\top \sqsubseteq C$ to ensure that the TBox $\Hmc_{\Bmc}$
admits trivial models. It should also be clear that $\Tmc$ is $\Sigma$-extensional.
Now let $M$ be a fresh concept name. Then the
following can be proved in a straightforward way.

\medskip
\noindent 
{\bf Claim 1.} For any ABox $\Amc$ the following conditions are equivalent:
\begin{enumerate}
\item $\Hmc_{\Bmc},\Amc|_{\Sigma}\not\models \exists x \, M(x)$;
\item $\Amc|_{\Sigma}$ is consistent w.r.t.~$\Hmc_{\Bmc}$;
\item $\Amc|_{\Sigma} \rightarrow \Bmc$.
\end{enumerate}  
Thus, ${\sf CSP}(\Bmc)$ and the complement of the query evaluation problem for $(\Hmc_{\Bmc},\exists x \, M(x))$
are reducible to each other in polynomial time. Because of the
disjunctions, however, the query evaluation problem w.r.t~$\Hmc_{\Bmc}$ is
typically {\sc coNP}-hard even if ${\sf CSP}(\Bmc)$ is in \PTime. 

In the second step, we thus `hide' the concept names $A_{d}$ by
replacing them with their abstractions $H_{A_{d}}$.  Let $\Hmc_{\Bmc}'
\cup \Tmc^{\exists}$ be the enriched $\Sigma$-abstraction of
$\Hmc_{\Bmc}$. From Claim~1 and Lemma~\ref{lem:abstraction} (1)
according to which $\Amc|_{\Sigma}$ is consistent w.r.t.~$\Hmc_{\Bmc}$
iff $\Amc$ is consistent w.r.t.~$\Hmc_{\Bmc}'\cup \Tmc^{\exists}$, we
obtain

\medskip
\noindent
{\bf Claim 2.} For any ABox $\Amc$ not containing the concept name $M$, the following conditions are
equivalent:
\begin{enumerate}
\item $\Hmc_{\Bmc}'\cup \Tmc^{\exists},\Amc \not\models \exists x \, M(x)$;
\item $\Amc$ is consistent w.r.t.~$\Hmc_{\Bmc}'\cup \Tmc^{\exists}$;
\item $\Amc|_{\Sigma} \rightarrow \Bmc$.
\end{enumerate} 
Let $\Tmc_{\Bmc}=  \Hmc_{\Bmc}'\cup \Tmc^{\exists}$ be the
enriched $\Sigma$-abstraction of $\Hmc_{\Bmc}$.
We show that $\Tmc_{\Bmc}$ is as required to prove
Theorem~\ref{thm:red0}.
The theorem comprises two points:

\medskip
\noindent
(1) We have to show that ${\sf CSP}(\Bmc)$ is equivalent to the
complement of the OMQ $(\Tmc_{\Bmc},\exists x \, M(x))$. This
is an immediate consequence of Claim~2.

\medskip
\noindent
(2) For the converse reduction, let $q$ be a PEQ. We have to show that
the query evaluation problem for $(\Tmc_{\Bmc},q)$ is
reducible in polynomial time to the complement of ${\sf CSP}(\mathcal{B})$.
Let \Amc be an ABox and $\vec{a}$ from ${\sf Ind}(\Amc)$.  We show that
the following are equivalent:
\begin{enumerate}[label=(\alph*),leftmargin=6.5mm]
\item $\Tmc_{\Bmc},\Amc\models q(\vec{a})$;
\item $\Amc|_{\Sigma} \not\rightarrow \Bmc$ or $\Tmc^{\exists},\Amc\models q(\vec{a})$.
\end{enumerate}
Regarding (b), we remark that checking whether
$\Tmc^{\exists},\Amc\models q(\vec{a})$ can be part of the
reduction since,  by Lemma~\ref{lem:abstraction} (3), it needs only polynomial
time. First assume that (a) holds.
If $\Amc|_{\Sigma} \rightarrow \Bmc$, then by Claim~1 
the ABox $\Amc|_{\Sigma}$ is consistent
w.r.t.~$\Hmc_{\Bmc}$. 
By Lemma~\ref{lem:abstraction} (4), we obtain 
$\Tmc_{\Bmc},\Amc\models q(\vec{a})$ iff $\Tmc^{\exists},\Amc\models q(\vec{a})$ for all PEQs $q$ and all $\vec{a}$, as required.

Conversely, assume (b) holds. If $\Amc|_{\Sigma} \not\rightarrow
\Bmc$, then by Claim~2 $\Amc$ is not consistent
w.r.t.~$\Tmc_{\Bmc}$ and so $\Tmc_{\Bmc},\Amc\models q(\vec{a})$. If $\Tmc^{\exists},\Amc\models q(\vec{a})$, then
$\Tmc_{\Bmc},\Amc\models q(\vec{a})$ since $\Tmc^{\exists} \subseteq
\Tmc_{\Bmc}$. 
\end{proof}

\smallskip

We close this section by illustrating an example consequence of
Theorem~\ref{thm:red0}. It was proved in
\cite{DBLP:journals/siamcomp/FederV98} that there are CSPs that are in
{\sc PTime} yet not rewritable into Datalog, and in fact also not into
Datalog$^{\not=}$ due to the results in
\cite{DBLP:conf/lics/FederV03}. This was strengthened to CSPs that
contain no relations of arity larger than two in
\cite{DBLP:journals/ejc/Atserias08}. It was also observed in
\cite{DBLP:journals/siamcomp/FederV98} that there are CSPs that are
rewritable into Datalog, but not into monadic Datalog (such as the CSP
expressing 2-colorability). This again extends to Datalog$^{\not=}$
and applies to CSPs with relations of arity at most two. With this in
mind, the following is a consequence of Theorems~\ref{thm:red0}
and~\ref{equivalence}.  \medskip
\begin{theorem}
~\\[-4mm]
  \begin{enumerate}
  \item 
  There are \ALC-TBoxes \Tmc such that PEQ-evaluation w.r.t.\ \Tmc is 
  in \PTime, but \Tmc is not Datalog-rewritable for ELIQs;
  \item
   there are \ALC-TBoxes that are Datalog-rewritable for PEQs, but not
   monadic Datalog-rewritable for ELIQs.
  \end{enumerate}
\end{theorem}

\newcommand{\gridsig}{g}

\section{Non-Dichotomy and Undecidability in \ALCF}
\label{sec:undec}
We show that the complexity landscape of query evaluation w.r.t.\
\ALCF-TBoxes is much richer than for $\mathcal{ALCI}$, and in fact too
rich to be fully manageable.  In particular, we prove that for
CQ-evaluation, there is no dichotomy between \PTime and {\sc coNP}
(unless {\PTime}$\,$=$\,${\NP}).  We also establish that
materializability, (monadic) Datalog$^{\neq}$-rewritability, \PTime
query evaluation, and {\sc coNP}-hardness of query evaluation are
undecidable.  We start with the undecidability proofs, which are by
reduction of an undecidable rectangle tiling problem and reuse the
`hidden concepts' introduced in the previous section. Next, the TBox
from that reduction is adapted to prove the non-dichotomy result by an
encoding of the computations of nondeterministic polynomial time
Turing machines (again using hidden
concepts). 
The basis for the technical development in this section is a TBox
constructed in \cite{emptiness} to prove the undecidability of
\emph{query emptiness} in \ALCF.

\newcommand{\xinv}{\ensuremath{x^{-}}}
\newcommand{\yinv}{\ensuremath{y^{-}}}

An instance of the finite rectangle tiling problem is given by a
triple $\mathfrak{P}=(\Tmf,H,V)$ with $\Tmf$ a finite set
of \emph{tile types} including an \emph{initial tile} $T_\mn{init}$ to
be placed on the lower left corner and a \emph{final tile}
$T_\mn{final}$ to be placed on the upper right corner, $H \subseteq
\Tmf \times \Tmf$ a \emph{horizontal matching relation}, and $V
\subseteq \Tmf \times \Tmf$ a \emph{vertical matching relation}. A
\emph{tiling} for $(\Tmf,H,V)$ is a map $f:\{0,\dots,n\} \times
\{0,\dots,m\} \rightarrow \Tmf$ such that $n,m \geq 0$,
$f(0,0)=T_\mn{init}$, $f(n,m)=T_\mn{final}$, $(f(i,j),f(i+1,j)) \in H$
for all $i < n$, and $(f(i,j),f(i,j+1)) \in V$ for all $i < m$. We say
that $\mathfrak{P}$ \emph{admits a tiling} if there exists a map $f$
that is a tiling for $\mathfrak{P}$. It is undecidable whether an
instance of the finite rectangle tiling problem admits a tiling.

Now let $\mathfrak{P}=(\Tmf,H,V)$ be a finite rectangle tiling problem
with $\Tmf=\{T_1,\dots,T_p\}$.  We regard the tile types in
$\mathfrak{T}$ as concept names and set $\Sigma_{g} = \{
T_1,\dots,T_p,x,y, \hat{x}, \hat{y}\}$, where $x$, $y$, $\hat{x}$, and
$\hat{y}$ are \emph{functional} role names. The TBox
$\Tmc_{\mathfrak{P}}$ is defined as the following set of CIs, where
$(T_i,T_j,T_\ell)$ range over all triples from \Tmf such that
$(T_i,T_j) \in H$ and $(T_i,T_\ell) \in V$ and where for $e \in \{c,
x, y\}$ the concept $B_{e}$ ranges over all conjunctions $L_1 \sqcap
L_2$ with $L_i\in \{Z_{e,i},\neg Z_{e,i}\}$, for concept names
$Z_{e,i}$~($i=1,2$):
\begin{align*}
 T_\mn{final} & \sqsubseteq  Y \sqcap U \sqcap R \\
 \exists x . (U \sqcap Y \sqcap T_j) \sqcap I_x \sqcap T_i & \sqsubseteq  U \sqcap Y  \\
 \exists y . (R \sqcap Y \sqcap T_{\ell}) \sqcap I_y \sqcap T_{i} & \sqsubseteq  R \sqcap Y \\
  \exists x . (T_{j} \sqcap Y \sqcap \exists y . Y) 
\sqcap   \exists y . (T_{\ell} \sqcap Y \sqcap \exists x . Y)
 \sqcap I_x \sqcap I_y \sqcap C \sqcap T_i & \sqsubseteq  Y  \\  
 Y \sqcap T_\mn{init} & \sqsubseteq  A \\[2mm]
   B_x \sqcap \exists x. \exists \hat{x}. B_{x} & \sqsubseteq  I_x \\
  B_y \sqcap \exists y. \exists \hat{y}. B_{y} & \sqsubseteq  I_y \\
  \exists x . \exists y . B_c \sqcap \exists y . \exists x . B_c 
 &\sqsubseteq C \\ 
 \midsqcup_{1 \leq s < t \leq p} T_s \sqcap T_t  & \sqsubseteq  \bot 
%
%
 \end{align*}
 $$ U   \sqsubseteq  \forall y. \bot\quad
R   \sqsubseteq  \forall x. \bot\quad
 U   \sqsubseteq \forall x. U\quad
 R   \sqsubseteq  \forall y. R
$$
$$ 
Y \sqcap T_{\mn{init}}  \sqsubseteq  D \sqcap L \quad
D   \sqsubseteq  \forall \hat{y}. \bot\quad
L   \sqsubseteq  \forall \hat{x}. \bot\quad
D   \sqsubseteq \forall x. D \sqcap \forall \hat{x}. D \quad
L   \sqsubseteq  \forall y. L \sqcap \forall \hat{y} . L 
$$
With the exception of the CIs in the last line, the TBox
$\Tmc_{\mathfrak{P}}$ has been defined and analyzed in
\cite{emptiness}. Here, we briefly give intuition and discuss its main
properties. The role names $x$ and $y$ are used to represent
horizontal and vertical adjacency of points in a rectangle. The role
names $\hat{x}$ and $\hat{y}$ are used to simulate the inverses of $x$
and $y$.  The concept names in $\Tmc_{\mathfrak{P}}$ serve the
following puroposes:
\begin{itemize}

\item $U$, $R$, $L$, and $D$ mark the upper, right, left, and lower
  (`down') border of the rectangle.

\item In the $B_{c}$ concepts, the concept names $Z_{c,1}$ and
  $Z_{c,2}$ serve as second-order variables and ensure that a flag $C$
  is set at positions where the grid cell is closed. 

\item In the concepts $B_{x}$ and $B_{y}$, the concept names $Z_{x,1},
  Z_{x,2}, Z_{y,1}, Z_{y,2}$ also serve as second-order variables and
  ensure that flags $I_x$ and $I_y$ are set at positions where $x$ and
  $\hat{x}$ as well as $y$ and $\hat{y}$ are inverse to each other.

\item The concept name $Y$ is propagated through the grid from the
  upper right corner to the lower left one, ensuring that these flags
  are set everywhere, that every position of the grid is labeled with
  at least one tile type, and that the horizontal and vertical
  matching conditions are satisfied.

\item Finally, when the lower left corner of the grid is reached, the
  concept name $A$ is set as a flag.

\end{itemize}
Because of the use of the concepts $B_{e}$, CQ evaluation
w.r.t.~$\Tmc_{\mathfrak{P}}$ is coNP-hard: we leave it as an exercise
to the reader to verify that $\Tmc_\Pmf$ does not have the ABox disjunction
property. $\Tmc_{\mathfrak{P}}$ without the three CIs involving the
concepts $B_{e}$, however, is (equivalent to) a Horn-$\mathcal{ALCF}$
TBox and thus enjoys \PTime CQ-evaluation. Call a $\Sigma_{\gridsig}$-ABox
$\Amc$ an \emph{grid ABox (with initial individual $a$)} if
$\Amc$ represents a grid along with a 
proper tiling for~$\mathfrak{P}$. In detail, we require
that there is a tiling $f$ for $\mathfrak{P}$ with domain
$\{0,\ldots,n\}\times \{0,\ldots,m\}$ and a bijection $g:
\{0,\ldots,n\}\times \{0,\ldots,m\} \rightarrow {\sf Ind}(\Amc)$ with
$g(0,0)=a$ such that
\begin{itemize}
\item for all $j < n$, $k \leq m$: $T(g(j,k))\in \Amc$ iff $T = f(j,k)$;
\item for all $b_{1},b_{2} \in {\sf Ind}(\Amc)$: $x(b_{1},b_{2})\in
  \Amc$ iff $\hat{x}(b_{2},b_{1}) \in \Amc$ iff there are $j < n$, $k \leq m$ such that
$(b_{1},b_{2})=(g(j,k),g(j+1,k))$;
\item for all $b_{1},b_{2} \in {\sf Ind}(\Amc)$: $y(b_{1},b_{2})\in
  \Amc$ iff $\hat{y}(b_{2},b_{1}) \in \Amc$ iff there are $j \leq n$, $k < m$ such that
$(b_{1},b_{2})=(g(j,k),g(j,k+1))$.
\end{itemize}
Clearly, if $\mathfrak{P}$ admits a tiling then a grid ABox exists and
for any grid ABox $\Amc$, $\Tmc_{\mathfrak{P}},\Amc \models A(a)$ for
the (uniquely determined) initial individual $a$ of $\Amc$.  The
following summarizes relevant properties of $\Sigma_g$-ABoxes that
follow almost directly from the analysis of $\Tmc_{\mathfrak{P}}$ in
\cite{emptiness}.  We say that an ABox \Amc \emph{contains} an ABox
$\Amc'$ if $\Amc' \subseteq \Amc$ and that $\Amc$ \emph{contains a
  closed ABox \Amc} if, additionally, $r(a,b)\in \Amc$ and
$a\in \mn{Ind}(\Amc')$ implies $r(a,b)\in \Amc'$ for
$r \in \{x,y,\hat{x},\hat{y}\}$.  Moreover, we say that
\emph{inconsistency of ($\Sigma$-)ABoxes w.r.t.~a TBox $\Tmc$ is
  monadic Datalog$^{\neq}$-rewritable} if there is a Boolean
monadic Datalog$^{\neq}$-program $\Pi$ such that for any ($\Sigma$-)ABox \Amc,
$\Amc \models \Pi()$ iff $\Amc$ is inconsistent w.r.t.~\Tmc.
%
%
\begin{lemma}\label{lem:descr}
Let $\mathfrak{P}$ be a finite rectangle tiling problem. Then the following holds.
\begin{enumerate}
\item $\Tmc_{\mathfrak{P}}$ admits trivial models and is $\Sigma_{\gridsig}$-extensional.
\item Inconsistency of $\Sigma_{\gridsig}$-ABoxes 
  w.r.t.~$\Tmc_{\mathfrak{P}}$ is
  monadic Datalog$^{\neq}$-rewritable.
\item If a $\Sigma_{\gridsig}$-ABox $\Amc$ is consistent
  w.r.t.~$\Tmc_{\mathfrak{P}}$, then \Amc contains
\begin{itemize}
\item closed grid ABoxes $\Amc_{1 },\ldots,\Amc_{n}$, $n \geq 0$, with mutually disjoint sets $\mn{Ind}(\Amc_i)$ and
\item a (possibly empty) $\Sigma_{\gridsig}$-ABox $\Amc'$ disjoint from $\Amc_{1}\cup \cdots \cup \Amc_{n}$
\end{itemize}
such that $\Amc= \Amc_{1}\cup \cdots \cup \Amc_{n}\cup \Amc'$ and $\Tmc_{\mathfrak{P}},\Amc\models A(a)$ iff 
$a$ is the initial node of some~$\Amc_{i}$. Moreover, there is
a model \Imc of $\Amc$ witnessing $\Sigma_{\gridsig}$-extensionality of $\Tmc_{\mathfrak{P}}$ that satisfies $a\in A^{\Imc}$ iff 
$a$ is the initial node of some $\Amc_{i}$.
\end{enumerate}
\end{lemma}
%
%
%
%
\begin{proof} 
(1) is a straightforward consequence of the definition of $\Tmc_{\mathfrak{P}}$.
(2) Assume a $\Sigma_{\gridsig}$-ABox $\Amc$ is given. Apply the following
  rules exhaustively to $\Amc$:
\begin{enumerate}
\item[(a)] add $I_{x}(a)$ to $\Amc$ if there exists $b$ with $x(a,b), \hat{x}(b,a)\in \Amc$;
\item[(b)] add $I_{y}(a)$ to $\Amc$ if there exists $b$ with $y(a,b), \hat{y}(b,a)\in \Amc$;
\item[(c)] add $C(a)$ to $\Amc$ if there exist $a_{1},a_{2},b$ with $x(a,a_{1})$, $y(a,a_{2})$, $y(a_{1},b)$, $x(a_{2},b)\in \Amc$.
\end{enumerate}
Denote the resulting ABox by $\Amc^{\dagger}$. Now remove the three
CIs involving the concepts $B_{e}$ from $\Tmc_{\mathfrak{P}}$ and
denote by $\Tmc_{\mathfrak{P}}^{\dagger}$ the resulting TBox.  Using
the analysis of the CIs involving the concepts $B_{e}$ in
\cite{emptiness}, one can show that $\Amc$ is consistent
w.r.t.~$\Tmc_{\mathfrak{P}}$ iff $\Amc^{\dagger}$ is consistent
w.r.t.~$\Tmc_{\mathfrak{P}}^{\dagger}$.  Since the latter is a
Horn-$\mathcal{ALCF}$-TBox, it is unraveling tolerant and one can
build a monadic Datalog$^{\neq}$-rewriting of the inconsistency of
$\Sigma_{\gridsig}$-ABoxes w.r.t.\ $\Tmc_{\mathfrak{P}}^{\dagger}$, essentially
as in the proof of Theorem~\ref{thm:unravupperlem:hornalclem:hornalc}.
Finally, the obtained program can be modified so as to behave as if
started on $\Amc^\dagger$ when started on \Amc, by implementing
rules (a) to (c) as monadic Datalog rules.

\smallskip 

(3) This is almost a direct consequence of the properties established
in \cite{emptiness}. In particular, one finds the desired model $\Imc$
from the `moreover' part by applying the three rules from the proof of
Point~(2) and then applying the CIs in $\Tmc^\dagger_\Pmf$ as rules.
The only condition on the decomposition of the ABox $\Amc$ into $\Amc=
\Amc_{1}\cup \cdots \cup \Amc_{n}\cup \Amc'$ that does not follow from
\cite{emptiness} is that the containment of $\Amc_{1},\ldots,\Amc_{n}$
in \Amc is closed also for the role names $\hat{x}$ and $\hat{y}$.  To
ensure this condition, we use the CIs that mention $L$ and $D$ that were
not present in the TBox used in that paper. In fact, the following two
properties follow directly from these CIs: (i)~the
individuals $c$ reachable along an $x$-path in $\Amc$ from some $a$
with $\Tmc_{\mathfrak{P}},\Amc \models A(a)$ all satisfy
$\Tmc_{\mathfrak{P}},\Amc \models D(c)$ and so do not have an
$\hat{x}$-successor; and (ii)~the individuals $c$ reachable along a
$y$-path in $\Amc$ from some $a$ with $\Tmc_{\mathfrak{P}},\Amc
\models A(a)$ all satisfy $\Tmc_{\mathfrak{P}},\Amc \models L(c)$ and
so do not have a $\hat{y}$-successor. (i) and (ii) together with the
properties established in \cite{emptiness} entail that
the containment of $\Amc_{1},\ldots,\Amc_{n}$
in \Amc is closed also for the role names $\hat{x}$ and $\hat{y}$, as required.
\end{proof}
Note that it follows from Lemma~\ref{lem:descr} (3) that if $\Amc$ is
consistent w.r.t.~$\Tmc$, then the sequence $\Amc_{1},\ldots,\Amc_{n}$
is empty iff $\Tmc_{\mathfrak{P}},\Amc\not\models\exists x \,
A(x)$ (and $\Amc'$ is non-empty since ABoxes are non-empty). In
particular this must be the case when \Pmf does not admit a tiling.
In the proof Lemma~\ref{lem:undecidability} below, this is actually all we need
from Lemma~\ref{lem:descr} (3). In full generality, it will only
be used in the proof of non-dichotomy later on. We also remark
that the decomposition $\Amc_{1}\cup \cdots \cup \Amc_{n}\cup \Amc'$
of \Amc in Lemma~\ref{lem:descr} (3) is unique.

\smallskip

Let $\Tmc = \Tmc_{\mathfrak{P}} \cup \{A \sqsubseteq B_{1} \sqcup
B_{2}\}$, where $B_{1}$ and $B_{2}$ are fresh concept names.  Set
$\Sigma = \Sigma_g \cup \{B_1,B_2\}$ and let $\Tmc' \cup \Tmc^\exists$
be the enriched $\Sigma$-abstraction of $\Tmc$.
\begin{lemma}\label{lem:undecidability}~\\[-4mm]
  \begin{enumerate}

  \item If $\mathfrak{P}$ admits a tiling, then CQ-evaluation w.r.t.~$\Tmc' \cup \Tmc^{\exists}$ is {\sc coNP}-hard and $\Tmc' \cup \Tmc^{\exists}$ is not materializable.

  \item If $\mathfrak{P}$ does not admit a tiling, then CQ-evaluation
    w.r.t.~$\Tmc' \cup \Tmc^{\exists}$ is monadic Datalog$^{\neq}$-rewritable
    and $\Tmc' \cup \Tmc^{\exists}$ is materializable.
  \end{enumerate}
\end{lemma}
\begin{proof}
  (1) If $\mathfrak{P}$ admits a tiling, then there is a
  grid ABox $\Amc$ with initial node $a$. $\Amc$ uses
  symbols from $\Sigma_{\gridsig}$, only.  We have
  $\Tmc_{\mathfrak{P}},\Amc\models A(a)$ and $\Amc$ is consistent
  w.r.t.~$\Tmc_{\mathfrak{P}}$.  
  By Lemma~\ref{lem:abstraction} (2), $\Tmc' \cup \Tmc^{\exists},
  \Amc\models A'(a)$ where $A'$ is the $\Sigma$-abstraction of $A$.
  Since $\Tmc'$ contains $A' \sqsubseteq B_1 \sqcup B_2$ and $B_1$ and
  $B_2$ do not occur elsewhere in $\Tmc' \cup \Tmc^\exists$, it is
  clear that $\Tmc' \cup \Tmc^{\exists}, \Amc\models B_{1}(a) \vee
  B_2(a)$ but $\Tmc' \cup \Tmc^{\exists}, \Amc\not\models B_{i}(a)$
  for $i=1,2$.  Thus $\Tmc' \cup \Tmc^{\exists}$ does not have the
  ABox disjunction property. It follows that $\Tmc'\cup
  \Tmc^{\exists}$ is not materializable and that CQ-evaluation
  w.r.t.~$\Tmc'\cup \Tmc^{\exists}$ is {\sc coNP}-hard.

  (2) Assume that $\mathfrak{P}$ does not admit a tiling. Let $q$ be a
  PEQ. We show how to construct a monadic Datalog$^{\neq}$-rewriting
  $\Pi$ of $(\Tmc' \cup \Tmc^\exists,q)$. On ABoxes \Amc that are
  inconsistent w.r.t.\ $\Tmc' \cup \Tmc^{\exists}$, $\Pi$ is supposed
  to return all tuples over $\mn{Ind}(\Amc)$ of the same arity as $q$.
  By Lemma~\ref{lem:abstraction} (1), an ABox \Amc is consistent
  w.r.t.~$\Tmc' \cup \Tmc^{\exists}$ iff
  $\Amc|_{\Sigma}$
  is consistent w.r.t.~$\Tmc$. It thus follows from
  Lemma~\ref{lem:descr} (2) that inconsistency of ABoxes w.r.t.\
  $\Tmc' \cup \Tmc^{\exists}$ is monadic Datalog$^{\neq}$-rewritable. From a
  concrete rewriting, we can build a monadic Datalog$^{\neq}$-program $\Pi_0$
  that checks inconsistency and, if successful, returns all answers.

  Now for ABoxes \Amc that are consistent w.r.t.\ $\Tmc' \cup
  \Tmc^\exists$.  By Lemma~\ref{lem:abstraction} (1),
  $\Amc|_{\Sigma}$ is consistent
  w.r.t.~$\Tmc_{\mathfrak{P}}$. Since $\mathfrak{P}$ does not admit a
  tiling and by Lemma~\ref{lem:descr} (2),
  $\Tmc_{\mathfrak{P}},\Amc|_{\Sigma}\not\models \exists x
  \, A(x)$. Thus, by Lemma~\ref{lem:descr} (1) we find a model $\Imc$
  of $\Tmc_{\mathfrak{P}}$ and $\Amc|_{\Sigma}$ such that
  $\Delta^{\Imc}= {\sf Ind}(\Amc)$, $B^{\Imc}=\{ a \mid B(a)\in
  \Amc\}$ for all $B\in \Sigma$, $r^{\Imc}=\{ (a,b) \mid r(a,b)\in
  \Amc\}$ for all $r \in \Sigma$, and
  $A^{\Imc}=\emptyset$. Since $A^{\Imc}=\emptyset$, $\Imc$ is a model of $\Tmc$. From
  Lemma~\ref{lem:abstraction} (4), we thus obtain that $\Tmc' \cup
  \Tmc^{\exists},\Amc\models q(\vec{a})$ iff
  $\Tmc^{\exists},\Amc\models q(\vec{a})$, for all~$\vec{a}$. By Lemma~\ref{lem:abstraction} (3),
  $(\Tmc^\exists,q)$ is monadic Datalog$^{\not=}$-rewritable into a program $\Pi_{1}$. 

  The desired program $\Pi$ is simply the union of $\Pi_0$ and
  $\Pi_1$, assuming disjointness of IDB relations.
\end{proof}
Lemma~\ref{lem:undecidability} implies the announced undecidability results.
\begin{theorem}\label{thm:undec}
For $\mathcal{ALCF}$-TBoxes $\Tmc$, the following problems are undecidable (Points~1 to~4
are subject to the side condition that $\PTime \neq
\NP$):
\begin{enumerate} 
\item CQ-evaluation w.r.t.~$\Tmc$ is in {\sc PTime};
\item CQ-evaluation w.r.t.~$\Tmc$ is {\sc coNP}-hard;
\item $\Tmc$ is monadic Datalog$^{\neq}$-rewritable;
\item $\Tmc$ is Datalog$^{\neq}$-rewritable;
\item $\Tmc$ is materializable.
\end{enumerate}
\end{theorem}  
We now come to the proof of non-dichotomy.
\newcommand{\mytp}{\mathfrak{P}}

\begin{theorem}[Non-Dichotomy]\label{nondicho}
For every language $L \in \text{\sc coNP}$, there exists an
$\mathcal{ALCF}$-TBox $\Tmc$ such that, for a distinguished concept name $M_{0}$, the following holds:
\begin{enumerate}
\item $L$ is polynomial time reducible to the evaluation of
$(\Tmc,\exists x\, M_{0}(x))$;
\item the evaluation of $(\Tmc,q)$ is polynomial time reducible to $L$,
for all PEQs $q$.
\end{enumerate}
\end{theorem}
To prove Theorem~\ref{nondicho} let $L\in \text{\sc coNP}$.  Take a
non-deterministic polynomial time Turing Machine $M$ that recognizes
the complement of $L$. Let
$M = (Q,\Gamma_0,\Gamma_1,\Delta,q_{0},q_{a},q_{r})$ with $Q$ a finite
set of states, $\Gamma_0$ and $\Gamma_1$ finite input and tape
alphabets such that $\Gamma_0 \subseteq \Gamma_1$ and
$\Gamma_1 \setminus \Gamma_0$ contains the blank symbol $\beta$,
$q_{0}\in Q$ the starting state,
$\Delta \subseteq Q \times \Gamma_1 \times Q \times \Gamma_1 \times
\{L,R\}$
the transition relation, and $q_{a},q_{r}\in Q$ the accepting and
rejecting states.  Denote by $L(M)$ the language recognized by $M$. We
can assume w.l.o.g.\ that there is a fixed symbol
$\gamma_0 \in \Gamma_0$ such that all words accepted by $M$ are of the
form $\gamma_0v$ with
$v\in (\Gamma_0 \setminus \{ \gamma_0 \})^{\ast}$; in fact, it is easy
to modify $M$ to satisfy this property without changing the complexity
of its word problem. We also assume that for any input
$v\in \Gamma_0^{\ast}$, $M$ uses exactly $|v|^{k_{1}}$ cells for the
computation, halts after exactly $|v|^{k_{2}}$ steps in the accepting
or rejecting state, and does not move to the left of the starting
cell.

  To represent inputs to $M$ 
  and to provide the space for simulating computations, we use grid
  ABoxes as in the proof of Theorem~\ref{thm:undec}, where the tiling
  of the bottom row represents the input word followed by blank
  symbols. As the set of tile types, we use
  $\mathfrak{T}= \Gamma_0 \cup \{ \beta, T, T_{\mn{final}}\}$ where
  $T$ is a `default tile' that labels every position except those in
  the bottom row and the upper right corner. Identify $T_{\sf init}$
  with $\gamma_0$ and let
  $\Sigma_{\gridsig}= \Gamma_0 \cup \{ \beta, T, T_{\mn{final}}\}\cup
  \{x,y,\hat{x},\hat{y}\}$.
  Consider the TBox $\Tmc_{\mytp_{M}}$ defined above, for
  $\mytp_{M}= (\mathfrak{T},H,V)$ with
\begin{eqnarray*}
H  & =
&\{(\gamma_0,\gamma),(\gamma,\gamma'),(\gamma,\beta),(\beta,\beta),(T,T),(T,T_{\mn{final}})
\mid \gamma,\gamma'\in \Gamma_0 \setminus \{\gamma_0\}\}\\
V & =  &\{(\gamma,T),(\beta,T),(T,T),(T,T_{\mn{final}}),(\gamma,T_{\mn{final}}),(\beta,T_{\mn{final}}) \mid \gamma \in \Gamma_0\}.
\end{eqnarray*}
Recall that $\Tmc_{\mytp_{M}}$ checks whether a given $\Sigma_{\gridsig}$-ABox
contains a grid structure with a tiling that respects $H$, $V$,
$T_{\mn{init}}$, and $T_{\mn{final}}$, and derives the concept name
$A$ at the lower left corner of such grids. 
%
%
We now construct a TBox $\Tmc_{M}$
that, after the verification has finished, initiates a computation of
$M$ on the grid. 
%
%
%
In addition to the concept names in $\Tmc_{\mytp_{M}}$, $\Tmc_{M}$
uses concept names $A_\gamma$ and $A_{q,\gamma}$ for all $\gamma \in
\Gamma_1$ and $q \in Q$ to represent symbols written during the
computation (in contrast to the elements of $\Gamma_1$ as concept
names, used to encode the input word) and to represent the state and
head position. In detail, $\Tmc_{M}$ contains the following CIs:
\begin{itemize}


\item When the verification of the grid has finished, $A$ floods the ABox:
$$
   A \sqsubseteq \forall r . A \quad \text{ for all } r\in \{x,y,\hat{x},\hat{y}\}.
$$

\item The initial configuration is as required:
$$
\gamma_0 \sqcap A \sqsubseteq A_{q_0,\gamma} \qquad
\gamma \sqcap A \sqsubseteq A_{\gamma}
\quad \text{ for all }\gamma\in (\Gamma_0 \cup \{ \beta \}) \setminus \{ \gamma_0\}.
$$
\item For every $(q,\gamma)\in Q\times \Gamma_1$, the transition
  relation of $M$ is satisfied:
$$
\begin{array}{rcl}
A_{q,\gamma} \sqcap A  &\sqsubseteq&
\midsqcup_{(q,\gamma,q',\gamma',L)\in \Delta}\exists y.
(A_{\gamma'} \sqcap \midsqcup_{\gamma'' \in \Gamma_1}\exists \hat{x}.A_{q',\gamma''})\;\sqcup\\[4mm]
&& \midsqcup_{(q,\gamma,q',\gamma',R)\in \Delta}\exists y.
(A_{\gamma'} \sqcap \midsqcup_{\gamma'' \in \Gamma_1}\exists x.A_{q',\gamma''}).
\end{array}
$$

\item The representations provided by $A_{q,\gamma}$ and $A_{\gamma}$ for symbols in $\Gamma_{1}$ coincide:
$$
A_{q,\gamma} \sqcap A_{\gamma'}\sqsubseteq A_{q,\gamma'} \sqcap A_{\gamma}, \quad \text{ for all $\gamma,\gamma'\in \Gamma_{1}$}
$$ 
\item The symbol written on a cell does not change if the head is not
  on the cell:
$$
A_{\gamma} \sqcap A \sqsubseteq \forall y.A_{\gamma}
\quad \text{ for all } \gamma\in \Gamma_1
$$
\item 
The rejecting state is never reached:
$$
A_{q_{r},\gamma} \sqcap A \sqsubseteq \bot 
\quad \text{ for all } \gamma \in \Gamma_1. 
$$
\end{itemize}
Let $\Tmc = \Tmc_{\mytp_{M}}\cup \Tmc_{M}$. We are going to show that
an appropriate extended abstraction of \Tmc satisfies Conditions~(1) and~(2) of Theorem~\ref{nondicho}.
We start with the following lemma which summarizes two important
properties of $\Tmc$.
\begin{lemma}\label{lem:coni}~\\[-4mm]
  \begin{enumerate}
\item 
$\Tmc$ admits trivial models and is $\Sigma_{\gridsig}$-extensional. 
\item For any $\Sigma_{\gridsig}$-ABox $\Amc$, $\Amc$ is consistent w.r.t.~$\Tmc$ iff
the following two conditions hold:
\begin{enumerate}
\item $\Amc$ is consistent w.r.t.~$\Tmc_{\mathfrak{P}_M}$;
\item let $\Amc= \Amc_{1}\cup \cdots \cup \Amc_{n} \cup \Amc'$ be the
  decomposition of $\Amc$ given in Lemma~\ref{lem:descr} (2) and
  assume that $\Amc_{i}$ is the $n_{i}\times
  m_{i}$-grid ABox with input $v_{i}$ for $1 \leq i \leq n$. Then the
  following hold for $1 \leq i \leq n$:
\begin{enumerate}
\item $n_{i} \geq |v_{i}|^{k_{1}}$ and $m_{i} \geq |v_{i}|^{k_{2}}$ and 
\item $v_{i}\in L(M)$.
\end{enumerate}
\end{enumerate}
\end{enumerate}
\end{lemma}
\begin{proof}
  Since (1) follows directly from the construction of $\Tmc$, we
  concentrate on (2).  Assume first that $\Amc$ is consistent
  w.r.t.~$\Tmc$. Then $\Amc$ is consistent
  w.r.t.~$\Tmc_{\mathfrak{P}_M}$ and so we can assume that there is
  a decomposition $\Amc_{1}\cup \cdots \cup \Amc_{n}\cup \Amc'$ of $\Amc$
  as in Lemma~\ref{lem:descr} (3). By definition, each $\Amc_{i}$ is
  an
  $n_{i}\times m_{i}$-grid ABox with input
  $v_{i}$. Since $\Amc$ is consistent w.r.t.~$\Tmc$, there is a model
  $\Imc$ of $\Amc$ and $\Tmc$.  By the first CIs of $\Tmc_{M}$ and
  since the initial node $a$ of each $\Amc_{i}$ must be in $A^{\Imc}$
  by Lemma~\ref{lem:descr}, ${\sf Ind}(\Amc_{i})\subseteq A^{\Imc}$
  for each $i$. Thus the restriction of $\Imc$ to ${\sf
    Ind}(\Amc_{i})$ simulates an accepting computation of $M$ starting
  with $v_{i}$.  But since every computatation of $M$ starting with a
  word of length $n$ requires at least $n^{k_{1}}$ space and
  $m^{k_{2}}$ time and the containment of $\Amc_{i}$ in
  $\Amc$ is closed for the role names $x$ and $y$, this is impossible
  if $n_{i} < |v_{i}|^{k_{1}}$ or $m_{i} < |v_{i}|^{k_{2}}$ and also
  if $v_{i}\not\in L(M)$.  Thus (i) and (ii) hold, as required.

  For the converse direction, assume that (a) and (b) hold.  Since
  $\Amc$ is consistent w.r.t.~$\Tmc_{\mathfrak{P}_M}$, there is a
  decomposition $\Amc_{1}\cup \cdots \cup \Amc_{n}\cup \Amc'$ of
  $\Amc$ as in Lemma~\ref{lem:descr} (3). Also by
  Lemma~\ref{lem:descr} (3), there is a model $\Imc$ of $\Amc$ that witnesses
  $\Sigma_{\gridsig}$-extensionality of $\Tmc_{\mathfrak{P}_M}$ such that $a \in A^{\Imc}$ iff $a$
  is the initial node of some $\Amc_{i}$.  We construct a model
  $\Imc'$ of $\Tmc$ by modifying $\Imc$ as follows: with the exception
  of $A$, the symbols of $\Tmc_{\mathfrak{P}_M}$ are interpreted in
  the same way as in $\Imc$ and thus $\Imc'$ is a model of
  $\Tmc_{\mathfrak{P}_M}$. To satisfy the first CI of $\Tmc_{M}$, we
  set $A^{\Imc'}= {\sf Ind}(\Amc_{1}\cup \cdots \cup\Amc_{n})$. Note
  that this suffices since the containment of each $\Amc_i$ in \Amc is
  closed for the role names $x$ and~$y$. The remaining symbols from
  $\Tmc_{M}$ are now interpreted in such a way that they describe on
  each $\Amc_{i}$ an accepting computation for $v_{i}$. This is
  possible since $v_{i} \in L(M)$, $n_{i} \geq |v_{i}|^{k_{1}}$ and
  $m_{i} \geq |v_{i}|^{k_{2}}$, and each computation of $M$ starting
  with a word $v$ of length $n$ requires at most $n^{k_{1}}$ space and
  $m^{k_{2}}$ time.  It can be verified that $\Imc'$ is a model of
  $\Tmc_M$; note that since $A$ is a conjunct of the left hand side of
  every CI in $\Tmc_{M}$, the CIs in $\Tmc_{M}$ are trivially
  satisfied in every node
  $d\in \Delta^{\Imc}\setminus {\sf Ind}(\Amc_{1}\cup \cdots \cup
  \Amc_{n})$.
  Thus $\Imc'$ satisfies $\Tmc$ and $\Amc$ and we have proved
  consistency of $\Amc$ w.r.t.\ $\Tmc$, as required.
\end{proof}
We are now in the position to prove Theorem~\ref{nondicho}.

\begin{proof}[Proof of Theorem~\ref{nondicho}] Let $L \in \text{\sc
  coNP}$ and let $M$ and \Tmc be the TM and TBox from above.  Set
$\Sigma = \Sigma_{\gridsig} \cup \{M_0\}$ where $M_0$ is a fresh
concept name and let $\Tmc'\cup \Tmc^{\exists}$ be the enriched
$\Sigma$-abstraction of $\Tmc$. 
We show that \Tmc satisfies Points~(1) and~(2)
from Theorem~\ref{nondicho}.

\medskip
\noindent
(1) It suffices to give a polynomial time reduction from $L(M)$ to the
complement of evaluating
$(\Tmc'\cup \Tmc^{\exists},\exists x\,M_{0}(x))$ (note that $M_0$ does
not occur in any of the involved TBoxes).  Assume that an input word
$v$ for $M$ is given. If $v$ is not from $\gamma_0 (\Gamma_0 \setminus
\{ \gamma_0 \})^{\ast}$, then reject. Otherwise, construct in polynomial time the
$|v|^{k_{1}}\times |v|^{k_{2}}$-grid ABox $\Amc$ with
input~$v$. Observe that $\Amc$ is consistent
w.r.t.~$\Tmc_{\mathfrak{P}_M}$ and has the trivial decomposition
$\Amc=\Amc_{1}$ in Lemma~\ref{lem:descr} (3). Thus
Lemma~\ref{lem:coni} (2) implies that $v\in L(M)$ iff $\Amc$ is
consistent w.r.t.~$\Tmc$. The latter condition is equivalent to
$\Tmc,\Amc \not\models \exists x\, M_{0}(x)$ since $M_{0}$ does not
occur in $\Amc$ or $\Tmc$.  Since $\Tmc$ admits trivial models,
Lemma~\ref{lem:abstraction} (2) thus yields $v\in L(M)$ iff
$\Tmc'\cup \Tmc^{\exists},\Amc \not\models \exists x\,M_{0}(x)$.

\medskip
\noindent
(2) 
We first make the following observation.

\medskip
\noindent
{\bf Claim 1.} $\Tmc'\cup \Tmc^{\exists},\Amc\models q(\vec{a})$ iff
\begin{itemize}
\item $\Amc|_{\Sigma}$ is not consistent w.r.t.~$\Tmc$ or
\item $\Tmc^{\exists},\Amc\models q(\vec{a})$.
\end{itemize}
For the `only if' direction, observe that if
$\Amc|_{\Sigma}$ is consistent w.r.t.~$\Tmc$, then by
Lemma~\ref{lem:abstraction} (4),
$\Tmc'\cup \Tmc^{\exists},\Amc\models q(\vec{a})$ iff
$\Tmc^{\exists},\Amc\models q(\vec{a})$. For the `if' direction,
observe that if $\Amc|_{\Sigma}$ is not consistent
w.r.t.~$\Tmc$, then by Lemma~\ref{lem:abstraction} (1) $\Amc$ is not
consistent w.r.t.~$\Tmc'\cup \Tmc^{\exists}$.  This finishes the proof
of the claim.

\smallskip

By Lemma~\ref{lem:abstraction} (3),
$\Tmc^{\exists},\Amc\models q(\vec{a})$ can be decided in polynomial
time. Thus, Claim~1 implies that it suffices to give a polynomial time
reduction of ABox consistency w.r.t. \Tmc to $L(M)$. But Lemma~\ref{lem:coni} (2) provides
a polynomial reduction of ABox consistency w.r.t. \Tmc to $L(M)$ since 
\begin{itemize}
\item Condition (a) of Lemma~\ref{lem:coni} (2) can be checked in polynomial time (by Lemma~\ref{lem:descr} (1));
\item the decomposition of $\Amc$ in Condition (b) of Lemma~\ref{lem:coni} (2) as well as the words $v_{i}$, $1\leq i \leq n$,
can be computed in polynomial time;
\item and Point~(i) of Condition (b) can be checked in polynomial time.
\end{itemize}
It thus remains to check whether $v_{i}\in L(M)$ for $1\leq i \leq n$.
This finishes the proof of Theorem~\ref{nondicho}.
\end{proof}

Theorems~\ref{nondicho} and~\ref{equivalence} imply that that there is
no \PTime/{\sc coNP}-dichotomy for query evaluation w.r.t.\
\ALCF-TBoxes, unless {\PTime}$\,$=$\,${\NP}.
%



Observe that the TBoxes constructed in the undecidability and the non-dichotomy
proof are both of depth four. This can be easily reduced to depth three: recall that the
TBoxes of depth four are obtained from TBoxes $\Tmc$ of depth two by taking their enriched 
$\Sigma$-abstractions. One can obtain a TBox of depth three (for which query evaluation
has the same complexity up to polynomial time reductions) by first replacing in $\Tmc$ 
compound concepts $C$ in the scope of a single value or existential restriction
by fresh concept names $A_{C}$ and adding $A_{C}\equiv C$ to $\Tmc$. Then
the fresh concept names are added to the signature $\Sigma$ and one constructs
the enriched abstraction of the resulting TBox for the extended signature.
This TBox is as required. Thus, our undecidability and non-dichotomy results hold for $\mathcal{ALCF}$-TBoxes
of depth three already.

\section{Discussion}

We have studied the complexity of query evaluation in the presence of
an ontology formulated in a DL between \ALC and $\mathcal{ALCFI}$,
focussing on the boundary between {\sc PTime} and {\sc coNP}. For
$\mathcal{ALCFI}$-TBoxes of depth one, we have established
a dichotomy between \PTime and {\sc coNP} and shown that it can be
precisely characterized in terms of unraveling tolerance and
materializability. Moreover and unlike in the general case, {\sc
  PTime} complexity coincides with rewitability into
Datalog$^{\neq}$. The case of higher or unrestricted depth is harder
to analyze. We have shown that for arbitrary \ALC- and \ALCI-TBoxes
there is a dichotomy between {\sc PTime} and {\sc coNP}. The proof is
by a reduction to the recently confirmed {\sc PTime}/{\sc NP}-dichotomy for
CSPs.  For \ALCF
TBoxes of depth three we have shown that there is no dichotomy unless
{\sc \PTime $=$ \NP} and that deciding whether a given TBox admits
{\sc PTime} query evaluation is undecidable, and so are related
questions.


\medskip

Several interesting research questions remain. We briefly discuss
three possible directions. 

\smallskip
\noindent
(1) Is it decidable whether a given \ALC- or \ALCI-TBox admits {\sc
 PTime} query evaluation and, closely related, whether it is
unraveling tolerant and whether it is materializable? First results for TBoxes of depth one
have been obtained in \cite{PODS17}, but the general problem remains open.
It is interesting to point out
that unraveling tolerance is decidable for OMQs whose TBox is
formulated in \ALCI (where a concrete query is given, unlike in the
case of unraveling tolerance of TBoxes); in that case, unraveling
tolerance is equivalent to rewritability into monadic Datalog
\cite{OurNewPaper}. It would also be interesting to study more general
notions of unraveling tolerance based on unravelings into structures
of bounded treewidth rather than into real trees.

\smallskip
\noindent
(2) It would be interesting to study additional complexity classes
such as {\sc LogSpace}, {\sc NLogSpace}, and AC$^{0}$. It is known
that all these classes occur even for \ALC-TBoxes of depth one, see
e.g.\ \cite{DBLP:journals/ai/CalvaneseGLLR13} and the recent
\cite{leif} which establishes a full complexity complexity
classification of OMQs that are based on an \EL-TBox and an ELIQ. For example,
CQ-evaluation w.r.t.\ the depth one \EL-TBox
$\{ \exists r . A \sqsubseteq A \}$, which encodes reachability in
directed graphs, is {\sc NLogSpace}-complete. It would thus be
interesting to identify further dichotomies such as between {\sc
  NLogSpace} and {\sc PTime}.  We conjecture that for \ALCFI-TBoxes of
depth one, it is decidable whether query evaluation is in {\sc PTime},
{\sc NLogSpace}, {\sc LogSpace}, and~AC$^0$.

%


\smallskip
\noindent
(3) Apart from Datalog, rewritability into FO queries is also of
interest. In the context of OMQs where the actual query is fixed
rather than quantified, several results have been obtained, see e.g.\
\cite{DBLP:conf/ijcai/BienvenuLW13,DBLP:conf/ijcai/0002LSW15,DBLP:conf/ijcai/Bienvenu0LW16}
for FO-rewritability of OMQs whose TBox is formulated in a Horn DL
and~\cite{DBLP:journals/tods/BienvenuCLW14,OurNewPaper} for FO- and
Datalog-rewritability of OMQs whose TBox is formulated in \ALC or an
extension thereof. When the query is quantified (as in the current
paper), a first relevant result has been established in
\cite{DBLP:conf/dlog/LutzW11} where it is shown that that
FO-rewritability is decidable for materializable
$\mathcal{ALCFI}$-TBoxes of depth one. This underlines the importance
of deciding materializability, which would allow to lift this result
to (otherwise unrestricted) $\mathcal{ALCFI}$-TBoxes of depth one.

\medskip

\smallskip
\noindent {\bf Acknowledgments.}\ \ Carsten Lutz was supported by ERC
consolidator grant 647289. Frank Wolter was supported by EPSRC grant EP/M012646/1.

\newcommand{\etalchar}[1]{$^{#1}$}


\begin{thebibliography}{CDGL{\etalchar{+}}07}

\bibitem[ABI{\etalchar{+}}05]{DBLP:conf/mfcs/AllenderBISV05}
Eric Allender, Michael Bauland, Neil Immerman, Henning Schnoor, and Heribert
  Vollmer.
\newblock The complexity of satisfiability problems: Refining {S}chaefer's
  theorem.
\newblock In {\em MFCS}, pages 71--82, 2005.

\bibitem[ACKZ09]{DBLP:journals/jair/ArtaleCKZ09}
Alessandro Artale, Diego Calvanese, Roman Kontchakov, and Michael
  Zakharyaschev.
\newblock The {D}{L}-{L}ite family and relations.
\newblock {\em J. Artif. Intell. Res. (JAIR)}, 36:1--69, 2009.

\bibitem[ACY91]{DBLP:conf/pods/AfratiCY91}
Foto~N. Afrati, Stavros~S. Cosmadakis, and Mihalis Yannakakis.
\newblock On datalog vs. polynomial time.
\newblock In {\em PODS}, pages 13--25, 1991.

\bibitem[Ats08]{DBLP:journals/ejc/Atserias08}
Albert Atserias.
\newblock On digraph coloring problems and treewidth duality.
\newblock {\em Eur. J. Comb.}, 29(4):796--820, 2008.

\bibitem[Bar14]{DBLP:journals/siglog/Barto14}
Libor Barto.
\newblock Constraint satisfaction problem and universal algebra.
\newblock {\em {SIGLOG} News}, 1(2):14--24, 2014.

\bibitem[BBLW16]{emptiness}
Franz Baader, Meghyn Bienvenu, Carsten Lutz, and Frank Wolter.
\newblock Query and predicate emptiness in ontology-based data access.
\newblock {\em J. Artif. Intell. Res. {(JAIR)}}, 56:1--59, 2016.

\bibitem[BBL05]{BaBrLu-IJCAI-05}
Franz Baader, Sebastian Brandt, and Carsten Lutz.
\newblock Pushing the $\mathcal{EL}$ envelope.
\newblock In {\em IJCAI}, pages 364--369. Professional Book Center, 2005.

\bibitem[BGO10]{LICSo}
Vince Barany, Georg Gottlob, and Martin Otto.
\newblock Querying the guarded fragment.
\newblock In {\em LICS}, pages 1--10, 2010.

\bibitem[BHLW16]{DBLP:conf/ijcai/Bienvenu0LW16}
Meghyn Bienvenu, Peter Hansen, Carsten Lutz, and Frank Wolter.
\newblock First order-rewritability and containment of conjunctive queries in
  {Horn} description logics.
\newblock In {\em IJCAI}, pages 965--971, 2016.

\bibitem[Bul17]{Dicho1}
Andrei~A. Bulatov.
\newblock A dichotomy theorem for nonuniform {CSPs}.
\newblock In {\em FOCS}, 2017.

\bibitem[BJK05]{DBLP:journals/siamcomp/BulatovJK05}
Andrei~A. Bulatov, Peter Jeavons, and Andrei~A. Krokhin.
\newblock Classifying the complexity of constraints using finite algebras.
\newblock {\em SIAM J. Comput.}, 34(3):720--742, 2005.

\bibitem[BLW13]{DBLP:conf/ijcai/BienvenuLW13}
Meghyn Bienvenu, Carsten Lutz, and Frank Wolter.
\newblock First-order rewritability of atomic queries in horn description
  logics.
\newblock In {\em IJCAI},
  pages 754--760, 2013.

\bibitem[BMRT11]{IJCO}
Jean-Francois Baget, Marie-Laure Mugnier, Sebastian Rudolph, and Michael
  Thomazo.
\newblock Walking the complexity lines for generalized guarded existential
  rules.
\newblock In {\em IJCAI}, pages 712--717, 2011.

\bibitem[BO15]{DBLP:conf/rweb/BienvenuO15}
Meghyn Bienvenu and Magdalena Ortiz.
\newblock Ontology-mediated query answering with data-tractable description
  logics.
\newblock In {\em Reasoning Web}, volume 9203 of {\em LNCS}, pages
  218--307. Springer, 2015.

\bibitem[BtCLW14]{DBLP:journals/tods/BienvenuCLW14}
Meghyn Bienvenu, Balder ten Cate, Carsten Lutz, and Frank Wolter.
\newblock Ontology-based data access: {A} study through disjunctive datalog,
  CSP, and {MMSNP}.
\newblock {\em {ACM} Trans. Database Syst.}, 39(4):33:1--33:44, 2014.

\bibitem[BLR{\etalchar{+}}16]{DBLP:conf/ijcai/BotoevaLRWZ16}
Elena Botoeva, Carsten Lutz, Vladislav Ryzhikov, Frank Wolter and Michael Zakharyaschev.
\newblock Query-Based Entailment and Inseparability for {ALC} Ontologies
\newblock In {\em IJCAI}, pages 1001--1007, 2016.

\bibitem[Bul02]{DBLP:conf/focs/Bulatov02}
Andrei~A. Bulatov.
\newblock A dichotomy theorem for constraints on a three-element set.
\newblock In {\em FOCS}, pages 649--658, 2002.

\bibitem[Bul11]{DBLP:conf/csr/Bulatov11}
Andrei~A. Bulatov.
\newblock On the {C}{S}{P} dichotomy conjecture.
\newblock In {\em CSR}, pages 331--344, 2011.

\bibitem[CDGL{\etalchar{+}}07]{CDLLR07}
Diego Calvanese, Giuseppe De~Giacomo, Domenico Lembo, Maurizio Lenzerini, and
  Riccardo Rosati.
\newblock Tractable reasoning and efficient query answering in description
  logics: The {{\textit{DL-Lite}}} family.
\newblock {\em J.\ of Autom.\ Reasoning}, 39(3):385--429, 2007.

\bibitem[CDL{\etalchar{+}}13]{DBLP:journals/ai/CalvaneseGLLR13}
Diego Calvanese, Giuseppe {De Giacomo}, Domenico Lembo, Maurizio Lenzerini, and
  Riccardo Rosati.
\newblock Data complexity of query answering in description logics.
\newblock {\em Artificial Intelligence}, 195:335--360, 2013.

\bibitem[CGK13]{DBLP:journals/jair/CaliGK13}
Andrea Cal{\`{\i}}, Georg Gottlob, and Michael Kifer.
\newblock Taming the infinite chase: Query answering under expressive
  relational constraints.
\newblock {\em J. Artif. Intell. Res. {(JAIR)}}, 48:115--174, 2013.

\bibitem[CGLV00]{LICS00}
Diego Calvanese, Giuseppe~De Giacomo, Maurizio Lenzerini, and Moshe~Y. Vardi.
\newblock View-based query processing and constraint satisfaction.
\newblock In {\em LICS}, pages 361--371, 2000.

\bibitem[CGLV03a]{DBLP:journals/sigmod/CalvaneseGLV03}
Diego Calvanese, Giuseppe~De Giacomo, Maurizio Lenzerini, and Moshe~Y. Vardi.
\newblock Reasoning on regular path queries.
\newblock {\em SIGMOD Record}, 32(4):83--92, 2003.

\bibitem[CGLV03b]{PODS03}
Diego Calvanese, Giuseppe~De Giacomo, Maurizio Lenzerini, and Moshe~Y. Vardi.
\newblock View-based query containment.
\newblock In {\em PODS}, pages 56--67, 2003.

\bibitem[CGT89]{DBLP:journals/tkde/CeriGT89}
Stefano Ceri, Georg Gottlob, and Letizia Tanca.
\newblock What you always wanted to know about datalog (and never dared to
  ask).
\newblock {\em {IEEE} Trans. Knowl. Data Eng.}, 1(1):146--166, 1989.

\bibitem[CK90]{ChangKeisler}
C.~C. Chang and H.~Jerome Keisler.
\newblock {\em Model Theory}, volume~73 of {\em Studies in Logic and the
  Foundations of Mathematics}.
\newblock Elsevier, 1990.

\bibitem[EGOS08]{conf/jelia/EiterGOS08}
Thomas Eiter, Georg Gottlob, Magdalena Ortiz, and Mantas Simkus.
\newblock Query answering in the description logic {Horn}-$\mathcal{SHIQ}$.
\newblock In {\em JELIA}, pages 166--179, 2008.

\bibitem[FKL17]{OurNewPaper}
Cristina Feier, Antti Kuusisto, and Carsten Lutz.
\newblock Rewritability in monadic disjunctive datalog, {MMSNP}, and expressive
  description logics.
\newblock In {\em ICDT}, pages 1--17, 2017.

\bibitem[FKMP05]{DBLP:journals/tcs/FaginKMP05}
Ronald Fagin, Phokion~G. Kolaitis, Ren{\'{e}}e~J. Miller, and Lucian Popa.
\newblock Data exchange: semantics and query answering.
\newblock {\em Theor. Comput. Sci.}, 336(1):89--124, 2005.


\bibitem[FV98]{DBLP:journals/siamcomp/FederV98}
Tom{\'a}s Feder and Moshe~Y. Vardi.
\newblock The Computational Structure of Monotone Monadic {SNP} and Constraint Satisfaction: {A} Study through Datalog and Group Theory
\newblock {\em SIAM J.\ Comput.}, 28(1):57--104, 1998.

\bibitem[FV03]{DBLP:conf/lics/FederV03}
Tom{\'{a}}s Feder and Moshe~Y. Vardi.
\newblock Homomorphism closed vs. existential positive.
\newblock In {\em LICS}, pages 311--320, 2003.

\bibitem[GLHS08]{jair2008g}
Birte Glimm, Carsten Lutz, Ian Horrocks, and Ulrike Sattler.
\newblock Conjunctive query answering for the description logic
  $\mathcal{SHIQ}$.
\newblock {\em JAIR}, 31:157--204, 2008.

\bibitem[GO07]{goranko20075}
Valentin Goranko and Martin Otto.
\newblock Model theory of modal logic.
\newblock In {\em Handbook of Modal Logic}, pages 249--329. Elsevier, 2007.

\bibitem[HLSW15]{DBLP:conf/ijcai/0002LSW15}
Peter Hansen, Carsten Lutz, Inan{\c{c}} Seylan, and Frank Wolter.
\newblock Efficient query rewriting in the description logic $\mathcal{EL}$ and beyond.
\newblock In {\em IJCAI}, pages 3034--3040, 2015.

\bibitem[HLPW17a]{PODS17}
Andr\'e Hernich, Carsten Lutz, Fabio Papacchini, Frank Wolter.
\newblock Dichotomies in Ontology-Mediated Querying with the Guarded Fragment. 
\newblock In {\em PODS}, pages 185--199, 2017.

\bibitem[HLPW17b]{DL17a}
Andr\'e Hernich, Carsten Lutz, Fabio Papacchini, Frank Wolter.
\newblock Horn Rewritability vs PTime Query Answering for Description Logic TBoxes.
\newblock In {\em Description Logics}, 2017.

\bibitem[HMS07]{journals/jar/HustadtMS07}
Ullrich Hustadt, Boris Motik, and Ulrike Sattler.
\newblock Reasoning in description logics by a reduction to disjunctive
  datalog.
\newblock {\em J. Autom. Reasoning}, 39(3):351--384, 2007.

\bibitem[HN90]{DBLP:journals/jct/HellN90}
Pavol Hell and Jaroslav Nesetril.
\newblock On the complexity of {\it h}-coloring.
\newblock {\em J. Comb. Theory, Ser. B}, 48(1):92--110, 1990.

\bibitem[KNC16]{DBLP:journals/ai/KaminskiNG16}
Mark Kaminski, Yavor Nenov, and Bernardo Cuenca Grau,
\newblock Datalog rewritability of Disjunctive Datalog programs and non-Horn ontologies.
\newblock {\em Artificial Intelligence}, 236: 90--118, 2016.

\bibitem[Kaz09]{conf/ijcai/Kazakov09}
Yevgeny Kazakov.
\newblock Consequence-driven reasoning for {Horn}-$\mathcal{SHIQ}$ ontologies.
\newblock In Craig Boutilier, editor, {\em IJCAI}, pages 2040--2045, 2009.

\bibitem[KL07]{DBLP:conf/lpar/KrisnadhiL07}
Adila Krisnadhi and Carsten Lutz.
\newblock Data complexity in the $\mathcal{EL}$ family of description logics.
\newblock In {\em LPAR}, pages 333--347, 2007.

\bibitem[KLT{\etalchar{+}}10]{DBLP:conf/kr/KontchakovLTWZ10}
Roman Kontchakov, Carsten Lutz, David Toman, Frank Wolter, and Michael
  Zakharyaschev.
\newblock The combined approach to query answering in {D}{L}-{L}ite.
\newblock In {\em KR}, 2010.

\bibitem[KRH07]{2007cbfhdl}
Markus Kr{\"o}tzsch, Sebastian Rudolph, and Pascal Hitzler.
\newblock Complexity boundaries for {Horn} description logics.
\newblock In {\em AAAI}, pages 452--457, 2007.

\bibitem[Kro10a]{DBLP:conf/csl/Krokhin10}
Andrei~A. Krokhin.
\newblock Tree dualities for constraint satisfaction.
\newblock In {\em CSL}, pages 32--33, 2010.

\bibitem[Kr{\"o}10b]{conf/jelia/Krotzsch10}
Markus Kr{\"o}tzsch.
\newblock Efficient inferencing for {OWL} {EL}.
\newblock In {\em JELIA}, pages 234--246, 2010.

\bibitem[KS09]{DBLP:conf/stoc/KunS09}
G{\'a}bor Kun and Mario Szegedy.
\newblock A new line of attack on the dichotomy conjecture.
\newblock In {\em STOC}, pages 725--734, 2009.

\bibitem[KZ14]{DBLP:conf/rweb/KontchakovZ14}
Roman Kontchakov and Michael Zakharyaschev.
\newblock An introduction to description logics and query rewriting.
\newblock In {\em Reasoning Web}, volume 8714 of {\em LNCS}, pages
  195--244. Springer, 2014.

\bibitem[LLT07]{DBLP:journals/lmcs/LaroseLT07}
Benoit Larose, Cynthia Loten, and Claude Tardif.
\newblock A characterisation of first-order constraint satisfaction problems.
\newblock {\em Logical Methods in Computer Science}, 3(4), 2007.

\bibitem[LPW11]{TBoxpaper}
Carsten Lutz, Robert Piro, and Frank Wolter.
\newblock Description logic tboxes: Model-theoretic characterizations and
  rewritability.
\newblock In {\em IJCAI}, 2011.

\bibitem[LSW13]{DBLP:conf/ijcai/LutzSW13}
Carsten Lutz, Inan{\c{c}} Seylan, and Frank Wolter.
\newblock Ontology-based data access with closed predicates is inherently
  intractable (sometimes).
\newblock In {\em IJCAI}, pages 1024--1030. {IJCAI/AAAI}, 2013. 

\bibitem[LS17]{leif}
Carsten Lutz and Leif Sabellek.
\newblock Ontology-Mediated Querying with the Description Logic
$\mathcal{EL}$: Trichotomy and Linear Datalog Rewritability.
\newblock In {\em IJCAI}, pages 1181--1187. {IJCAI/AAAI}, 2017. 

\bibitem[LSW15]{DBLP:conf/ijcai/LutzSW15}
Carsten Lutz, Inan{\c{c}} Seylan, and Frank Wolter.
\newblock Ontology-mediated queries with closed predicates.
\newblock In {\em IJCAI}, pages 3120--3126. {AAAI} Press, 2015.

\bibitem[LTW09]{DBLP:conf/ijcai/LutzTW09}
Carsten Lutz, David Toman, and Frank Wolter.
\newblock Conjunctive query answering in the description logic $\mathcal{EL}$
  using a relational database system.
\newblock In {\em IJCAI}, pages 2070--2075, 2009.

\bibitem[LW10]{DBLP:journals/jsc/LutzW10}
Carsten Lutz and Frank Wolter.
\newblock Deciding inseparability and conservative extensions in the
  description logic $\mathcal{EL}$.
\newblock {\em J. Symb. Comput.}, 45(2):194--228, 2010.

\bibitem[LW11]{DBLP:conf/dlog/LutzW11}
Carsten Lutz and Frank Wolter.
\newblock Non-uniform data complexity of query answering in description logics.
\newblock In {\em Description Logics}, 2011.

\bibitem[LW12]{DBLP:conf/kr/LutzW12}
Carsten Lutz and Frank Wolter.
\newblock Non-uniform data complexity of query answering in description logics.
\newblock In {\em KR}. {AAAI} Press, 2012.

\bibitem[Mak87]{DBLP:journals/jcss/Makowsky87}
Johann~A. Makowsky.
\newblock Why {Horn} formulas matter in computer science: Initial structures
  and generic examples.
\newblock {\em J. Comput. Syst. Sci.}, 34(2/3):266--292, 1987.

\bibitem[Mal71]{Malcev}
Anatoli~I. Malcev.
\newblock {\em The metamathematics of algebraic systems, collected
  papers:1936-1967}.
\newblock North-Holland, 1971.

\bibitem[MG85]{MesGog}
Jose Meseguer and Joseph~A. Goguen.
\newblock Initiality, induction, and computability.
\newblock In {\em Algebraic Methods in Semantics}, pages 459--541. Cambridge
  University Press, 1985.

\bibitem[OCE08]{DBLP:journals/jar/OrtizCE08}
Magdalena Ortiz, Diego Calvanese, and Thomas Eiter.
\newblock Data complexity of query answering in expressive description logics
  via tableaux.
\newblock {\em J.\ of Autom.\ Reasoning}, 41(1):61--98, 2008.

\bibitem[PLC{\etalchar{+}}08]{DBLP:journals/jods/PoggiLCGLR08}
Antonella Poggi, Domenico Lembo, Diego Calvanese, Giuseppe~De Giacomo, Maurizio
  Lenzerini, and Riccardo Rosati.
\newblock Linking data to ontologies.
\newblock {\em J. Data Semantics}, 10:133--173, 2008.

\bibitem[Ros07]{DBLP:conf/icdt/Rosati07}
Riccardo Rosati.
\newblock The limits of querying ontologies.
\newblock In {\em ICDT}, volume 4353 of {\em LNCS}, pages 164--178.
  Springer, 2007.

\bibitem[Sch78]{DBLP:conf/stoc/Schaefer78}
Thomas~J. Schaefer.
\newblock The complexity of satisfiability problems.
\newblock In {\em STOC}, pages 216--226, 1978.

\bibitem[Sch93]{Schaerf-93}
Andrea Schaerf.
\newblock On the complexity of the instance checking problem in concept
  languages with existential quantification.
\newblock {\em J.\ of Intel.\ Inf.\ Systems}, 2:265--278, 1993.

\bibitem[Zhu17]{Dicho2}
Dmitriy Zhuk.
\newblock The Proof of CSP Dichotomy Conjecture.
\newblock In {\em FOCS}, 2017.

\end{thebibliography}
\end{document}